\documentclass[12pt]{article}
\usepackage{mathrsfs}
\usepackage{natbib,graphicx,setspace,lscape,longtable}
\usepackage{mathrsfs,amsmath,amsthm,amssymb,color}
\usepackage{natbib,epsfig,graphicx,pdfpages}
\usepackage[figuresright]{rotating}
\usepackage{geometry}
\usepackage{subcaption}
\usepackage{threeparttable}
\usepackage{enumerate}
\usepackage{xurl}
\usepackage{graphicx}
\usepackage{natbib}
\usepackage{algorithmic}
\usepackage[linesnumbered,ruled,vlined]{algorithm2e}
\usepackage{booktabs}

\usepackage{url}
\usepackage{hyperref}
\usepackage{bm}
\usepackage[table]{xcolor}
\usepackage{colortbl}
\usepackage{color}
\usepackage{multirow}
\usepackage{multicol}
\usepackage{booktabs}
\bibpunct{(}{)}{;}{a}{,}{,}

\newtheorem{theorem}{Theorem}
\newtheorem{lemma}{Lemma}
\newtheorem{proposition}{Proposition}
\newtheorem{remark}{Remark}
\usepackage{amsthm}

\newtheorem{assumption}{{Assumption}}
\def\beq{\begin{equation}}
\def\eeq{\end{equation}}
\def\beqr{\begin{eqnarray}}
\def\eeqr{\end{eqnarray}}
\def\beqrs{\begin{eqnarray*}}
\def\eeqrs{\end{eqnarray*}}
\def\bet{\begin{theorem}}
\def\eet{\end{theorem}}
\def\bel{\begin{lemma}}
\def\eel{\end{lemma}}
\def\bep{\begin{proposition}}
\def\eep{\end{proposition}}
\def\bg{\begin{figure}[tbph]\begin{center}}
\def\eg{\end{center}\end{figure}}

\def\bc{\begin{center}}
\def\ec{\end{center}}

\def\wt{\widetilde}
\def\wh{\widehat}

\renewcommand{\baselinestretch}{1.4}
\numberwithin{equation}{section}
\numberwithin{table}{section}

\makeatletter
\newcommand*{\rom}[1]{\romannumeral #1}
\makeatother
\newcommand{\bA}{{\mathbf A}}
\newcommand{\bB}{{\mathbf B}}
\newcommand{\bF}{{\mathbf F}}
\newcommand{\bE}{{\mathbf E}}
\newcommand{\bG}{{\mathbf G}}
\newcommand{\bH}{{\mathbf H}}
\newcommand{\bI}{{\mathbf I}}

\newcommand{\bM}{{\mathbf M}}

\newcommand{\bQ}{{\mathbf Q}}
\newcommand{\bP}{{\mathbf P}}
\newcommand{\bR}{{\mathbf R}}
\newcommand{\bS}{{\mathbf S}}

\newcommand{\bU}{{\mathbf U}}
\newcommand{\bV}{{\mathbf V}}
\newcommand{\bW}{{\mathbf W}}

\newcommand{\bY}{{\mathbf Y}}
\newcommand{\bZ}{{\mathbf Z}}

\newcommand{\bd}{{\mathbf d}}
 
\newcommand{\be}{{\mathbf e}}

\newcommand{\bu}{{\mathbf u}}
\newcommand{\bv}{{\mathbf v}}
\newcommand{\bw}{{\mathbf w}}
\newcommand{\bx}{{\mathbf x}}

\newcommand{\bC}{{\mathbf C}}
\newcommand{\bD}{{\mathbf D}}

\newcommand{\ve}{{\varepsilon}}
\renewcommand{\epsilon}{{\ve}}
\renewcommand{\hat}{\widehat}

\def\wt{\widetilde}
\renewcommand{\tilde}{\wt}

\begin{document}
\def\spacingset#1{\renewcommand{\baselinestretch}%
{#1}\small\normalsize} \spacingset{1}

\title{Split-and-Conquer: Distributed Factor Modeling for High-Dimensional Matrix-Variate Time Series}

\author{Hangjin Jiang$^1$, Yuzhou Li$^{1,2}$, and Zhaoxing Gao$^3$\footnote{Corresponding author: \texttt{zhaoxing.gao@uestc.edu.cn} (Z. Gao),  School of Mathematical Sciences, University of Electronic Science and Technology of China, Chengdu, 611731 P.R. China. }\\
{\normalsize $^{1}$Center for Data Science, Zhejiang University}\\ {\normalsize $^2$School of Mathematical Sciences, Zhejiang University} \\ {\normalsize $^3$School of Mathematical Sciences, University of Electronic Science and Technology of China}
}
\date{}

\maketitle

\begin{abstract}
In this paper, we propose a distributed framework for reducing the dimensionality of high-dimensional, large-scale, heterogeneous matrix-variate time series data using a factor model. The data are first partitioned column-wise (or row-wise) and allocated to node servers, where each node estimates the row (or column) loading matrix via two-dimensional tensor PCA. These local estimates are then transmitted to a central server and aggregated, followed by a final PCA step to obtain the global row (or column) loading matrix estimator. Given the estimated loading matrices, the corresponding factor matrices are subsequently computed. Unlike existing distributed approaches, our framework preserves the latent matrix structure, thereby improving computational efficiency and enhancing information utilization. We also discuss row- and column-wise clustering procedures for settings in which the group memberships are unknown. Furthermore, we extend the analysis to unit-root nonstationary matrix-variate time series. Asymptotic properties of the proposed method are derived for the diverging dimension of the data in each computing unit and the sample size $T$. Simulation results assess the computational efficiency and estimation accuracy of the proposed framework, and real data applications further validate its predictive performance.
\end{abstract}


\newpage

\section{Introduction}
Time series models with factor structures are widely applied across diverse fields, including finance, macroeconomics, and machine learning, due to their ability to represent large-scale time series through a small number of latent factors. Early research considered models with a fixed number of factors. For instance, \cite{ROSS1976341} modeled asset returns as linear functions of macroeconomic risks, and \cite{GREGORY1999423} represented multinational current accounts using global and country-specific factors. With the growing focus on dimensionality reduction for large-scale data, subsequent research developed approaches to infer a low-dimensional factor structure. \cite{Chamberlain1981} introduced the approximate factor model employing principal component analysis (PCA) to estimate the latent factors, with important extensions by \cite{Stock1998}, \cite{Stock01122002}, and \cite{bai2002}. To address strong temporal dependence, later research captured co-movements across multiple moments, leading to the dynamic factor model \citep{forni2000,forni2005,FORNI2015359} and factor models with dynamically dependent factors \citep{Lam_2012,gao2019,Gao_2021}. When the number of factors is unknown, various methods have been proposed for its estimation, including the matrix rank test \citep{Lewbel1991,Donald1997}, modified BIC \citep{Stock1998,bai2002}, multivariate AIC \citep{forni2000}, and eigenvalue ratio-based methods \citep{Lam_2012,ahn2013}. 

While vector factor models are widely used, they are not well-suited for matrix time series because they ignore multi-dimensional structural information. Matrix time series have become increasingly important in fields such as neuroscience (e.g., large-scale fMRI data), economics (e.g., high-frequency trading records), and machine learning (e.g., parameter snapshots of pre-trained models). \cite{WANG2019231} introduced a matrix factor model that preserves the latent matrix structure, employing row and column loading matrices to capture the variance across the corresponding dimensions. Building on this model and two-dimensional tensor PCA, \cite{chen2021} proposed an $\alpha$-PCA estimation procedure to enhance estimation accuracy and applicability, while \cite{yu2020} introduced a projection method that further improves estimation accuracy. \cite{he2022matrixfactoranalysissquares} developed a robust iterative estimation procedure, and \cite{he2022matrixkendallstauhighdimensions} robustified the factor number estimation via eigenvalue ratios of generalized row/column matrix Kendall’s tau matrices. Additionally, \cite{GAO202383} proposed a two-way transformed model providing solid statistical guarantees and improved interpretability.

Despite the effectiveness of factor models in dimensionality reduction for high-dimensional time series, storing and processing all data on a single machine are often impractical. Estimation procedures commonly require matrix operations involving billions or trillions of elements, making single-machine estimation infeasible. In practice, large-scale data are always distributed across multiple machines due to concerns over communication costs, privacy, and security. For example, high-frequency trading data are stored across exchanges, remote sensing data are massive and dispersed, and large-scale fMRI datasets are distributed across institutions to protect medical privacy. To address these challenges, distributed estimation procedures have been developed in which local estimators are computed on node servers, transmitted to a central server, and aggregated into global estimates. Early work by \cite{Qu2002} introduced a distributed PCA method under a horizontal partitioning regime, where each node stores a subset of observations across all variables. \cite{fan2018} refined the distributed PCA and established solid statistical guarantees, while \cite{garber2017communicationefficientalgorithmsdistributedstochastic} and \cite{chen2021distributedestimationprincipalcomponent} proposed multi-round methods to reduce communication cost and improve computational efficiency. Additionally, \cite{Kargupta2000} introduced a collective PCA method with a vertical partitioning regime, where each node stores all observations for a subset of variables, with further developments by \cite{Li2011} and \cite{BERTRAND2014120}.

However, as previously discussed, distributed methods developed for vector-variate factor models perform sub-optimally when applied to matrix time series. This motivates the study on distributed estimation for matrix-variate factor models. A naive approach is to partition the matrix data along columns or rows, estimate the loading matrices within each block, and then aggregate the local results to obtain global estimates. Although this framework preserves the latent matrix structure, it has two notable limitations when both the row and column loading matrices are estimated from a single block. Without loss of generality, consider an equal column-wise partition. First, the estimation accuracy of the column loading matrix decreases substantially faster than that of the row loading matrix. Specifically, the former decays at a rate proportional to the square of the number of splits, whereas the latter declines only linearly (see Section~\ref{sec3}). Second, because each computing node uses all data rows, localized heterogeneity may be obscured. These limitations motivate the development of improved distributed estimation methods.

This paper proposes a distributed estimation procedure for high-dimensional matrix-variate factor models, introducing a new partitioning regime. We first partition the observations column-wise according to their grouping structure and assign each subset to a node server, where the row loading matrix is estimated using $\alpha$-PCA \citep{chen2021}. These local estimates are transmitted to a central server, where they are aggregated and followed by a final PCA step to obtain a global estimator. In parallel, the observations are partitioned row-wise to estimate the column loading matrix. The factor matrices are then calculated based on the estimated loading matrices. This framework extends existing estimation methods for the matrix-variate factor model to efficiently handle high-dimensional heterogeneous data. Unlike existing distributed approaches, it preserves the latent matrix structure, thereby improving computational efficiency and information utilization. When data groupings are unknown, we recommend several clustering strategies for distributed data allocation, while the eigenvalue ratio-based method \citep{Lam_2012,ahn2013} is used to estimate the number of factors. We establish rigorous statistical guarantees under mild conditions, demonstrating the consistency of the proposed estimators as the sample size and data dimension at each computing unit diverge, and we further derive asymptotic normality of the loading matrix estimators. Extensive empirical validations, including simulations and applications to the Fama–French Stock Return data and the OECD CPI data, confirm the effectiveness of our method. Across these studies, our approach achieves substantially higher computational efficiency and better predictive performance, while maintaining accuracy comparable to that of $\alpha$-PCA.

From a statistical perspective, our method extends $\alpha$-PCA \citep{chen2021} to accommodate high-dimensional heterogeneous settings, and we further establish asymptotic normality for both their estimators and our proposed estimators. In related work, \cite{chen2025distributedtensorprincipalcomponent} proposed a distributed tensor PCA method for large-scale tensor datasets, motivated by considerations similar to ours in employing distributed estimation. However, their double-loop singular value decomposition (SVD) framework is less efficient and less streamlined. Our method is also connected to the distributed hierarchical factor model \citep{gao2021divideandconquerdistributedhierarchicalfactor}, particularly when both row and column loading matrices are estimated from column-wise (or row-wise) partitions. While the hierarchical model accommodates dispersed data and captures the commonality and heterogeneity across partitions, it converges more slowly and fails to preserve the latent matrix structure. Moreover, \cite{JIANG2027}, also based on the matrix factor model, exploited banded and general sparse structures in high-dimensional data to achieve dimensionality reduction, thereby enhancing interpretability and scalability. In contrast, our method imposes no specific structural assumptions and is applicable in more general settings.

The contributions of this paper are threefold. First, we propose a distributed estimation procedure for reducing the dimensionality of high-dimensional, large-scale, heterogeneous matrix-variate time series data using the matrix factor model. By preserving the matrix structure and employing the proposed partitioning regime, our method efficiently and accurately extracts low-dimensional latent factors while maintaining the overall common structure and respecting inter-group heterogeneity. Simulation results confirm that our approach substantially improves computational efficiency while maintaining high estimation accuracy, and real data applications further validate its predictive performance. Second, we establish rigorous statistical guarantees for the proposed method by deriving its asymptotic properties for the diverging sample size and data dimension at each computing unit. In particular, we obtain new results on the asymptotic normality of the $\alpha$-PCA loading estimators and derive the asymptotic normality of our loading estimators. Third, we extend the proposed method to accommodate unit-root nonstationary time series, enabling the identification of common stochastic trends, and provide corresponding statistical guarantees.

The rest of the paper is organized as follows. Section \ref{sec2} introduces our estimation methodology, and Section \ref{sec3} presents its theoretical properties. Our framework is extended to unit-root nonstationary data in Section~\ref{sec3.5}. Numerical studies based on simulated and real datasets are reported in Section \ref{sec4}. Section \ref{sec5} provides some concluding remarks. All technical proofs and additional numerical results are presented in an online supplementary Appendix. Throughout the article, we use the following notation. For a $p\times 1$ vector $\bv=(v_1,..., v_p)^\top,$  $||\bv||_2 =\|\bv^\top\|_2= (\sum_{i=1}^{p} v_i^2)^{1/2} $ denotes the Euclidean norm and $\|\bv\|_\infty=\max_{1\leq i\leq p}|v_i|$ denotes the $\ell_\infty$-norm. $\bI_p$ denotes a $p\times p$ identity matrix. For a matrix $\bH=(h_{ij})\in \mathbb{R}^{m\times n}$, $\|\bH\|_1=\max_j\sum_i|h_{ij}|$, $\|\bH\|_\infty=\max_{i,j}|h_{ij}|$,  $\|\bH\|_\mathsf{F}=\sqrt{\sum_{i,j}h_{ij}^2}$ denotes the Frobenius norm, and $\|\bH\|_2=\sqrt{\lambda_{\max} (\bH^\top \bH ) }$ denotes the operator norm, where $\lambda_{\max} (\cdot) $ denotes the largest eigenvalue of a matrix. $\bH^l\in \mathbb{R}^{1\times n}$ denotes the $l$-th row vector of a matrix $\bH$. For a matrix time series $\{\bH_t\in\mathbb{R}^{m\times n};\ t\in[T]\}$, $\overline{\bH}\in\mathbb{R}^{m\times n}$ denotes its sample mean and $[T]$ denotes the set $ \{1,2,...,T\}$. We also use the notation $a\asymp b$ to denote $a=O(b)$ and $b=O(a)$. 

\section{Model and Methodology}\label{sec2}

\subsection{Model Setting}\label{sec2.1}

Let $\bY_t$ be an observable $p\times q$-dimensional stationary time series for $t=1,2,...,T$, and assume it admits a latent matrix factor structure:
\begin{equation}\label{model1}
\mathbf{Y}_t=\mathbf{R}\mathbf{F}_t\mathbf{C}^\top+\mathbf{E}_t,\quad t=1,2,...,T,
\end{equation}
where $\mathbf{F}_t\in \mathbb{R}^{k\times r}$ is the common factor matrix with dimensions $k\ll p$ and $r\ll q$, $\mathbf{R}\in\mathbb{R}^{p\times k}$ and $\mathbf{C}\in\mathbb{R}^{q\times r}$ are the row and column loading matrices respectively, capturing the common covariance, and $\mathbf{E}_t\in \mathbb{R}^{p\times q}$ denotes the idiosyncratic component, assumed uncorrelated with $\mathbf{F}_t$.

It is well known that the triplet $[\mathbf{R},\mathbf{F}_t,\mathbf{C}]$ is not separately identifiable. Specifically, for any nonsingular matrices $\mathbf{H}_1\in \mathbb{R}^{k\times k}$ and $\mathbf{H}_2\in \mathbb{R}^{r\times r}$, the triplet $[\mathbf{R},\mathbf{F}_t,\mathbf{C}]$ is equivalent to $[\mathbf{R}\mathbf{H}_1,\mathbf{H}_1^{-1}\mathbf{F}_t(\mathbf{H}_2^{-1})^\top,\mathbf{C}\mathbf{H}_2]$ under model~(\ref{model1}). Following the restriction used in \cite{chen2021}, we assume that each row of the loading matrices is bounded, i.e.,
$$\left\|\mathbf{R}^{l_1}\right\|_2=\mathcal{O}(1)\quad \text{and}\quad \left\|\mathbf{C}^{l_2}\right\|_2=\mathcal{O}(1),\quad l_1=1,2,...,p,\quad l_2=1,2,...,q.$$
Our estimation procedure yields the triplet $[\widehat{\mathbf{R}},\widehat{\mathbf{F}}_t,\widehat{\mathbf{C}}]$, which serve as an estimator of $[\mathbf{R},\mathbf{F}_t,\mathbf{C}]$. Specifically, there exist nonsingular matrices $\mathbf{H}_\bR\in \mathbb{R}^{k\times k}$ and $\mathbf{H}_\bC\in \mathbb{R}^{r\times r}$, such that $\widehat{\bR}$ estimates $\mathbf{R}\mathbf{H}_\bR$, $\widehat{\bC}$ estimates $\mathbf{C}\mathbf{H}_\bC$, and $\widehat{\bF}_t$ estimates $\mathbf{H}_\bR^{-1}\mathbf{F}_t(\mathbf{H}_\bC^{-1})^\top$.

\subsection{Estimation Procedure for Low-dimensional Case}\label{sec2.2}

Several estimation methods have been developed for model (\ref{model1}). Among them, the $\alpha$-PCA method in \cite{chen2021} is computationally efficient and easy to implement, as it applies two-dimensional tensor PCA to statistics based on the sample mean and covariance. We first outline this estimation procedure and then introduce our distributed approach developed for high-dimensional observations. Throughout this section, we assume that $k$ and $r$ are known, with their estimation discussed later. Define the $p\times p$ and $q\times q$ aggregation matrices,
\begin{equation}\label{mr1}
    \widehat{\bM}_{\bR}^\alpha=\frac{1}{pq}\left[(1+\alpha)\overline{\mathbf{Y}}\overline{\mathbf{Y}}^\top+\frac{1}{T}\sum_{t=1}^T(\mathbf{Y}_{t}-\overline{\mathbf{Y}})(\mathbf{Y}_{t}-\overline{\mathbf{Y}})^\top\right]
\end{equation}
and
\begin{equation}\label{mc1}
 \widehat{\mathbf{M}}_{\mathbf{C}}^\alpha=\frac{1}{pq}\left[(1+\alpha)\overline{\mathbf{Y}}^\top\overline{\mathbf{Y}}+\frac{1}{T}\sum_{t=1}^T(\mathbf{Y}_{t}-\overline{\mathbf{Y}})^\top(\mathbf{Y}_{t}-\overline{\mathbf{Y}})\right],
\end{equation}
where $\alpha$ $\in [-1,+\infty)$ is a hyper-parameter controlling the weight of the two terms. The row and column loading matrix estimators, $\widehat{\mathbf{R}}_\alpha$ and $\widehat{\mathbf{C}}_\alpha$, are obtained as $\sqrt{p}$ times the top $k$ eigenvectors of $\widehat{\mathbf{M}}_{\mathbf{R}}^\alpha$ and $\sqrt{q}$ times the top $r$ eigenvectors of $\widehat{\mathbf{M}}_{\mathbf{C}}^\alpha$, respectively. Let $\tilde{\alpha}=\sqrt{\alpha+1}-1$, and define $\widetilde{\mathbf{Y}}_{t}=\mathbf{Y}_{t}+\tilde{\alpha}\overline{\mathbf{Y}}$, $\widetilde{\mathbf{F}}_{t}=\mathbf{F}_{t}+\tilde{\alpha}\overline{\mathbf{F}}$, and $ \widetilde{\mathbf{E}}_{t}=\mathbf{E}_{t}+\tilde{\alpha}\overline{\mathbf{E}}$, so that $\widetilde{\mathbf{Y}}_{t}=\bR\widetilde{\mathbf{F}}_{t}\bC^\top+\widetilde{\mathbf{E}}_{t}$. Under this definition, Equations (\ref{mr1}) and (\ref{mc1}) are simplified to
\begin{equation*}
    \widehat{\mathbf{M}}_{\mathbf{R}}^\alpha=\frac{1}{pqT}\sum_{t=1}^T \widetilde{\mathbf{Y}}_{t}\widetilde{\mathbf{Y}}_{t}^\top\quad \text{and}\quad
    \widehat{\mathbf{M}}_{\mathbf{C}}^\alpha=\frac{1}{pqT}\sum_{t=1}^T \widetilde{\mathbf{Y}}_{t}^\top\widetilde{\mathbf{Y}}_{t}.
\end{equation*}


\subsection{Estimation Procedure for High-dimensional Case}\label{sec2.3}

In this section, we introduce our distributed estimation procedure, starting with known $k$ and $r$. The first step is to estimate the row loading matrix $\bR$. We assume that $\bY_t$ admits a column-wise grouping structure,
$$\bY_t= \begin{pmatrix}\mathbf{Y}_{1t}& \mathbf{Y}_{2t} & \cdots & \mathbf{Y}_{s_1t}\end{pmatrix},$$
where $\mathbf{Y}_{it}\in \mathbb{R}^{p\times m_{i}}$ for $i\in[s_1]$ and $\sum_{i=1}^{s_1}m_i=q$. These blocks are assigned to $s_1$ computing nodes. Under model~(\ref{model1}), the block $\bY_{it}$ follows the sub-model,
\begin{equation}
    \mathbf{Y}_{it}=\mathbf{R}\mathbf{F}_t\mathbf{C}_i^\top+\mathbf{E}_{it},\quad t=1,2,...,T,
    \label{submodel1}
\end{equation} 
where $$\begin{pmatrix}\mathbf{C}_1^\top &\mathbf{C}_2^\top & \cdots &\mathbf{C}_{s_1}^\top\end{pmatrix}=\mathbf{C}^\top\quad\text{and}\quad\begin{pmatrix}\mathbf{E}_{1t}& \mathbf{E}_{2t}& \cdots& \mathbf{E}_{s_1t}\end{pmatrix}=\mathbf{E}_t,$$
with $\mathbf{C}_{i} \in \mathbb{R}^{m_{i}\times r}$ and $\mathbf{E}_{it} \in \mathbb{R}^{p\times m_{i}}$. At each node, we apply $\alpha$-PCA to its assigned data block to obtain a local estimator of $\bR$. Specifically, for each $i\in[s_1]$, the estimator $\widehat{\mathbf{R}}_i$ is constructed by multiplying $\sqrt{p}$ with the eigenvectors of $\widehat{\mathbf{M}}_{\mathbf{R}_i}$ corresponding to its largest $k$ eigenvalues in decreasing order, where
\begin{equation}\label{mri}
    \widehat{\mathbf{M}}_{\mathbf{R}_i}=\frac{1}{p m_i}\left[(1+\alpha)\overline{\mathbf{Y}}_i\overline{\bY}_i^\top+\frac{1}{T} \sum_{t=1}^T\left(\mathbf{Y}_{i t}-\overline{\mathbf{Y}}_i\right)\left(\mathbf{Y}_{i t}-\overline{\mathbf{Y}}_i\right)^\top\right]=\frac{1}{pm_iT}\sum_{t=1}^T\widetilde{\bY}_{it}\widetilde{\bY}_{it}^\top,
\end{equation}
with $\overline{\mathbf{Y}}_i=\frac{1}{T} \sum_{t=1}^T \mathbf{Y}_{it}$ and $\widetilde{\mathbf{Y}}_{it}=\mathbf{Y}_{it}+\tilde{\alpha}\overline{\mathbf{Y}}_i$. The local estimators $\{\wh{\bR}_i\}_{i=1}^{s_1}$ are transmitted to a center server and aggregated by 
\begin{equation}\label{mr}
\widehat{\mathbf{M}}_{\mathbf{R}}=\frac{1}{ps_1}\sum_{i=1}^{s_1} \wh{\bR}_i\wh{\bR}_i^\top.
\end{equation} 
The global estimator $\wh{\bR}$ is then obtained by multiplying $\sqrt{p}$ with the eigenvectors of $\widehat{\mathbf{M}}_{\mathbf{R}}$ corresponding to its largest k eigenvalues in decreasing order.

Next, we estimate the column loading matrix $\bC$ similarly. We assume that $\bY_t$ also admits a row-wise grouping structure,
$$\mathbf{Y}_t=\begin{pmatrix}\mathbf{Z}_{1t}^\top& \mathbf{Z}_{2t}^\top& \cdots& \mathbf{Z}_{s_2t}^\top\end{pmatrix}^\top,$$
where $\mathbf{Z}_{jt}\in \mathbb{R}^{ n_{j}\times q}$ for $j\in[s_2]$ and $\sum_{j=1}^{s_2}n_j=p$. These blocks are assigned to $s_2$ computing nodes. By the transpose of model (\ref{model1}), the block $\mathbf{Z}_{jt}$ follows another sub-model,
\begin{equation}
\mathbf{Z}_{jt}=\mathbf{R}_j\mathbf{F}_t\mathbf{C}^\top+\boldsymbol{\varepsilon}_{jt},\quad t=1,2,...,T,
    \label{submodel2}
\end{equation}
where 
$$\begin{pmatrix}\mathbf{R}_1^\top& \mathbf{R}_2^\top& \cdots& \mathbf{R}_{s_2}^\top\end{pmatrix}=\mathbf{R}^\top\quad\text{and}\quad\begin{pmatrix}\boldsymbol{\varepsilon}_{1t}^\top& \boldsymbol{\varepsilon}_{2t}^\top& \cdots &\boldsymbol{\varepsilon}_{s_2t}^\top\end{pmatrix}=\mathbf{E}_t^\top,$$
with $\mathbf{R}_{j} \in \mathbb{R}^{n_{j}\times k}$ and $\boldsymbol{\varepsilon}_{jt} \in \mathbb{R}^{n_{j}\times q}$. Each node computes a local estimator of $\bC$ using $\alpha$-PCA. Specifically, for each $j\in[s_2]$, the estimator $\wh{\bC}_j$ is formed by multiplying $\sqrt{q}$ with the largest $r$ eigenvectors of $\widehat{\mathbf{M}}_{\mathbf{C}_j}$, where
\begin{equation}\label{mci}
\widehat{\mathbf{M}}_{\mathbf{C}_j}=\frac{1}{q n_j}\left[(1+\alpha) \overline{\mathbf{Z}}_j^\top\overline{\mathbf{Z}}_j+\frac{1}{T} \sum_{t=1}^T(\mathbf{Z}_{j t}-\overline{\mathbf{Z}}_j)^\top(\mathbf{Z}_{j t}-\overline{\mathbf{Z}}_j)\right]=\frac{1}{qn_jT}\sum_{t=1}^T\widetilde{\bZ}_{jt}^\top\widetilde{\bZ}_{jt},
\end{equation}
with $\overline{\mathbf{Z}}_j=\frac{1}{T} \sum_{t=1}^T \mathbf{Z}_{jt}$ and $\widetilde{\mathbf{Z}}_{jt}=\mathbf{Z}_{jt}+\tilde{\alpha}\overline{\mathbf{Z}}_j$. The local estimators $\{\wh{\bC}_j\}_{j=1}^{s_2}$ are sent to a center server and aggregated by
\begin{equation}\label{mc}
\widehat{\mathbf{M}}_{\mathbf{C}}=\frac{1}{qs_2}\sum_{j=1}^{s_2} \wh{\bC}_j\wh{\bC}_j^\top.
\end{equation}
The global estimator $\wh{\bC}$ is then obtained by multiplying $\sqrt{q}$ with the largest $r$ eigenvectors of $\widehat{\mathbf{M}}_{\mathbf{C}}$. We summarize our estimation procedure in \textbf{Algorithm \ref{algorithm1}}, denoted as dPCA.

\begin{algorithm}
\caption{Distributed estimation for $\bR$ and $\bC$ (dPCA)}
\label{algorithm1}
\DontPrintSemicolon 
\SetKwBlock{DoParallel}{Do in parallel}{end}
\KwIn{Observations $\{\bY_t,t=1,2,...,T\}$}
\KwOut{Estimated loading matrices $\wh{\bR}$ and $\wh{\bC}$}
Partition $\bY_t$ column-wise into $\{\bY_{it}\}_{i=1}^{s_1}$, for $t=1,2,...,T$\\
Partition $\bY_t$ row-wise into $\{\bZ_{jt}\}_{j=1}^{s_2}$, for $t=1,2,...,T$\\

\DoParallel{
Obtain $\wh{\bR}_i$ from $\{\bY_{it},t=1,2,...,T\}$ employing $\alpha$-PCA, for $i=1,2,...,s_1$ \\
Obtain $\wh{\bC}_j$ from $\{\bZ_{jt},t=1,2,...,T\}$ employing $\alpha$-PCA, for $j=1,2,...,s_2$
}
Obtain $\wh{\bR}$ by multiplying $\sqrt{p}$ with the top $k$ eigenvectors of $\frac{1}{ps_1}\sum_{i=1}^{s_1} \wh{\bR}_i\wh{\bR}_i^\top$\\
Obtain $\wh{\bC}$ by multiplying $\sqrt{q}$ with the top $r$ eigenvectors of $\frac{1}{qs_2}\sum_{j=1}^{s_2} \wh{\bC}_j\wh{\bC}_j^\top$
\end{algorithm}

The proposed method substantially improves the computational efficiency of modeling high-dimensional time series data under model~\eqref{model1}. Calculating $\wh{\bM}_{\bR_i}$ and $\wh{\bM}_{\bC_j}$ requires $\mathcal{O}(p^2Tm_i)$ and $\mathcal{O}(q^2Tn_j)$ operations, respectively, for $i\in[s_1]$ and $j\in[s_2]$. The associated eigenvector matrices $\wh\bR_i$ and $\wh\bC_j$ are obtained using the Lanczos algorithm with computational costs $\mathcal{O}(p^2k)$ and $\mathcal{O}(q^2r)$. Similarly, Computing $\wh{\bM}_{\bR}$ and $\wh{\bM}_{\bC}$ requires $\mathcal{O}(p^2ks_1)$ and $\mathcal{O}(q^2rs_2)$ operations, followed by $\mathcal{O}(p^2k)$ and $\mathcal{O}(q^2r)$ to obtain $\wh\bR$ and $\wh\bC$. With parallel computation and sufficiently large $T$, the total complexity for estimating the loading matrices is $\mathcal{O}(p^2T\max_i m_i+q^2T\max_j n_j)$. In contrast, $\alpha$-PCA requires $\mathcal{O}(p^2Tq+q^2Tp)$ operations. Consequently, our method achieves considerable computational savings over $\alpha$-PCA. 

\subsection{Estimation of other unknown parameters\label{section_estimateftstkr}}

Using the loading matrix estimators $\wh{\bR}$ and $\wh{\bC}$ obtained from \textbf{Algorithm 1}, the common factor matrix $\bF_t$ is estimated by
\begin{equation}\label{estimateFt}
    \wh{\bF}_t=\frac{1}{pq}\wh{\bR}^\top\bY_t\wh{\bC}.
\end{equation}
The common component $\bS_t=\bR\bF_t\bC^\top$ is then estimated as
\begin{equation}\label{estimateSt}
    \wh{\bS}_t=\frac{1}{pq}\wh{\bR}\wh{\bR}^\top\bY_t\wh{\bC}\wh{\bC}^\top.
\end{equation}
We next estimate the factor numbers $k$ and $r$ using the eigenvalue ratio-based method of \cite{Lam_2012} and \cite{ahn2013}. In the low-dimensional case, $k$ is estimated by
\begin{equation}\label{estimate_k}
    \hat{k}_\alpha=\text{arg}\max_{1\le l\le k_{\max}} \frac{\widehat{\lambda}^\alpha_l}{\widehat{\lambda}^\alpha_{l+1}},
\end{equation}
where $\widehat{\lambda}^\alpha_1 \geq \cdots \geq \widehat{\lambda}^\alpha_p \geq 0$ are the ordered eigenvalues of $\widehat{\bM}_\bR^\alpha$, and $k_{\max}=\left \lfloor p/2 \right \rfloor $ following \citet{Lam_2012}. The estimator of $r$, $\wh{r}_\alpha$, is obtained similarly from the eigenvalue ratios of $\widehat{\bM}_\bC^\alpha$. In the high-dimensional case, the row factor number estimators $\wh{k}_i$ (on node $i$) and $\wh{k}$ (on central server) are calculated using the eigenvalue ratios of $\widehat{\bM}_{\bR_i}$ and $\widehat{\bM}_\bR$, respectively, for $i\in[s_1]$. Similarly, the column factor number estimators $\wh{r}_j$ (on node $j$) and $\wh{r}$ (on central server) are computed using the eigenvalue ratios of $\widehat{\bM}_{\bC_j}$ and $\widehat{\bM}_\bC$, respectively, for $j\in[s_2]$.

\subsection{Strategies for distributed data allocation\label{sec2.5}}

We examined in Section~\ref{sec2.3} how the numbers of columns $\{m_i\}_{i=1}^{s_1}$ and rows $\{n_j\}_{j=1}^{s_2}$ allocated to computing nodes affect computational efficiency, and we later evaluate their impact on estimation accuracy in Section~\ref{sec3}. When the data groupings are unknown, the assignment of rows or columns to each node, which also influences model performance, has not yet been discussed. We therefore discuss several strategies for distributed data allocation, which are further employed in the real data applications in Section~\ref{realexample_sec}.

The first strategy allocates data according to external covariates. For example, when modeling multinational trade data, countries may be ordered by geographic distance or GDP rankings, and trade items by trade volume or value added, and then sequentially assigned to computing nodes. The second strategy allocates data by clustering. To mitigate the scaling effect in PCA, we recommend clustering data based on variance, following \cite{gao2021divideandconquerdistributedhierarchicalfactor}. For the data rows $\{\bY_t^{l},l=1,2,...,p\}$, let $v_l$ denote the norm of the auto-covariance matrix of the $l$-th row. By sorting $\{v_l\}_{l=1}^p$ in decreasing order, the corresponding rows are reordered accordingly. The first $n_1$ rows are allocated to the first node, the next $n_2$ to the second node, and so on, until the last $n_{s_2}$ rows are assigned to the $s_2$-th node. This approach is computationally efficient for high-dimensional time series. Alternative clustering approaches may instead use the Euclidean distance, Manhattan distance, or sample autocorrelations \citep{CAIADO20062668}.

\section{Theoretical Properties}\label{sec3}

In this section, we establish the theoretical properties of the proposed method. We begin by introducing the assumptions required to derive these results. Most of them are adopted from \cite{chen2021} and \cite{yu2020}, or are standard in the factor model literature. Throughout, $C$ denotes a generic positive constant whose value may vary across contexts.

\begin{assumption}\label{assum1}
    The vectorized process $\{\operatorname{vec}(\mathbf{F}_t),\operatorname{vec}(\mathbf{E}_t)\}$ is stationary and $\alpha$-mixing with the mixing coefficient $\alpha(h)$ satisfying $\sum_{h=1}^{\infty} \alpha(h)^{1-2 / \gamma}<\infty$ for some $\gamma>2$, where
    $$\alpha(h)=\sup _\tau \sup _{A \in \mathcal{F}_{-\infty}^\tau, B \in \mathcal{F}_{\tau+h}^{\infty}}|P(A \cap B)-P(A) P(B)|,$$ 
    with the $\sigma$-field $\mathcal{F}_\tau^s$ generated by $\{(\operatorname{vec}(\mathbf{F}_t),\operatorname{vec}(\mathbf{E}_t)): \tau \leq t \leq s\}$.
\end{assumption} 

\begin{assumption}\label{assum2}
    For the common factor $\bF_t$ and the idiosyncratic component $\bE_t$, we assume (a) $\mathbb{E}\|\mathbf{F}_t\|_2^4 \leq C$, for $t \in[T]$;\\ (b) $\mathbb{E}[e_{t, i j}]=0$ and $\mathbb{E}|e_{t, i j}|^8 \leq C$, for $i \in[p]$, $j \in[q]$, and $t \in[T]$;\\ (c) $\bF_t$ and $\bE_t$ are uncorrelated, i.e., $\mathbb{E}[e_{t, i j} f_{s, l h}]=0$, for $l \in[k]$, $h \in[r]$, $i \in[p]$, $j \in[q]$, and $t, s \in[T]$.
\end{assumption}

\begin{assumption}\label{assum3}
    For the loading matrices $\mathbf{R}$ and $\mathbf{C}$, we assume (a) $\|\mathbf{R}^{l_1}\|_2=\mathcal{O}(1)$ and $\|\mathbf{C}^{l_2}\|_2=\mathcal{O}(1)$, for $ l_1\in[p]$ and $l_2\in[q]$;\\ (b) There exist positive definite matrices $\boldsymbol{\Omega}_{\bR_{i}} \in \mathbb{R}^{k\times k}$ and $\boldsymbol{\Omega}_{\bC_{j}} \in \mathbb{R}^{r\times r}$, such that, $\|n_{i}^{-1} \mathbf{R}_{i}^{\top} \mathbf{R}_{i}-\boldsymbol{\Omega}_{\bR_{i}}\|_2 \to 0$ uniformly for $\bR_i$ as $n_{i},q\to\infty$, and $\|m_j^{-1} \mathbf{C}_{j}^{\top} \mathbf{C}_{j}-\boldsymbol{\Omega}_{\bC_{j}}\|_2 \to 0$ uniformly for $\bC_j$ as $m_{j},p\to\infty$, for $i\in[s_2]$ and $j\in[s_1]$. 
\end{assumption}
This assumption leads directly to the following result.
\begin{proposition}    
  Under Assumption~\ref{assum3}, there exist positive definite matrices $\boldsymbol{\Omega}_\bR\in\mathbb{R}^{k\times k}$ and $\boldsymbol{\Omega}_\bC\in\mathbb{R}^{r\times r}$, such that, $\|p^{-1} \mathbf{R}^{\top} \mathbf{R}-\boldsymbol{\Omega}_\bR\|_2 \to 0$ and $\|q^{-1} \mathbf{C}^{\top} \mathbf{C}-\boldsymbol{\Omega}_\bC\|_2 \to 0$, as $p,q\to\infty$.
\end{proposition}
\begin{proof} 
  Define the positive definite matrix $\boldsymbol{\Omega}_\bR=\sum_{i=1}^{s_2}\frac{n_{i}}{p}\boldsymbol{\Omega}_{\bR_{i}}$. We have
  \begin{align*}    
    \left\|\frac{1}{p} \mathbf{R}^{\top} \mathbf{R}-\boldsymbol{\Omega}_\bR\right\|_2&=\left\|\frac{1}{p} \begin{pmatrix} \mathbf{R}_1^\top& \cdots& \mathbf{R}_{s_2}^\top \end{pmatrix} \begin{pmatrix}  \mathbf{R}_1\\ \cdots\\ \mathbf{R}_{s_2} \end{pmatrix}-\boldsymbol{\Omega}_\bR\right\|_2\\    
    &=\left\|\frac{1}{p} \sum_{i=1}^{s_2}\left(\mathbf{R}_i^{\top} \mathbf{R}_i-n_{i}\boldsymbol{\Omega}_{\bR_{i}}\right)\right\|_2   \\ &\le\sum_{i=1}^{s_2}\frac{n_{i}}{p}\left\|\frac{1}{n_i}\mathbf{R}_i^{\top} \mathbf{R}_i-\boldsymbol{\Omega}_{\bR_{i}}\right\|_2\longrightarrow 0.
  \end{align*} 
  The last step is based on Assumption~\ref{assum3}(b). The proof for $\bC$ is similar, thus omitted.
\end{proof}
\begin{assumption}\label{assum4}
For the idiosyncratic component $\bE_t$, we assume (a) $\|\mathbf{U}_\bE\|_1 \leq C$ and $\|\mathbf{V}_\bE\|_1 \leq C$, where $\mathbf{U}_\bE=\mathbb{E}[\frac{1}{q T} \sum_{t=1}^T \mathbf{E}_t \mathbf{E}_t^\top]$ and $\mathbf{V}_\bE=\mathbb{E}[\frac{1}{p T} \sum_{t=1}^T \mathbf{E}_t^\top \mathbf{E}_t]$;\\ (b) $\sum_{l=1}^p \sum_{h=1}^q|\mathbb{E}[e_{t, l j} e_{t, i h}]| \leq C$, for $i\in [p]$, $j \in [q]$, and $t \in [T]$;\\ (c) $\sum_{m=1}^p \sum_{s=1}^T \sum_{\substack{h=1 \\ h \neq j}}^q|\operatorname{cov}[e_{t, i j} e_{t, l j}, e_{s, i h} e_{s, m h}]| \leq C$, for $i,l\in [p]$, $j \in [q]$, and $t \in [T]$;\\ (d) $\sum_{m=1}^q \sum_{s=1}^T \sum_{\substack{l=1 \\ l \neq i}}^p|\operatorname{cov}[e_{t, i j} e_{t, i h}, e_{s, l j} e_{s, l m}]| \leq C$, for $i\in [p]$, $j,h \in [q]$, and $t \in [T]$.
\end{assumption}
\begin{assumption} \label{assum5}
    There exist $m>2$, $1<a, b<\infty$ with $1 / a+1 / b=1$, such that: \\
    (a) $\mathbb{E}[\|\mathbf{F}_{t}^l \|_2^{m a}] \leq C$, $\mathbb{E}[|\frac{1}{\sqrt{q}} \sum_{j=1}^q e_{t, i j}|^{m b}]=\mathcal{O}(1)$, and $\mathbb{E}[\|\frac{1}{\sqrt{q}} \sum_{j=1}^q \mathbf{C}^j e_{t, i j}\|_2^{m b}]=\mathcal{O}(1)$, for $l \in[k]$, $i \in[p]$, and $t \in[T]$; \\
    (b) $\mathbb{E}[\|(\mathbf{F}_{t}^\top)^h\|_2^{m a}] \leq C$, $\mathbb{E}[|\frac{1}{\sqrt{p}} \sum_{i=1}^p e_{t, i j}|^{m b}]=\mathcal{O}(1)$, and $\mathbb{E}[\|\frac{1}{\sqrt{p}} \sum_{i=1}^p \mathbf{R}^i e_{t, i j}\|_2^{m b}]=\mathcal{O}(1)$, for $h \in[r]$, $j \in[q]$, and $t \in[T]$; \\
    (c) $\mathbb{E}[|\frac{1}{\sqrt{p q}} \sum_{i=1}^p \sum_{j=1}^q e_{t, i j}|^{m b}]=\mathcal{O}(1)$ and $\mathbb{E}[\|\frac{1}{\sqrt{p q}} \sum_{i=1}^p \sum_{j=1}^q (\mathbf{R}^i)^{\top} \mathbf{C}^j e_{t, i j}\|_\mathsf{F}^{m b}]=\mathcal{O}(1)$, for $t \in[T]$.
\end{assumption}

Assumption~\ref{assum1} imposes standard temporal dependence conditions on $\bE_t$ and $\bF_t$ (e.g., \citeauthor{WANG2019231}, \citeyear{WANG2019231}), ensuring their asymptotic independence and controlling the variance of partial sums. Assumption~\ref{assum2}(a)-(b) guarantee that $\bF_t$ and $\bE_t$ are well defined, respectively, as is typical in factor model settings (e.g., \citeauthor{bai2002}, \citeyear{bai2002}), while Assumption~\ref{assum2}(c) ensures their mutual uncorrelatedness. Assumption~\ref{assum3}, similar to Assumption B in \cite{bai2002},  guarantees the non-trivial contribution of each factor to the common variance. This assumption justifies our data partitioning regime and is sufficient for Assumption 3 in \cite{chen2021}. Assumption~\ref{assum4}(a)-(b) allows for weak dependence of $\bE_t$ across time, rows, and columns, extending the assumptions in \cite{Stock01122002} and \cite{forni2005} to matrix data, and aligning with Assumption D.2 in \cite{yu2020}. Assumption~\ref{assum4}(c)-(d), similar to Assumption M1(c) in \cite{Stock01122002}, controls the magnitude of the fourth moments of $\bE_t$. Finally, Assumption~\ref{assum5}, adopted from \cite{chen2021}, further regulates the dependence of $\bF_t$ and $\bE_t$ across rows and columns.

To present our theoretical results, we introduce additional notations. Let $\widehat{\bV}_{\bR_i} \in \mathbb{R}^{k\times k}$ and $\widehat{\bV}_{\bC_j} \in \mathbb{R}^{r\times r}$ denote diagonal matrices consisting of the $k$ and $r$ largest eigenvalues of $\widehat{\mathbf{M}}_{\mathbf{R}_{i}}$ and $\widehat{\mathbf{M}}_{\mathbf{C}_{j}}$, respectively, in decreasing order, for $i\in[s_1]$ and $j\in[s_2]$. Define the invertible auxiliary matrices $\bH_{\bR_i} \in \mathbb{R}^{k\times k}$ and $\bH_{\bC_j} \in \mathbb{R}^{r\times r}$ by
$$\bH_{\bR_i}=\frac{1}{pm_iT}\sum_{t=1}^T\widetilde{\bF}_t\bC_i^\top\bC_i\widetilde{\bF}_t^\top\bR^\top\wh{\bR}_i\widehat{\bV}_{\bR_i}^{-1}\ \ \text{and}\ \ \bH_{\bC_j}=\frac{1}{qn_jT}\sum_{t=1}^T\widetilde{\bF}_t^\top\bR_j^\top\bR_j\widetilde{\bF}_t\bC^\top\wh{\bC}_j\widehat{\bV}_{\bC_j}^{-1}.$$
By Theorem 3.3 of \cite{yu2020}, under Assumptions~\ref{assum1}-\ref{assum5} with known $k$ and $r$, as $m_i,n_j,T\to\infty$, it follows that 
$$\frac{1}{p}\left\|\widehat{\mathbf{R}}_i-\mathbf{R H}_{\bR_i}\right\|_\mathsf{F}^2=\mathcal{O}_p\left(\frac{1}{m_iT}+\frac{1}{p^2}\right)\ \ \text{and}\ \ \frac{1}{q}\left\|\widehat{\mathbf{C}}_j-\mathbf{C H}_{\bC_j}\right\|_\mathsf{F}^2=\mathcal{O}_p\left(\frac{1}{n_jT}+\frac{1}{q^2}\right).$$
Next, let $\widehat{\bV}_{\bR} \in \mathbb{R}^{k\times k}$ and $\widehat{\bV}_{\bC} \in \mathbb{R}^{r\times r}$ denote diagonal matrices consisting of the $k$ and $r$ largest eigenvalues of $\widehat{\mathbf{M}}_{\mathbf{R}}$ and $\widehat{\mathbf{M}}_{\mathbf{C}}$, respectively, in decreasing order. Define the invertible auxiliary matrices $\bH_{\bR} \in \mathbb{R}^{k\times k}$ and $\bH_{\bC} \in \mathbb{R}^{r\times r}$ by 
$$\bH_{\bR}=\frac{1}{ps_1}\sum_{i=1}^{s_1} \bH_{\bR_i} \bH_{\bR_i}^\top\bR^\top\wh{\bR}\widehat{\bV}_{\bR}^{-1}\quad\text{and}\quad\bH_{\bC}=\frac{1}{qs_2}\sum_{j=1}^{s_2} \bH_{\bC_j} \bH_{\bC_j}^\top\bC^\top\wh{\bC}\widehat{\bV}_{\bC}^{-1}.$$ Following \cite{chen2021}, an additional assumption is required to control the magnitudes of $\bH_{\bR_i}$ and $\bH_{\bC_j}$, which further ensures the boundedness for $\bH_{\bR}$ and $\bH_{\bC}$.

\begin{assumption} \label{assum6} We assume (a) The eigenvalues of $\boldsymbol{\Omega}_{\bR}\widetilde{\boldsymbol{\Sigma}}_{\mathbf{FC}_i}$ and $\boldsymbol{\Omega}_{\bC}\widetilde{\boldsymbol{\Sigma}}_{\mathbf{FR}_j}$ are distinct, for $i\in[s_1]$, $j\in[s_2]$; (b) The eigenvalues of $\boldsymbol{\Omega}_\bR\boldsymbol{\Sigma}_{\bH_\bR}$ and $\boldsymbol{\Omega}_\bC\boldsymbol{\Sigma}_{\bH_\bC}$ are distinct. 
\\The matrices $\widetilde{\boldsymbol{\Sigma}}_{\mathbf{FC}_i}=\frac{1}{m_i}\mathbb{E}[\widetilde{\mathbf{F}}_t\mathbf{C}_i^\top\mathbf{C}_i\widetilde{\mathbf{F}}_t^\top]$ and $\boldsymbol{\Sigma}_{\bH_\bR}=\frac{1}{s_1}\sum_{i=1}^{s_1}\mathbf{Q}_{\mathbf{R}_i}^{-1}(\mathbf{Q}_{\mathbf{R}_i}^{-1})^\top$, where $\mathbf{Q}_{\mathbf{R}_i}=\mathbf{V}_{\mathbf{R}_i}^{1/2}\Psi_{\mathbf{R}_i}^\top\widetilde{\boldsymbol{\Sigma}}_{\mathbf{FC}_i}^{-1/2}$. Here $\mathbf{V}_{\mathbf{R}_i}$ is the diagonal matrix of eigenvalues (ordered decreasingly) of $\widetilde{\boldsymbol{\Sigma}}_{\mathbf{FC}_i}^{1/2}\boldsymbol{\Omega}_{\mathbf{R}}\widetilde{\boldsymbol{\Sigma}}_{\mathbf{FC}_i}^{1/2}$, and $\Psi_{\mathbf{R}_i}$ denotes the corresponding eigenvectors satisfying $\Psi_{\mathbf{R}_i}^\top\Psi_{\mathbf{R}_i}=\mathbf{I}$. The matrices $\widetilde{\boldsymbol{\Sigma}}_{\mathbf{FR}_j}$ and $\boldsymbol{\Sigma}_{\bH_\bC}$ are similarly defined.
\end{assumption}

Assumption~\ref{assum6}(a) controls the magnitudes of the auxiliary matrices, and Assumption~\ref{assum6} also ensures the asymptotic stability of $\bH_\bR$ and $\bH_\bC$
. For convenience, define $\overline{m}=\max_{1\leq i \leq s_1}m_i$, $\overline{n}=\max_{1\leq j \leq s_2}n_j$, $\underline{m}=\min_{1\leq i \leq s_1} m_i$, and $\underline{n}=\min_{1\leq j \leq s_2} n_j$. Under these assumptions and notations, we establish the following results.

\begin{theorem} \label{them1}
Under Assumptions~\ref{assum1}-\ref{assum6}(a) with known $k$ and $r$, as $\underline{m}, \underline{n}, T \rightarrow \infty$, we have
\begin{equation*}
    \frac{1}{p}\left\|\widehat{\mathbf{R}}-\mathbf{R H}_\bR\right\|_\mathsf{F}^2=\mathcal{O}_p\left(\frac{1}{\underline{m}T}+\frac{1}{p^2}\right)
\quad \text{and} \quad
    \frac{1}{q}\left\|\widehat{\mathbf{C}}-\mathbf{C H}_\bC\right\|_\mathsf{F}^2=\mathcal{O}_p \left(\frac{1}{\underline{n}T}+\frac{1}{q^2}\right).
\end{equation*}
Consequently,
\begin{equation*}
    \frac{1}{p}\left\|\widehat{\mathbf{R}}-\mathbf{R H}_\bR\right\|_2^2=\mathcal{O}_p\left(\frac{1}{\underline{m}T}+\frac{1}{p^2}\right)
\quad \text{and} \quad
    \frac{1}{q}\left\|\widehat{\mathbf{C}}-\mathbf{C H}_\bC\right\|_2^2=\mathcal{O}_p \left(\frac{1}{\underline{n}T}+\frac{1}{q^2}\right).
\end{equation*}
\end{theorem}

\begin{theorem} \label{them2}
Under Assumptions~\ref{assum1}-\ref{assum6}(a) with known $k$ and $r$, as $\underline{m}, \underline{n}, T \rightarrow \infty$, we have
\begin{equation*}
\left\|\widehat{\mathbf{F}}_t-\mathbf{H}_\bR^{-1} \mathbf{F}_t (\mathbf{H}_\bC^{-1})^{\top}\right\|_2=\mathcal{O}_p\left(\frac{1}{\sqrt{\underline{m}T}}+\frac{1}{\sqrt{\underline{n}T}}+\frac{1}{p}+\frac{1}{q}\right).
\end{equation*}
\end{theorem}

\begin{theorem} \label{them3}
Under Assumptions~\ref{assum1}-\ref{assum6}(a) with known $k$ and $r$, as $\underline{m}, \underline{n}, T \rightarrow \infty$, we have
\begin{equation*}
    \wh{\bS}_{t,ij}-\bS_{t,ij}=\mathcal{O}_p\left(\frac{1}{\sqrt{\underline{m}T}}+\frac{1}{\sqrt{\underline{n}T}}+\frac{1}{p}+\frac{1}{q}\right),\quad i=1,2,...,p,\quad j=1,2,...,q.
\end{equation*}
\end{theorem}

\begin{remark}
 (\rom{1}) Theorems~\ref{them1}-\ref{them3} imply that the estimation accuracy of our method improves as $\underline{m}$ and $\underline{n}$ increase, while section~\ref{sec2.3} shows that smaller $\overline{m}$ and $\overline{n}$ reduce the computational cost. Together, these results reveal a trade-off between computational efficiency and estimation accuracy when allocating data across computing nodes;\\
 (\rom{2}) Table~\ref{table_convergence} summarizes the convergence rates of loading matrix estimators from our procedure, \cite{chen2021}, and \cite{yu2020}. The rates established in Theorem~\ref{them1} are comparable to those approaches. In particular, when $s_1=s_2=1$, our convergence rates match those of most methods, whereas the projection estimator converges faster due to its projection techniques. Notably, our method can also use the projection estimator in place of the $\alpha$-PCA estimator on each node, yielding faster convergence rates similar to theirs. The detailed derivations are omitted. These comparisons indicate that our distributed framework still preserves high estimation accuracy compared with single-machine approaches;\\
 (\rom{3}) Our estimation procedure within each node can also be viewed as a projection framework. However, the estimators of \citet{yu2020} converge faster by employing initial estimators $\widehat{\bR}^\text{I}$ and $\widehat{\bC}^\text{I}$ as projection matrices. In both methods, the row loading matrix is estimated by multiplying the largest $k$ eigenvalues of the column sample covariance matrix by $\sqrt{p}$. Our method constructs the covariance matrix on the first node as
 $$\wh{\bM}_{\bR_1}=\frac{1}{pm_1T}\sum_{t=1}^T\widetilde{\bY}_{1t}\widetilde{\bY}_{1t}^\top=\frac{1}{pm_1T}\sum_{t=1}^T\widetilde{\bY}_{t}\begin{pmatrix}\mathbf{e}_{1} \cdots\mathbf{e}_{m_1}\end{pmatrix}\begin{pmatrix}\mathbf{e}_{1} \cdots\mathbf{e}_{m_1}\end{pmatrix}^\top\widetilde{\bY}_{t}^\top,$$
 where $\be_{l}\in\mathbb{R}^{q\times1}$ denotes the $l$-th standard basis vector for $l\in[m_1]$. In contrast, \citet{yu2020} compute the covariance matrix as
 $$\wh{\bM}_{\bR}^{\text{P}}=\frac{1}{Tpq^2}\sum_{t=1}^T{\bY}_{t}\widehat{\bC}^\text{I}(\widehat{\bC}^\text{I})^\top{\bY}_{t}^\top,$$
 where $\widehat{\bC}^\text{I}\overset{\mathcal{P}}{\longrightarrow}\bC\bH_\bC^\text{I}$ under the assumptions $\frac{1}{q}\bC^\top\bC\overset{\mathcal{P}}{\longrightarrow}\bI$ and $\bH_\bC^\text{I}(\bH_\bC^\text{I})^\top\overset{\mathcal{P}}{\longrightarrow}\bI$. The rapid convergence of the projection estimator relies on this consistency property of the projection matrix $\widehat{\bC}^\text{I}$, which does not hold for our projection matrix.
\end{remark}

\begin{table}[h]
\caption{\raggedright Convergence rates of estimators.}
\label{table_convergence}
\resizebox{\linewidth}{!}{
\centering
\begin{tabular}{lp{0.25cm}lllll}
\hline\addlinespace 
\multicolumn{1}{c}{\begin{tabular}[c]{@{}c@{}}Estimation\\[0.05em] Procedure\end{tabular}} & & \multicolumn{1}{c}{$\alpha$-PCA} & \multicolumn{1}{c}{IE} & \multicolumn{1}{c}{PE} & \multicolumn{1}{c}{dPCA} & \multicolumn{1}{c}{\begin{tabular}[c]{@{}c@{}}dPCA\\[0.05em] $(s_1=s_2=1)$\end{tabular}} \\[1em] \hline \addlinespace 
$\frac{1}{p}\|\widehat{\mathbf{R}}-\mathbf{R H}_\bR\|_\mathsf{F}^2$ & & $\mathcal{O}_p(\frac{1}{qT}+\frac{1}{p})$ & $\mathcal{O}_p(\frac{1}{qT}+\frac{1}{p^2})$ & $\mathcal{O}_p(\frac{1}{qT}+\frac{1}{p^2q^2}+\frac{1}{p^2T^2})$ & $\mathcal{O}_p(\frac{1}{\underline{m}T}+\frac{1}{p^2})$ & $\mathcal{O}_p(\frac{1}{qT}+\frac{1}{p^2})$ \\[1em]
$\frac{1}{q}\|\widehat{\mathbf{C}}-\mathbf{C H}_\bC\|_\mathsf{F}^2$ & & $\mathcal{O}_p(\frac{1}{pT}+\frac{1}{q})$ & $\mathcal{O}_p(\frac{1}{pT}+\frac{1}{q^2})$ & $\mathcal{O}_p(\frac{1}{pT}+\frac{1}{p^2q^2}+\frac{1}{q^2T^2})$ & $\mathcal{O}_p(\frac{1}{\underline{n}T}+\frac{1}{q^2})$ & $\mathcal{O}_p(\frac{1}{pT}+\frac{1}{q^2})$ \\[0.5em] \hline
\end{tabular}
}
\begin{minipage}{\linewidth}
\footnotesize
\vspace{0.1cm}
    This table reports the convergence rate of loading matrix estimators. ``$\alpha$-PCA'' refers to the estimators introduced by \cite{chen2021}, while ``IE'' and ``PE'' denote the initial and projection estimators proposed by \cite{yu2020}, respectively.
\end{minipage}

\end{table}

Next, we study asymptotic distributions of $\widehat{\bR}-\bR\bH_\bR$ and $\widehat{\bC}-\bC\bH_\bC$. We begin by introducing several notations used in the subsequent analysis. Let $\mathbf{V}_{\mathbf{R}}$ and $\mathbf{V}_{\mathbf{C}}$ denote the diagonal matrices consisting of eigenvalues of ${\boldsymbol{\Sigma}}_{\mathbf{HR}}^{1/2}\boldsymbol{\Omega}_{\mathbf{R}}{\boldsymbol{\Sigma}}_{\mathbf{HR}}^{1/2}$ and ${\boldsymbol{\Sigma}}_{\mathbf{HC}}^{1/2}\boldsymbol{\Omega}_{\mathbf{C}}{\boldsymbol{\Sigma}}_{\mathbf{HC}}^{1/2}$, respectively, in decreasing order. Let $\Psi_{\mathbf{R}}$ and $\Psi_{\mathbf{C}}$ denote the corresponding eigenvector matrices, satisfying $\Psi_{\mathbf{R}}^\top\Psi_{\mathbf{R}}=\mathbf{I}$ and $\Psi_{\mathbf{C}}^\top\Psi_{\mathbf{C}}=\mathbf{I}$. Define $\mathbf{Q}_{\mathbf{R}}=\mathbf{V}_{\mathbf{R}}^{1/2}\Psi_{\mathbf{R}}^\top{\boldsymbol{\Sigma}}_{\mathbf{HR}}^{-1/2}$ and $\mathbf{Q}_{\mathbf{C}}=\mathbf{V}_{\mathbf{C}}^{1/2}\Psi_{\mathbf{C}}^\top{\boldsymbol{\Sigma}}_{\mathbf{HC}}^{-1/2}$. For each $i,l\in[p]$ and $j\in[s_1]$, define
\begin{align*}
        \bG_{i}^{j,l}=\begin{cases}\sqrt{\frac{\overline{m}}{m_j}}\Big[(\mathbf{R}^{i}\bQ_{\bR_j}^{-1})^\top\bR^{i}+\sum_{\tau=1}^p(\mathbf{R}^\tau\bQ_{\bR_j}^{-1})^\top\bR^\tau\Big]\bQ_{\bR}^{-1}{\bV}_\bR^{-1},&\quad l={i}, \\\sqrt{\frac{\overline{m}}{m_j}}(\mathbf{R}^{i}\bQ_{\bR_j}^{-1})^\top\bR^l\bQ_{\bR}^{-1}{\bV}_\bR^{-1},&\quad l\neq {i}.\end{cases}
\end{align*}
Analogously, for each $i',l'\in[q]$ and $j'\in[s_2]$, define
\begin{align*}
        \bP_{i'}^{j',l'}=\begin{cases}\sqrt{\frac{\overline{n}}{n_{j'}}}\Big[(\mathbf{C}^{i'}\bQ_{\bC_{j'}}^{-1})^\top\bC^{i'}+\sum_{\tau'=1}^q(\mathbf{C}^{\tau'}\bQ_{\bC_{j'}}^{-1})^\top\bC^{\tau'}\Big]\bQ_{\bC}^{-1}{\bV}_\bC^{-1},&\quad l'={i'}, \\\sqrt{\frac{\overline{n}}{n_{j'}}}(\mathbf{C}^{i'}\bQ_{\bC_j}^{-1})^\top\bC^{l'}\bQ_{\bC}^{-1}{\bV}_\bC^{-1},&\quad l'\neq {i'}.\end{cases}
\end{align*}
The following theorem establishes the asymptotic normality of $\widehat{\bR}-\bR\bH_\bR$ and $\widehat{\bC}-\bC\bH_\bC$.
\begin{theorem} \label{them4}
    Under Assumptions~\ref{assum1}-\ref{assum6} with known $k$ and $r$, as $\underline{m}, \underline{n},T\to \infty$, we have
    \begin{enumerate}
        \item[(a)] If ${\sqrt{\overline{m}T}}/{p}=\mathbf{o}(1)$ and ${\overline{m}}/{\underline{m}}=\mathcal{O}(1)$, then for any $i\in[p]$,
        $$\sqrt{\overline{m}T}\left(\widehat{\bR}^{i}-\bR^{i}\bH_\bR\right)^\top \overset{\mathcal{D}}{\longrightarrow}\mathcal{N}\left(\mathbf{0},\boldsymbol{\Sigma}_{\bR_{i}}\right),$$
        where the asymptotic covariance matrix is 
        \begin{align*}
        \boldsymbol{\Sigma}_{\bR_{i}}=&\frac{1}{p^2s_1^2}\sum_{j_1,j_2=1}^{s_1}\sum_{l_1,l_2=1}^p(\bG_{i}^{j_1,l_1})^\top\bV_{\bR_{j_1}}^{-1}\bQ_{\bR_{j_1}}\left(\mathbf{\Phi}_{\bR_{j_1,j_2,l_1,l_2}}^{1,1}+\alpha\mathbf{\Phi}_{\bR_{j_1,j_2,l_1,l_2}}^{1,2}\overline{\bF}^\top\right.\\
        &\left.+\alpha\overline{\bF}\mathbf{\Phi}_{\bR_{j_1,j_2,l_1,l_2}}^{2,1}+\alpha^2\overline{\bF}\mathbf{\Phi}_{\bR_{j_1,j_2,l_1,l_2}}^{2,2}\overline{\bF}^\top\right)\bQ_{\bR_{j_2}}^\top\bV_{\bR_{j_2}}^{-1}\bG_{i}^{j_2,l_2},
        \end{align*}
        with
        
         $$  \mathbf{\Phi}_{\bR_{j_1,j_2,l_1,l_2}}^{1,1}=\underset{\underline{m}, T \rightarrow \infty}{\operatorname{plim}} \frac{1}{\sqrt{m_{j_1}m_{j_2}} T} \sum_{t,s=1}^{T} \mathbb{E}\Big[\mathbf{F}_{t} \mathbf{C}_{j_1}^\top (\mathbf{E}_{j_1t}^{l_1})^\top \mathbf{E}_{j_2s}^{l_2} \mathbf{C}_{j_2} \mathbf{F}_{s}^\top\Big], $$
        $$    \mathbf{\Phi}_{\bR_{j_1,j_2,l_1,l_2}}^{1,2}=(\mathbf{\Phi}_{\bR_{j_1,j_2,l_1,l_2}}^{2,1})^\top=\underset{\underline{m}, T \rightarrow \infty}{\operatorname{plim}} \frac{1}{\sqrt{m_{j_1}m_{j_2}} T} \sum_{t,s=1}^{T} \mathbb{E}\Big[\mathbf{F}_{t} \mathbf{C}_{j_1}^\top (\mathbf{E}_{j_1t}^{l_1})^\top \mathbf{E}_{j_2s}^{l_2} \mathbf{C}_{j_2}\Big], $$
        and
        $$    \mathbf{\Phi}_{\bR_{j_1,j_2,l_1,l_2}}^{2,2}=\underset{\underline{m}, T \rightarrow \infty}{\operatorname{plim}} \frac{1}{\sqrt{m_{j_1}m_{j_2}} T} \sum_{t,s=1}^{T} \mathbb{E}\Big[\mathbf{C}_{j_1}^\top (\mathbf{E}_{j_1t}^{l_1})^\top \mathbf{E}_{j_2s}^{l_2} \mathbf{C}_{j_2}\Big]. $$
        
        \item[(b)] If ${\sqrt{\overline{n}T}}/{q}=\mathbf{o}(1)$ and ${\overline{n}}/{\underline{n}}=\mathcal{O}(1)$, then for any ${i'}\in[q]$,
        $$\sqrt{\overline{n}T}\left(\widehat{\bC}^{i'}-\bC^{i'}\bH_\bC\right)^\top \overset{\mathcal{D}}{\longrightarrow}\mathcal{N}\left(\mathbf{0},\boldsymbol{\Sigma}_{\bC_{i'}}\right),$$
        where the asymptotic covariance matrix is 
        \begin{align*}
        \boldsymbol{\Sigma}_{\bC_{i'}}=&\frac{1}{q^2s_2^2}\sum_{j_1',j_2'=1}^{s_2}\sum_{l_1',l_2'=1}^q(\bP_{i'}^{j_1',l_1'})^\top\bV_{\bC_{j_1'}}^{-1}\bQ_{\bC_{j_1'}}\left(\mathbf{\Phi}_{\bC_{j_1',j_2',l_1',l_2'}}^{1,1}+\alpha\mathbf{\Phi}_{\bC_{j_1',j_2',l_1',l_2'}}^{1,2}\overline{\bF}\right.\\
        &\left.+\alpha\overline{\bF}^\top\mathbf{\Phi}_{\bC_{j_1',j_2',l_1',l_2'}}^{2,1}+\alpha^2\overline{\bF}^\top\mathbf{\Phi}_{\bC_{j_1',j_2',l_1',l_2'}}^{2,2}\overline{\bF}\right)\bQ_{\bC_{j_2'}}^\top\bV_{\bC_{j_2'}}^{-1}\bP_{i'}^{j_2',l_2'},
        \end{align*}
        with $\mathbf{\Phi}_{\bC_{j_1',j_2',l_1',l_2'}}^{1,1}$, $\mathbf{\Phi}_{\bC_{j_1',j_2',l_1',l_2'}}^{1,2}$, $\mathbf{\Phi}_{\bC_{j_1',j_2',l_1',l_2'}}^{2,1}$, and $\mathbf{\Phi}_{\bC_{j_1',j_2',l_1',l_2'}}^{2,2}$ defined similarly. 
        \end{enumerate}
\end{theorem}

\begin{remark}
    Theorem~\ref{them4} presents the asymptotic normality of our loading matrix estimators, consistent with the fact that the estimators of \cite{chen2021} and \cite{yu2020} are also asymptotically normal. Our proof first extends their results to derive the joint asymptotic distribution of all rows of $\widehat{\bR}_j-{\bR}\bH_{{\bR}_j}$, $j\in[s_1]$. We then express $\widehat{\bR}-{\bR}\bH_{{\bR}}$ as a linear combination of $\{\widehat{\bR}_j-{\bR}\bH_{{\bR}_j}\}_{j}$, yielding the asymptotic normality of its row vectors. The argument for $\widehat{\bC}$ follows similarly.
\end{remark}


Finally, we introduce consistent estimators for the asymptotic covariance matrices in Theorem~\ref{them4}. For any $j_1,j_2\in[s_1]$, $j_1',j_2'\in[s_2]$, $l_1,l_2\in[p]$, and $l_1',l_2'\in[q]$, define  
$$\boldsymbol{\Sigma}_{\bR_{j_1,j_2,l_1,l_2}}=\bQ_{\bR_{j_1}}(\mathbf{\Phi}_{\bR_{j_1,j_2,l_1,l_2}}^{1,1}+\alpha\mathbf{\Phi}_{\bR_{j_1,j_2,l_1,l_2}}^{1,2}\overline{\bF}^\top+\alpha\overline{\bF}\mathbf{\Phi}_{\bR_{j_1,j_2,l_1,l_2}}^{2,1}+\alpha^2\overline{\bF}\mathbf{\Phi}_{\bR_{j_1,j_2,l_1,l_2}}^{2,2}\overline{\bF}^\top)\bQ_{\bR_{j_2}}^\top,$$
and
$$\boldsymbol{\Sigma}_{\bC_{j_1',j_2',l_1',l_2'}}=\bQ_{\bC_{j_1'}}(\mathbf{\Phi}_{\bC_{j_1',j_2',l_1',l_2'}}^{1,1}+\alpha\mathbf{\Phi}_{\bC_{j_1',j_2',l_1',l_2'}}^{1,2}\overline{\bF}+\alpha\overline{\bF}^\top\mathbf{\Phi}_{\bC_{j_1',j_2',l_1',l_2'}}^{2,1}+\alpha^2\overline{\bF}^\top\mathbf{\Phi}_{\bC_{j_1',j_2',l_1',l_2'}}^{2,2}\overline{\bF})\bQ_{\bC_{j_2'}}^\top.$$
With these definitions, the asymptotic covariance matrices simplify to
$$\boldsymbol{\Sigma}_{\bR_{i}}=\frac{1}{p^2s_1^2}\sum_{j_1,j_2=1}^{s_1}\sum_{l_1,l_2=1}^p (\bG_{i}^{j_1,l_1})^\top\bV_{\bR_{j_1}}^{-1}\boldsymbol{\Sigma}_{\bR_{j_1,j_2,l_1,l_2}}\bV_{\bR_{j_2}}^{-1}\bG_{i}^{j_2,l_2},$$
and 
$$\boldsymbol{\Sigma}_{\bC_{i'}}=\frac{1}{q^2s_2^2}\sum_{j_1',j_2'=1}^{s_2}\sum_{l_1',l_2'=1}^q(\bP_{i'}^{j_1',l_1'})^\top\bV_{\bC_{j_1'}}^{-1}\boldsymbol{\Sigma}_{\bC_{j_1',j_2',l_1',l_2'}}\bV_{\bC_{j_1'}}^{-1}\bP_{i'}^{j_1',l_1'}.$$
We estimate $\boldsymbol{\Sigma}_{\bR_{j_1,j_2,l_1,l_2}}$ and $\boldsymbol{\Sigma}_{\bC_{j_1',j_2',l_1',l_2'}}$ using the heteroskedasticity and autocorrelation consistent (HAC) estimators \citep{Newey1986}. Specifically, for some $v\in\mathbb{N}$ satisfying $v/(\overline{m}T)^{1/4}=\mathbf{o}(1)$, the HAC estimator of $\boldsymbol{\Sigma}_{\bR_{j_1,j_2,l_1,l_2}}$ is defined as
$$\hat{\boldsymbol{\Sigma}}_{\bR_{j_1,j_2,l_1,l_2}}=\bD_{\bR_{j_1,j_2,l_1,l_2}}^{0}+\sum_{\tau=1}^v\left(1-\frac{\tau}{1+v}\right)\left[\bD_{\bR_{j_1,j_2,l_1,l_2}}^{\tau}+(\bD_{\bR_{j_1,j_2,l_1,l_2}}^{\tau})^\top\right],$$
where
\begin{align*}
    &\bD_{\bR_{j_1,j_2,l_1,l_2}}^{\tau}=\frac{1}{\sqrt{m_{j_1}m_{j_2}} T} \sum_{t=1+\tau}^{T}\begin{pmatrix}\bI&\alpha\overline{\wh{\bF}}_{j_1,t}\end{pmatrix}\times\\&\ \ \ \begin{pmatrix}
    \wh{\bF}_{j_1,t}(\wh{\bC}_{j_1}^*)^\top(\wh{\bE}_{j_1t}^{l_1})^\top\wh{\bE}_{j_2(t-\tau)}^{l_2}\wh{\bC}_{j_2}^*\wh{\bF}_{j_2,t-\tau}&\wh{\bF}_{j_1,t}(\wh{\bC}_{j_1}^*)^\top(\wh{\bE}_{j_1t}^{l_1})^\top\wh{\bE}_{j_2(t-\tau)}^{l_2}\wh{\bC}_{j_2}^*\\(\wh{\bC}_{j_1}^*)^\top(\wh{\bE}_{j_1t}^{l_1})^\top\wh{\bE}_{j_2(t-\tau)}^{l_2}\wh{\bC}_{j_2}^*\wh{\bF}_{j_2,t-\tau}&(\wh{\bC}_{j_1}^*)^\top(\wh{\bE}_{j_1t}^{l_1})^\top\wh{\bE}_{j_2(t-\tau)}^{l_2}\wh{\bC}_{j_2}^*\end{pmatrix}\begin{pmatrix}\bI\\\alpha\overline{\wh{\bF}}_{j_2,t}^\top\end{pmatrix}.
\end{align*}
Here $\wh{\bC}_{j}^*$ is an $\alpha$-PCA estimator of $\bC_j$ from the $j$-th sub-model (\ref{submodel1}),  $\wh{\bF}_{j,t}=\frac{1}{pm_j}\wh{\bR}_j^\top\bY_{jt}\wh{\bC}_{j}^*$, and  $\wh{\bE}_{jt}=\bY_{jt}-\wh{\bR}_j\wh{\bF}_{j,t}(\wh{\bC}_{j}^*)^\top$, for $j\in[s_1]$. The estimator $\hat{\boldsymbol{\Sigma}}_{\bC_{j_1',j_2',l_1',l_2'}}$ is defined similarly. The next theorem establishes the consistency of $\boldsymbol{\Sigma}_{\bR_{i}}$ and $\boldsymbol{\Sigma}_{\bC_{i'}}$.
\begin{theorem}
    \label{them5}
    Under Assumptions~\ref{assum1}-\ref{assum6} with known $k$ and $r$, as $\underline{m},\underline{n},T\to \infty$, $\wh{\boldsymbol{\Sigma}}_{\bR_{i}}$ and $\wh{\boldsymbol{\Sigma}}_{\bC_{i'}}$ are consistent estimators of $\boldsymbol{\Sigma}_{\bR_{i}}$ and $\boldsymbol{\Sigma}_{\bC_{i'}}$, respectively, for any $i\in[p]$ and $i'\in[q]$, where
    $$\wh{\boldsymbol{\Sigma}}_{\bR_{i}}=\frac{1}{p^2s_1^2}\sum_{j_1,j_2=1}^{s_1}\sum_{l_1,l_2=1}^p(\wh{\bG}_{i}^{j_1,l_1})^\top\wh{\bV}_{\bR_{j_1}}^{-1} \hat{\boldsymbol{\Sigma}}_{\bR_{j_1,j_2,l_1,l_2}}\wh{\bV}_{\bR_{j_2}}^{-1}\wh{\bG}_{i}^{j_2,l_2},$$
    and 
    $$\wh{\boldsymbol{\Sigma}}_{\bC_{i'}}=\frac{1}{q^2s_2^2}\sum_{j_1',j_2'=1}^{s_2}\sum_{l_1',l_2'=1}^q(\wh{\bP}_{i'}^{j_1',l_1'})^\top\wh{\bV}_{\bC_{j_1'}}^{-1}\wh{\boldsymbol{\Sigma}}_{\bC_{j_1',j_2',l_1',l_2'}}\wh{\bV}_{\bC_{j_1'}}^{-1}\wh{\bP}_{i'}^{j_1',l_1'}.$$
    The matrices $\wh{\bG}_{i}^{j,l}$ and $\wh{\bP}_{i'}^{j',l'}$ are given by
    \begin{align*}
    \wh{\bG}_{i}^{j,l}=\begin{cases}\sqrt{\frac{\overline{m}}{m_j}}\Big[(\wh{\mathbf{R}}^{i}_j)^\top\wh{\bR}^{i}+\sum_{\tau=1}^p(\wh{\mathbf{R}}^\tau_j)^\top\wh{\bR}^\tau\Big]\wh{\bV}_\bR^{-1},&\quad l={i}, \\\sqrt{\frac{\overline{m}}{m_j}}(\wh{\mathbf{R}}^{i}_j)^\top\wh{\bR}^l\wh{\bV}_\bR^{-1},&\quad l\neq {i},\end{cases}
    \end{align*}
    and 
    \begin{align*}
    \wh{\bP}_{i'}^{j',l'}=\begin{cases}\sqrt{\frac{\overline{n}}{n_{j'}}}\Big[(\wh{\mathbf{C}}^{i'}_{j'})^\top\wh{\bC}^{i'}+\sum_{\tau'=1}^q(\wh{\mathbf{C}}^{\tau'}_{j'})^\top\wh{\bC}^{\tau'}\Big]\wh{\bV}_\bC^{-1},&\quad l'={i'}, \\\sqrt{\frac{\overline{n}}{n_{j'}}}(\wh{\mathbf{C}}^{i'}_{j'})^\top\wh{\bC}^{l'}\wh{\bV}_\bC^{-1},&\quad l'\neq {i'}.\end{cases}
    \end{align*}
\end{theorem}

\section{Extension to Unit-Root Nonstationary Time Series}\label{sec3.5}

In this section, we apply the proposed method to unit-root nonstationary time series. Let the observations $\{\bY_t,t=1,2,..., T\}$ be an $\text{I}(1)$ process, that is, $\bY_t-\bY_{t-1}$ is stationary. Following \cite{PENA20061237}, we assume that $\bY_t$ admits the latent matrix factor structure, $$\bY_t=\bR\bF_t\bC^\top+\bE_t,\quad t=1,2,...,T,$$ where the common component captures all stochastic trends. The common factor $\bF_t$ is assumed to be $\text{I}(1)$, and the idiosyncratic component $\bE_t$ is stationary with zero mean. This assumption is weaker than that imposed in \citet{PENA20061237}.

We next show that the proposed method remains valid under the above assumptions. The numbers of factors $k$ and $r$ are assumed to be known, as their estimators introduced in Section~\ref{sec2} are consistent by Theorem 3.3 of \cite{li2025factormodelsmatrixvaluedtime}. To establish the consistency of the estimated loading matrices, factor matrix, and common component, we introduce the following notations for the unit-root setting. For each $i\in[s_1]$ and $j\in[s_2]$, let $\widehat{\bV}_{\bR_i}^* \in \mathbb{R}^{k\times k}$ and $\widehat{\bV}_{\bC_j}^* \in \mathbb{R}^{r\times r}$ denote diagonal matrices consisting of the $k$ and $r$ largest eigenvalues of $\frac{1}{T}\widehat{\mathbf{M}}_{\mathbf{R}_{i}}$ and $\frac{1}{T}\widehat{\mathbf{M}}_{\mathbf{C}_{j}}$, respectively, in decreasing order. Define the auxiliary matrices 
$$\bH_{\bR_i}^*=\frac{1}{pm_iT^2}\sum_{t=1}^T\widetilde{\bF}_t\bC_i^\top\bC_i\widetilde{\bF}_t^\top\bR^\top\wh{\bR}_i(\widehat{\bV}_{\bR_i}^*)^{-1},$$ and $$\bH_{\bC_j}^*=\frac{1}{qn_jT^2}\sum_{t=1}^T\widetilde{\bF}_t^\top\bR_j^\top\bR_j\widetilde{\bF}_t\bC^\top\wh{\bC}_j(\widehat{\bV}_{\bC_j}^*)^{-1}.$$
Lemma 5 in Appendix A shows that our estimators $\wh{\bR}_i$ and $\wh{\bC}_j$ consistently estimate $\bR\bH_{\bR_i}^*$ and $\bC\bH_{\bC_j}^*$, respectively. Next, let $\widehat{\bV}_{\bR} \in \mathbb{R}^{k\times k}$ and $\widehat{\bV}_{\bC} \in \mathbb{R}^{r\times r}$ denote diagonal matrices consisting of the $k$ and $r$ largest eigenvalues of $\widehat{\mathbf{M}}_{\mathbf{R}}$ and $\widehat{\mathbf{M}}_{\mathbf{C}}$, respectively, in decreasing order. Define the auxiliary matrices
$$\bH_{\bR}^*=\frac{1}{ps_1}\sum_{i=1}^{s_1} \bH_{\bR_i}^* (\bH_{\bR_i}^*)^\top\bR^\top\wh{\bR}\widehat{\bV}_{\bR}^{-1}\quad\text{and} \quad\bH_{\bC}^*=\frac{1}{qs_2}\sum_{j=1}^{s_2} \bH_{\bC_j}^* (\bH_{\bC_j}^*)^\top\bC^\top\wh{\bC}\widehat{\bV}_{\bC}^{-1}.$$
We will show that our estimators $\wh{\bR}$ and $\wh{\bC}$ consistently estimate $\bR\bH_{\bR}^*$ and $\bC\bH_{\bC}^*$, respectively. To this end, the following assumptions for the unit-root setting are required.

\begin{assumption}\label{assum7}
    Let the stationary process $\bw_t=\operatorname{vec}(\mathbf{F}_t)-\operatorname{vec}(\mathbf{F}_{t-1})\in\mathbb{R}^{kr}$, for $t=2,3,..., T$. We assume that $\bw_t=\sum_{s=0}^\infty\bG_s\bu_{t-s}$, where $\bu_t$ is a $kr$-dimensional i.i.d. random vector with zero mean, positive definite covariance matrix $\boldsymbol{\Sigma}_\bu$, and finite fourth moments. The coefficient matrices $\{\bG_s\}$ are $kr\times kr$-dimensional and satisfy $$\sum_{s=0}^\infty s\|\bG_s\|_2<\infty,\quad \overline{\bG}\overset{\triangle}{=}\sum_{s=0}^\infty \bG_s \succ0,$$ where $\succ0$ denotes positive definiteness.
\end{assumption}

\begin{assumption}\label{assum8}
For the idiosyncratic component $\bE_t$, we assume (a) The process $\operatorname{vec}(\mathbf{E}_t)$ is stationary, independent of $\bF_t$, and $\alpha$-mixing with the mixing coefficient $\alpha(h)$ satisfying $\sum_{h=1}^{\infty} \alpha(h)^{1-2 / \gamma}<\infty$, where $\gamma>2$ is given in Assumption~\ref{assum1}; \\(b) $\mathbb{E}[e_{t, i j}]=0$ and $\mathbb{E}|e_{t, i j}|^8 \leq C$, for $i \in[p]$, $j \in[q]$, and $t \in[T]$; \\(c) $\mathbb{E}\|\frac{1}{\sqrt{q}} \sum_{j=1}^q \mathbf{C}^j e_{t, i j}\|_2^2=\mathcal{O}(1)$ for $i\in[p]$ and $t\in[T]$, $\mathbb{E}\|\frac{1}{\sqrt{p}} \sum_{i=1}^p \mathbf{R}^i e_{t, i j}\|_2^2=\mathcal{O}(1)$ for $j\in[q]$ and $t\in[T]$, and $\mathbb{E}\|\frac{1}{\sqrt{pq}} \sum_{i=1}^p\sum_{j=1}^q (\mathbf{R}^i)^\top\mathbf{C}^j e_{t, i j}\|_\mathsf{F}^2=\mathcal{O}(1)$ for $t\in[T]$.
\end{assumption}

\begin{assumption} \label{assum9} We assume the eigenvalues of $\boldsymbol{\Omega}_{\bR}^{1/2}\widetilde{\boldsymbol{\Sigma}}_{\mathbf{FC}_i}^*\boldsymbol{\Omega}_{\bR}^{1/2}$ and $\boldsymbol{\Omega}_{\bC}^{1/2}\widetilde{\boldsymbol{\Sigma}}_{\mathbf{FR}_j}^*\boldsymbol{\Omega}_{\bC}^{1/2}$ are both positive, bounded, and distinct, for $i\in[s_1]$ and $j\in[s_2]$. The matrices $\widetilde{\boldsymbol{\Sigma}}_{\mathbf{FC}_i}^*=\int_0^1 \widetilde{\bW}(u)\boldsymbol{\Omega}_{\bC_i}\widetilde{\bW}(u)^\top\text{d}u$ and $\widetilde{\boldsymbol{\Sigma}}_{\mathbf{FR}_j}^*=\int_0^1 \widetilde{\bW}(u)^\top\boldsymbol{\Omega}_{\bR_j}\widetilde{\bW}(u)\text{d}u$, where $\widetilde{\bW}(u)=\bW(u)+\frac{\widetilde{\alpha}}{T}\sum_{t=1}^T\bW(t/T)$ and $\bW(\cdot)$ is an $k\times r$ matrix of Brownian motions with the covariance of $\text{vec}(\bW(\cdot))$ being $\overline{\bG}\boldsymbol{\Sigma}_\bu\overline{\bG}^\top$.
\end{assumption}

Assumption~\ref{assum7}, adapted from \cite{li2025factormodelsmatrixvaluedtime}, controls the magnitude of $\bw_t$ and $\bF_t$ and ensures that $\bF_t$ satisfies the basic condition of unit-root processes discussed in \cite{BARIGOZZI2021455}. Assumption~\ref{assum8}(a)-(b) replaces the conditions on $\bE_t$ in Assumptions~\ref{assum1}-\ref{assum2}, while Assumption~\ref{assum8}(c) corresponds to Assumption~\ref{assum5}. Similar to Assumption~\ref{assum6}, Assumption~\ref{assum9} controls the magnitude of $\bH_{\bR_i}^*$ and $\bH_{\bC_j}^*$, thereby also bounding $\bH_{\bR}^*$ and $\bH_{\bC}^*$. The distinct eigenvalue condition also guarantees the asymptotic stability of $\bH_{\bR_i}^*$ and $\bH_{\bC_j}^*$, as well as of $\bH_{\bR}^*$ and $\bH_{\bC}^*$.

The following theorems establish the consistency of our estimators for unit-root processes, mirroring the results for stationary processes in Theorems~\ref{them1}–\ref{them3}.

\begin{theorem} \label{them6}
Under Assumptions~\ref{assum3}-\ref{assum4} and \ref{assum7}-\ref{assum9} with known $k$ and $r$, as $\underline{m}, \underline{n}, T \rightarrow \infty$, we have
\begin{equation*}
    \frac{1}{p}\left\|\widehat{\mathbf{R}}-\mathbf{R H}_\bR^*\right\|_\mathsf{F}^2=\mathcal{O}_p\left(\frac{1}{\underline{m}T^2}+\frac{1}{p^2T^2}\right)\quad \text{and}\quad 
    \frac{1}{q}\left\|\widehat{\mathbf{C}}-\mathbf{C H}_\bC^*\right\|_\mathsf{F}^2=\mathcal{O}_p\left(\frac{1}{\underline{n}T^2}+\frac{1}{q^2T^2}\right).
\end{equation*}
Consequently,
\begin{equation*}
    \frac{1}{p}\left\|\widehat{\mathbf{R}}-\mathbf{R H}_\bR^*\right\|_2^2=\mathcal{O}_p\left(\frac{1}{\underline{m}T^2}+\frac{1}{p^2T^2}\right)\quad \text{and}\quad 
    \frac{1}{q}\left\|\widehat{\mathbf{C}}-\mathbf{C H}_\bC^*\right\|_2^2=\mathcal{O}_p\left(\frac{1}{\underline{n}T^2}+\frac{1}{q^2T^2}\right).
\end{equation*}
\end{theorem}

\begin{theorem} \label{them7}
Under Assumptions~\ref{assum3}-\ref{assum4} and \ref{assum7}-\ref{assum9} with known $k$ and $r$, as $\underline{m}, \underline{n}, T \rightarrow \infty$, we have
\begin{equation*}
\left\|\widehat{\mathbf{F}}_t-(\mathbf{H}_\bR^*)^{-1} \mathbf{F}_t [(\mathbf{H}_\bC^*)^{-1}]^{\top}\right\|_2=\mathcal{O}_p\left(\frac{1}{\sqrt{\underline{m}}T}+\frac{1}{\sqrt{\underline{n}}T}+\frac{1}{\sqrt{pq}}\right).
\end{equation*}
\end{theorem}

\begin{theorem} \label{them8}
Under Assumptions~\ref{assum3}-\ref{assum4} and \ref{assum7}-\ref{assum9} with known $k$ and $r$, as $\underline{m}, \underline{n}, T \rightarrow \infty$, we have
\begin{equation*}
    \wh{\bS}_{t,ij}-\bS_{t,ij}=\mathcal{O}_p\left(\frac{1}{\sqrt{\underline{m}}T}+\frac{1}{\sqrt{\underline{n}}T}+\frac{1}{\sqrt{pq}}\right),\quad  i=1,2,...,p,\quad j=1,2,...,q.
\end{equation*}
\end{theorem}

\begin{remark}
    (\rom{1}) Theorems~\ref{them6}-\ref{them8} imply that our procedure achieves higher accuracy for unit-root nonstationary processes than for stationary processes. This improvement is driven by the stronger trending signals contained in $\bF_t$, a phenomenon also noted in \cite{li2025factormodelsmatrixvaluedtime};\\
    (\rom{2}) The convergence rates of our loading matrix estimators are comparable to those in \cite{li2025factormodelsmatrixvaluedtime}, which conducts estimation on a single machine. Under the assumptions adopted in this paper, their estimators (mPCA) satisfy
    $$\frac{1}{p}\left\|\widehat{\mathbf{R}}-\mathbf{R H}_\bR^*\right\|_\mathsf{F}^2=\mathcal{O}_p\left(\frac{1}{qT^2}+\frac{1}{pT^2}\right)\quad\text{and}\quad\frac{1}{q}\left\|\widehat{\mathbf{C}}-\mathbf{C H}_\bC^*\right\|_\mathsf{F}^2=\mathcal{O}_p \left(\frac{1}{pT^2}+\frac{1}{qT^2}\right).$$
    This comparison indicates that our distributed procedure substantially improves computational efficiency while incurring only a slight loss in convergence rate.
\end{remark}

\section{Numerical Results}\label{sec4}

In this section, we assess the performance of our method in finite samples through simulation studies and real data applications. 

\subsection{Simulation studies}

In the simulation studies, the observations $\{\mathbf{Y}_t,t=1,2,...,T\}$ are generated from model~(\ref{model1}), $\mathbf{Y}_t=\mathbf{R}\mathbf{F}_t\mathbf{C}^\top+\mathbf{E}_t$, where the entries of $\bR$ and $\bC$ are independently drawn from $\mathcal{U}(-1, 1)$ with $({k},{r})=(3, 3)$. The generation of $\bF_t$ and $\bE_t$, along with the choices of $p$ and $q$, is described later since they vary across experiments. We evaluate the estimation accuracy, asymptotic normality, and computational efficiency of our estimators, focusing on the following aspects:
\begin{enumerate}
    \item Accuracy of $\wh{k}$ and $\wh{r}$. The row and column factor numbers, $k$ and $r$, are estimated by $\widehat{k}$ and $\widehat{r}$, as described in Section~\ref{section_estimateftstkr}. We report the frequencies of $(\widehat{k},\widehat{r})$.

    \item Accuracy of $\{\widehat{\bR}_i\}_{i=1}^{s_1}$ and $\{\widehat{\bC}_j\}_{j=1}^{s_2}$. For any nonsingular matrix $\bA$ and $\bB$, the column space distance is defined as
    \begin{equation} \label{functionD}
    \mathcal{D}(\bA,\bB)=\left\|\bB(\bB^\top\bB)^{-1}\bB^\top-\bA(\bA^\top\bA)^{-1}\bA^\top\right\|_\mathsf{F}.
     \end{equation}
    The accuracy of $\widehat{\bR}_i$ and $\widehat{\bC}_j$ is evaluated using $\mathcal{D}(\bR,\widehat{\bR}_i)$ and $\mathcal{D}(\bC,\widehat{\bC}_j)$, for any $i\in[s_1]$ and $j\in[s_2]$. We report the average and maximum distances between the local estimators and the true loading matrices.

    \item Accuracy of $\widehat{\bR}$, $\widehat{\bC}$, $\widehat{\bF}_t$ and $\widehat{\bS}_t$. According to Theorems~\ref{them1}-\ref{them3}, $\widehat{\bR}$, $\widehat{\bC}$, and $\widehat{\bF}_t$ estimate rotated ${\bR}$, ${\bC}$, and ${\bF}_t$, respectively, while $\widehat{\bS}_t$ directly estimates ${\bS}_t$. The accuracy of $\widehat{\bR}$, $\widehat{\bC}$, and $\widehat{\bF}_t$ is measured using 
    $\frac{1}{p} \| \hat{\mathbf{R}}-\mathbf{RH_R} \|_2^2$, $\frac{1}{q} \| \hat{\mathbf{C}}-\mathbf{CH_C} \|_2^2$, and $\frac{1}{T}\sum_{t=1}^T \| \widehat{\mathbf{F}}_t-\mathbf{H}_\bR^{-1}\mathbf{F}_t(\mathbf{H}_\bC^{-1})^\top \|_2^2$, where $\bH_\bR$ and $\bH_\bC$ are computed as described in Section~\ref{sec3}. The accuracy of $\widehat{\bS}_t$ is assessed using $\frac{1}{T}\sum_{t=1}^T \| \widehat{\mathbf{S}}_t-\mathbf{S}_t \|_2^2$. 

    \item Asymptotic normality of $\widehat{\bR}-\bR\bH_{\bR}$. We examine the asymptotic normality of the first row, $\sqrt{\overline{m}T}(\widehat{\bR}^1-\bR^1\bH_\bR)$. Following Theorems~\ref{them4}-\ref{them5}, its asymptotic covariance matrix is estimated by $\widehat{\Sigma}_{\bR_1}$. Consequently, each element of $\widehat{\Sigma}_{\bR_1}^{-1/2}\sqrt{\overline{m}T}(\widehat{\bR}^1-\bR^1\bH_\bR)^\top$ is expected to follow the standard normal distribution. For illustration, we present the density histogram of its first component.

   \item Computational efficiency. We compare the computing time required to estimate the loading matrices using our method and $\alpha$-PCA.
\end{enumerate}

\subsubsection{Accuracy of estimators\label{estimatingconvergence}}

First, we assess the estimation accuracy of our method. Following the AR(1) processes used in \citet{WANG2019231}, we generate $\bF_t$ and $\bE_t$ with controlled temporal dependence:
\begin{align} \label{generatedata1}
    \mathrm{vec}(\mathbf{F}_t)=0.1\mathrm{vec}(\mathbf{F}_{t-1})+\boldsymbol{\eta}_t,\quad\mathrm{vec}(\mathbf{E}_t)=\psi\mathrm{vec}(\mathbf{E}_{t-1})+\boldsymbol{\eta}'_t,
\end{align}
where $\boldsymbol{\eta}_t$ and $\boldsymbol{\eta}'_t$ are noise vectors with independent elements drawn from $\mathcal{N}(0,0.99)$ and $\mathcal{N}(0,1-\psi^2)$, respectively. The initial states are set to $\mathrm{vec}(\mathbf{F}_0)=\boldsymbol{0}$ and $\mathrm{vec}(\mathbf{E}_0)=\boldsymbol{0}$. To evaluate the impact of temporal dependence on estimation accuracy, we consider $\psi=0.1$ and $\psi=0.5$, representing weaker and stronger temporal dependence in $\mathbf{E}_{t}$.

We first consider the performance of estimating $(k,r)$, as inaccurate estimates of the number of factors can substantially impair subsequent estimation. Table~\ref{table_kr} reports the frequencies of our estimators (on the central server) and $\alpha$-PCA estimators under several settings. As expected, accuracy improves as the data dimensions $(p,q)$ and sample size $T$ increase and the temporal dependence $\psi$ decreases. Across all configurations of $p$, $q$, $T$, and $\psi$, our method achieves estimation accuracy comparable to that of $\alpha$-PCA. 

\begin{table}
\caption{\raggedright Frequencies of the estimated factor numbers.}
\label{table_kr}
\centering
\resizebox{\textwidth}{!}{
\Large
\begin{tabular}{p{5.5cm}p{2cm}p{2cm}p{1cm}p{2cm}p{2cm}p{1cm}p{2cm}p{2cm}p{1cm}p{2cm}p{2cm}}
\hline\addlinespace 
 & \multicolumn{2}{c}{$(\widehat{k},\widehat{r})=(3,3)$} &  & \multicolumn{2}{c}{$(\widehat{k},\widehat{r})=(3,2)$} &  & \multicolumn{2}{c}{$(\widehat{k},\widehat{r})=(2,3)$} &  & \multicolumn{2}{c}{Others} \\ \cline{2-3} \cline{5-6} \cline{8-9} \cline{11-12} 
 & Ours & $\alpha$-PCA &  & Ours & $\alpha$-PCA &  & Ours & $\alpha$-PCA &  & Ours & $\alpha$-PCA \\ \hline
\multicolumn{4}{l}{$(p,q)=(20,20)$} &  &  &  &  &  &  &  &  \\
$\psi=0.1$ &  &  &  &  &  &  &  &  &  &  &  \\
$T=10$ & \textbf{0.560} & 0.485 &  & 0.125 & 0.090 &  & 0.110 & 0.125 &  & 0.205 & 0.300 \\
$T=15$ & 0.730 & \textbf{0.760} &  & 0.065 & 0.045 &  & 0.090 & 0.060 &  & 0.115 & 0.135 \\
$T=20$ & \textbf{0.880} & 0.815 &  & 0.010 & 0.070 &  & 0.040 & 0.060 &  & 0.070 & 0.055 \\
$T=25$ & \textbf{0.890} & 0.860 &  & 0.030 & 0.060 &  & 0.035 & 0.040 &  & 0.045 & 0.040 \\
$\psi=0.5$ &  &  &  &  &  &  &  &  &  &  &  \\
$T=10$ & 0.475 & \textbf{0.570} &  & 0.165 & 0.095 &  & 0.085 & 0.135 &  & 0.275 & 0.200 \\
$T=15$ & 0.690 & \textbf{0.735} &  & 0.090 & 0.085 &  & 0.080 & 0.065 &  & 0.140 & 0.115 \\
$T=20$ & \textbf{0.800} & 0.790 &  & 0.050 & 0.060 &  & 0.050 & 0.085 &  & 0.100 & 0.065 \\
$T=25$ & 0.815 & \textbf{0.870} &  & 0.075 & 0.050 &  & 0.065 & 0.030 &  & 0.045 & 0.050 \\ \hline
\multicolumn{4}{l}{$(p,q)=(20,50)$} &  &  &  &  &  &  &  &  \\
$\psi=0.1$ &  &  &  &  &  &  &  &  &  &  &  \\
$T=10$ & \textbf{0.740} & 0.560 &  & 0.050 & 0.010 &  & 0.135 & 0.130 &  & 0.075 & 0.100 \\
$T=15$ & \textbf{0.940} & 0.890 &  & 0.010 & 0.020 &  & 0.020 & 0.050 &  & 0.030 & 0.040 \\
$T=20$ & \textbf{0.950} & 0.920 &  & 0.000 & 0.000 &  & 0.035 & 0.055 &  & 0.015 & 0.025 \\
$T=25$ & \textbf{0.950} & 0.945 &  & 0.000 & 0.000 &  & 0.045 & 0.040 &  & 0.005 & 0.015 \\
$\psi=0.5$ &  &  &  &  &  &  &  &  &  &  &  \\
$T=10$ & 0.660 & \textbf{0.720} &  & 0.080 & 0.045 &  & 0.150 & 0.125 &  & 0.110 & 0.110 \\
T=15 & 0.860 & \textbf{0.875} &  & 0.015 & 0.020 &  & 0.090 & 0.070 &  & 0.035 & 0.035 \\
T=20 & \textbf{0.960} & 0.950 &  & 0.000 & 0.000 &  & 0.025 & 0.035 &  & 0.015 & 0.015 \\
T=25 & 0.940 & \textbf{0.955} &  & 0.000 & 0.000 &  & 0.050 & 0.040 &  & 0.010 & 0.005 \\ \hline
\multicolumn{4}{l}{$(p,q)=(50,50)$} &  &  &  &  &  &  &  &  \\
$\psi=0.1$ &  &  &  &  &  &  &  &  &  &  &  \\
$T=10$ & \textbf{0.960} & 0.950 &  & 0.020 & 0.035 &  & 0.010 & 0.010 &  & 0.000 & 0.005 \\
$T=15$ & \textbf{0.995} & \textbf{0.995} &  & 0.005 & 0.000 &  & 0.000 & 0.000 &  & 0.000 & 0.005 \\
$T=20$ & \textbf{1.000} & \textbf{1.000} &  & 0.000 & 0.000 &  & 0.000 & 0.000 &  & 0.000 & 0.000 \\
$T=25$ & \textbf{1.000} & \textbf{1.000} &  & 0.000 & 0.000 &  & 0.000 & 0.000 &  & 0.000 & 0.000 \\
$\psi=0.5$ &  &  &  &  &  &  &  &  &  &  &  \\
$T=10$ & 0.960 & \textbf{0.965} &  & 0.005 & 0.010 &  & 0.030 & 0.020 &  & 0.005 & 0.005 \\
$T=15$ & \textbf{1.000} & 0.990 &  & 0.000 & 0.010 &  & 0.000 & 0.000 &  & 0.000 & 0.000 \\
$T=20$ & \textbf{1.000} & \textbf{1.000} &  & 0.000 & 0.000 &  & 0.000 & 0.000 &  & 0.000 & 0.000 \\
$T=25$ & \textbf{1.000} & \textbf{1.000} &  & 0.000 & 0.000 &  & 0.000 & 0.000 &  & 0.000 & 0.000 \\ \hline
\multicolumn{4}{l}{$(p,q)=(100,100)$} &  &  &  &  &  &  &  &  \\
$\psi=0.1$ &  &  &  &  &  &  &  &  &  &  &  \\
$T=10$ & \textbf{1.000} & 0.995 &  & 0.000 & 0.000 &  & 0.000 & 0.005 &  & 0.000 & 0.000 \\
$T=15$ & \textbf{1.000} & \textbf{1.000} &  & 0.000 & 0.000 &  & 0.000 & 0.000 &  & 0.000 & 0.000 \\
$T=20$ & \textbf{1.000} & \textbf{1.000} &  & 0.000 & 0.000 &  & 0.000 & 0.000 &  & 0.000 & 0.000 \\
$T=25$ & \textbf{1.000} & \textbf{1.000} &  & 0.000 & 0.000 &  & 0.000 & 0.000 &  & 0.000 & 0.000 \\
$\psi=0.5$ &  &  &  &  &  &  &  &  &  &  &  \\
$T=10$ & 0.985 & \textbf{0.995} &  & 0.000 & 0.005 &  & 0.015 & 0.000 &  & 0.000 & 0.000 \\
$T=15$ & \textbf{1.000} & \textbf{1.000} &  & 0.000 & 0.000 &  & 0.000 & 0.000 &  & 0.000 & 0.000 \\
$T=20$ & \textbf{1.000} & \textbf{1.000} &  & 0.000 & 0.000 &  & 0.000 & 0.000 &  & 0.000 & 0.000 \\
$T=25$ & \textbf{1.000} & \textbf{1.000} &  & 0.000 & 0.000 &  & 0.000 & 0.000 &  & 0.000 & 0.000 \\ \hline
\end{tabular}
}
\begin{minipage}{\textwidth}
    \vspace{0.1cm}
    \footnotesize
    This table reports the frequency of the estimated factor numbers $(\wh{k},\wh{r})$ using our method and $\alpha$-PCA, with true values $(k,r)=(3,3)$. Data are generated from model (\ref{model1}), $\bY_t=\bR\bF_t\bC^\top+\bE_t$, where $\bF_t$ and $\bE_t$ follow the AR(1) processes in Equation (\ref{generatedata1}), and entries of $\bR$ and $\bC$ are drawn independently from $\mathcal{U}(-1, 1)$. We consider four dimension settings, $(p,q)\in\{(20,20),(20,50),(50,50),(100,100)\}$, four sample size settings, $T\in\{10,15,20,25\}$, and two levels of temporal dependence of $\bE_t$, $\psi\in\{0.1,0.5\}$. Results are based on 200 repetitions, with $\alpha=0$, $s_1=s_2=5$, and even partitions.   
\end{minipage}
\end{table}

We next study the accuracy of the local loading matrix estimators, which directly determines the accuracy of our global estimators. To evaluate the accuracy of these estimators, we define 
\begin{align}\label{D1D2}
    \mathcal{D}_1(\widehat{\mathbf{R}}_i)=\frac{1}{s_1}\sum_{i=1}^{s_1}\mathcal{D}(\widehat{\mathbf{R}}_i,\mathbf{R}),\quad
    \mathcal{D}_2(\widehat{\mathbf{R}}_i)=\max_{1\leq i \leq s_1}\mathcal{D}(\widehat{\mathbf{R}}_i,\mathbf{R}),
\end{align}
which measure the average and maximum column space distances from the estimators to the true row loading matrix $\bR$, where $\mathcal{D}(\cdot)$ is defined in Equation (\ref{functionD}). The quantities $\mathcal{D}_1(\widehat{\mathbf{C}}_j)$ and $\mathcal{D}_2(\widehat{\mathbf{C}}_j)$ are defined similarly. Table~\ref{table_rici} reports these statistics, with visualizations provided in Figure~\ref{figure_rici} ($\psi=0.1$) and Figure B1 ($\psi=0.5$, in Appendix B). Overall, these results demonstrate the high accuracy of our local estimators across all settings.

\begin{table}[h]
\caption{\raggedright Summary statistics of local estimator accuracy.}
\label{table_rici}
\centering
\resizebox{1.0\linewidth}{!}{
\centering
\Large
\begin{tabular}{p{4cm}llllp{1cm}llll}
\toprule
 & \multicolumn{1}{c}{$\mathcal{D}_1(\widehat{\mathbf{R}}_i)$} & \multicolumn{1}{c}{$\mathcal{D}_2(\widehat{\mathbf{R}}_i)$} & \multicolumn{1}{c}{$\mathcal{D}_1(\widehat{\mathbf{C}}_j)$} & \multicolumn{1}{c}{$\mathcal{D}_2(\widehat{\mathbf{C}}_j)$}&  & \multicolumn{1}{c}{$\mathcal{D}_1(\widehat{\mathbf{R}}_i)$} & \multicolumn{1}{c}{$\mathcal{D}_2(\widehat{\mathbf{R}}_i)$} & \multicolumn{1}{c}{$\mathcal{D}_1(\widehat{\mathbf{C}}_j)$} & \multicolumn{1}{c}{$\mathcal{D}_2(\widehat{\mathbf{C}}_j)$} \\ \cmidrule{2-5} \cmidrule{7-10}  
 & \multicolumn{4}{c}{$(p,q)=(50,50)$} &  & \multicolumn{4}{c}{$(p,q)=(50,100)$} \\ \cmidrule{2-5} \cmidrule{7-10} 
$\psi=0.1$ &  &  &  &  &  &  &  &  &  \\
$T=0.5pq$ & 0.52(0.04) & 0.61(0.07) & 0.52(0.05) & 0.6(0.07) &  & 0.26(0.02) & 0.3(0.03) & 0.37(0.02) & 0.43(0.04) \\
$T=1.0pq$ & 0.37(0.03) & 0.44(0.06) & 0.37(0.04) & 0.43(0.05) &  & 0.18(0.01) & 0.21(0.02) & 0.26(0.02) & 0.31(0.03) \\
$T=1.5pq$ & 0.29(0.02) & 0.34(0.04) & 0.3(0.02) & 0.35(0.04) &  & 0.15(0.01) & 0.17(0.02) & 0.22(0.01) & 0.25(0.03) \\
$T=2.0pq$ & 0.26(0.02) & 0.3(0.03) & 0.26(0.02) & 0.3(0.04) &  & 0.13(0.01) & 0.15(0.02) & 0.19(0.01) & 0.22(0.03) \\
$\psi=0.5$ &  &  &  &  &  &  &  &  &  \\
$T=0.5pq$ & 0.54(0.05) & 0.64(0.08) & 0.55(0.05) & 0.64(0.08) &  & 0.27(0.02) & 0.31(0.04) & 0.39(0.03) & 0.45(0.05) \\
$T=1.0pq$ & 0.39(0.04) & 0.46(0.06) & 0.39(0.04) & 0.45(0.05) &  & 0.19(0.02) & 0.22(0.03) & 0.28(0.02) & 0.32(0.04) \\
$T=1.5pq$ & 0.31(0.02) & 0.36(0.04) & 0.31(0.03) & 0.37(0.05) &  & 0.16(0.01) & 0.18(0.02) & 0.23(0.02) & 0.27(0.03) \\
$T=2.0pq$ & 0.27(0.02) & 0.32(0.04) & 0.27(0.02) & 0.33(0.04) &  & 0.14(0.01) & 0.16(0.02) & 0.2(0.02) & 0.23(0.03) \\ \midrule
 & \multicolumn{4}{c}{$(p,q)=(100,50)$} &  & \multicolumn{4}{c}{$(p,q)=(100,100)$} \\ \cmidrule{2-5} \cmidrule{7-10} 
$\psi=0.1$ &  &  &  &  &  &  &  &  &  \\
$T=0.5pq$ & 0.37(0.03) & 0.43(0.05) & 0.26(0.02) & 0.29(0.03) &  & 0.19(0.01) & 0.21(0.02) & 0.19(0.01) & 0.21(0.02) \\
$T=1.0pq$ & 0.27(0.02) & 0.31(0.03) & 0.18(0.01) & 0.21(0.02) &  & 0.13(0.01) & 0.15(0.01) & 0.13(0.01) & 0.15(0.01) \\
$T=1.5pq$ & 0.21(0.01) & 0.25(0.02) & 0.15(0.01) & 0.17(0.02) &  & 0.11(0.01) & 0.12(0.01) & 0.11(0.01) & 0.12(0.01) \\
$T=2.0pq$ & 0.19(0.01) & 0.22(0.02) & 0.13(0.01) & 0.15(0.02) &  & 0.09(0.01) & 0.11(0.01) & 0.09(0) & 0.11(0.01) \\
$\psi=0.5$ &  &  &  &  &  &  &  &  &  \\
$T=0.5pq$ & 0.39(0.03) & 0.45(0.04) & 0.27(0.02) & 0.31(0.04) &  & 0.19(0.01) & 0.22(0.02) & 0.2(0.01) & 0.22(0.02) \\
$T=1.0pq$ & 0.27(0.02) & 0.32(0.03) & 0.19(0.01) & 0.22(0.02) &  & 0.14(0.01) & 0.16(0.01) & 0.14(0.01) & 0.16(0.02) \\
$T=1.5pq$ & 0.22(0.01) & 0.25(0.02) & 0.15(0.01) & 0.18(0.02) &  & 0.11(0.01) & 0.13(0.01) & 0.11(0.01) & 0.13(0.01) \\
$T=2.0pq$ & 0.2(0.01) & 0.23(0.03) & 0.14(0.01) & 0.16(0.02) &  & 0.1(0.01) & 0.11(0.01) & 0.1(0.01) & 0.11(0.01) \\ \bottomrule
\multicolumn{10}{l}{} \\
\end{tabular}
}
\begin{minipage}{\textwidth}
    \footnotesize
    This table reports the mean and standard deviation (in parentheses) of the average and maximum column space distances $\mathcal{D}_1(\widehat{\mathbf{R}}_i)$, $\mathcal{D}_2(\widehat{\mathbf{R}}_i)$, $\mathcal{D}_1(\widehat{\mathbf{C}}_j)$, and $\mathcal{D}_2(\widehat{\mathbf{C}}_j)$, defined in Equation (\ref{D1D2}). All values are multiplied by 10. Data are generated from model (\ref{model1}), $\bY_t=\bR\bF_t\bC^\top+\bE_t$, where $\bF_t$ and $\bE_t$ follow the AR(1) processes in Equation (\ref{generatedata1}) with $(k,r)=(3,3)$, and entries of $\bR$ and $\bC$ are drawn independently from $\mathcal{U}(-1, 1)$. We consider four dimension settings, $(p,q)\in\{(50,50),(50,100),(100,50),(100,100)\}$, four sample size settings, $T\in\{0.5pq,1.0pq,1.5pq,2.0pq\}$, and two levels of temporal dependence of $\bE_t$, $\psi\in\{0.1,0.5\}$. Results are based on 200 repetitions, with $\alpha=0$, $s_1=s_2=5$, and even partitions.   
\end{minipage}
\end{table}

\begin{figure}[h]
\begin{center}
\includegraphics[width=0.75\linewidth]{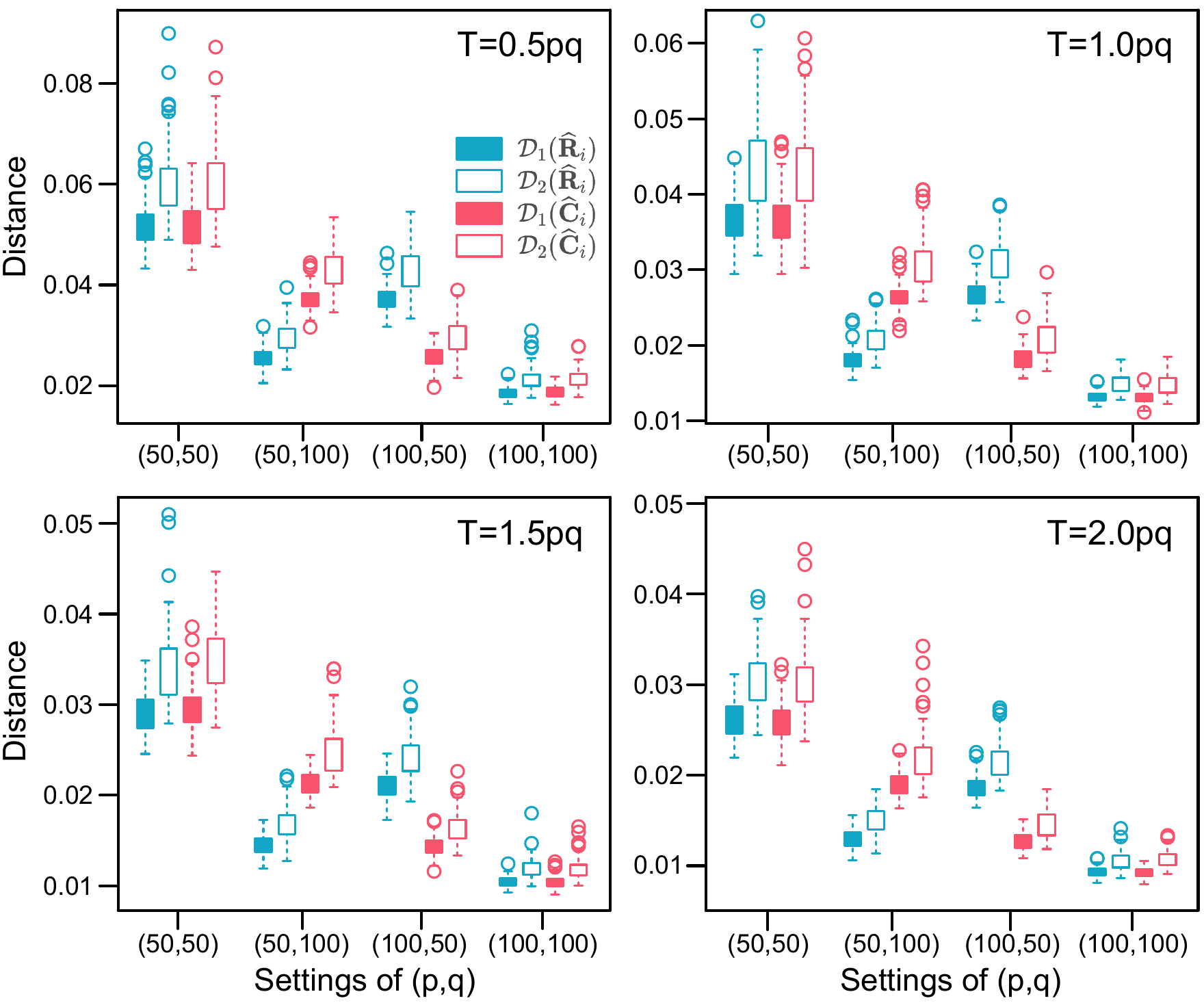}   
\end{center}
\caption{\footnotesize Boxplots of the average and maximum column space distances $\mathcal{D}_1(\widehat{\mathbf{R}}_i)$, $\mathcal{D}_2(\widehat{\mathbf{R}}_i)$, $\mathcal{D}_1(\widehat{\mathbf{C}}_j)$, and $\mathcal{D}_2(\widehat{\mathbf{C}}_j)$, defined in Equation (\ref{D1D2}). Data are generated from model (\ref{model1}), $\bY_t=\bR\bF_t\bC^\top+\bE_t$, where $\bF_t$ and $\bE_t$ follow the AR(1) processes in Equation (\ref{generatedata1}) with $(k,r)=(3,3)$ and $\psi=0.1$, and entries of $\bR$ and $\bC$ are drawn independently from $\mathcal{U}(-1, 1)$. We consider four dimension settings, $(p,q)\in\{(50,50),(50,100),(100,50),(100,100)\}$, and four sample size settings, $T\in\{0.5pq,1.0pq,1.5pq,2.0pq\}$. Results are based on 200 repetitions, with $\alpha=0$, $s_1=s_2=5$, and even partitions.   
}
\label{figure_rici}
\end{figure}

We now assess the accuracy of our global estimators. Table~\ref{figure_rcfs} reports means and standard deviations, based on 200 repetitions, of $e_{\bR}=\frac{1}{p} \| \widehat{\mathbf{R}}-\mathbf{RH_R} \|_2^2$, $e_{\bC}=\frac{1}{q} \| \widehat{\mathbf{C}}-\mathbf{CH_C} \|_2^2$, $e_{\bF}=\frac{1}{T}\sum_{t=1}^T \| \widehat{\mathbf{F}}_t-\mathbf{H}_\bR^{-1}\mathbf{F}_t(\mathbf{H}_\bC^{-1})^\top \|_2^2$, and $e_{\bS}=\frac{1}{T}\sum_{t=1}^T \| \widehat{\mathbf{S}}_t-\mathbf{S}_t \|_2^2$ from our method and $\alpha$-PCA. Consistent with Theorems \ref{them1}-\ref{them3}, estimation accuracy improves as $p$, $q$, and $T$ increase.
In addition, unlike the other estimators, the estimation accuracy of the common component decreases sharply with higher temporal dependence $\psi$ in the idiosyncratic term $\bE_t$, as theoretically expected. Across all settings, our method achieves accuracy comparable to that of $\alpha$-PCA in estimating the loading matrices $\bR$ and $\bC$, factor matrix $\bF_t$, and common component $\bS_t$. 

\begin{table}
\caption{\raggedright Summary statistics of global estimator accuracy.}
\label{figure_rcfs}
\centering
\resizebox{\textwidth}{!}{
\Large
\begin{tabular}{p{4cm}llp{0.5cm}llp{0.5cm}llp{0.5cm}ll}
\hline
 & \multicolumn{2}{c}{$e_{\bR}$} &  & \multicolumn{2}{c}{$e_{\bC}$} &  & \multicolumn{2}{c}{$e_{\bF}$} &  & \multicolumn{2}{c}{$e_{\bS}$} \\ \cline{2-3} \cline{5-6} \cline{8-9} \cline{11-12} 
 & Ours & $\alpha$-PCA &  & Ours & $\alpha$-PCA &  & Ours & $\alpha$-PCA &  & Ours & $\alpha$-PCA \\ \hline
\multicolumn{12}{l}{$(p,q)=(50,50)$} \\
\multicolumn{12}{l}{$\psi=0.1$} \\
$T=0.5pq$ & \textbf{0.32(0.1)} & 0.61(0.2) &  & \textbf{0.32(0.12)} & 0.59(0.19) &  & 0.73(0.2) & \textbf{0.37(0.08)} &  & 8.03(1.93) & \textbf{7.85(1.98)} \\
$T=1.0pq$ & \textbf{0.31(0.1)} & 0.57(0.17) &  & \textbf{0.32(0.1)} & 0.59(0.18) &  & 0.5(0.13) & \textbf{0.26(0.07)} &  & 5.6(1.26) & \textbf{5.46(1.39)} \\
$T=1.5pq$ & \textbf{0.31(0.09)} & 0.58(0.18) &  & \textbf{0.33(0.1)} & 0.6(0.22) &  & 0.41(0.11) & \textbf{0.21(0.05)} &  & 4.5(1.23) & \textbf{4.49(1.05)} \\
$T=2.0pq$ & \textbf{0.3(0.09)} & 0.58(0.19) &  & \textbf{0.3(0.09)} & 0.59(0.2) &  & 0.37(0.09) & \textbf{0.18(0.04)} &  & \textbf{3.85(0.96)} & 3.99(1.07) \\
\multicolumn{12}{l}{$\psi=0.5$} \\
$T=0.5pq$ & \textbf{0.31(0.09)} & 0.6(0.18) &  & \textbf{0.33(0.1)} & 0.61(0.21) &  & 0.72(0.2) & \textbf{0.42(0.11)} &  & \textbf{12.32(3.04)} & 12.47(3.35) \\
$T=1.0pq$ & \textbf{0.32(0.12)} & 0.57(0.17) &  & \textbf{0.31(0.1)} & 0.61(0.22) &  & 0.51(0.13) & \textbf{0.29(0.07)} &  & \textbf{8.68(2.19)} & 8.84(2.19) \\
$T=1.5pq$ & \textbf{0.31(0.1)} & 0.61(0.22) &  & \textbf{0.31(0.09)} & 0.61(0.2) &  & 0.45(0.11) & \textbf{0.24(0.06)} &  & 7.38(1.93) & \textbf{7.29(1.91)} \\
$T=2.0pq$ & \textbf{0.31(0.1)} & 0.57(0.17) &  & \textbf{0.31(0.1)} & 0.58(0.17) &  & 0.38(0.12) & \textbf{0.21(0.06)} &  & 6.65(1.58) & \textbf{6.16(1.67)} \\ \hline
\multicolumn{12}{l}{$(p,q)=(50,100)$} \\
\multicolumn{12}{l}{$\psi=0.1$} \\
$T=0.5pq$ & \textbf{0.04(0.01)} & 0.13(0.03) &  & \textbf{0.44(0.12)} & 0.57(0.17) &  & 0.38(0.1) & \textbf{0.19(0.05)} &  & \textbf{5.67(1.39)} & 5.75(1.47) \\
$T=1.0pq$ & \textbf{0.04(0.01)} & 0.14(0.04) &  & \textbf{0.44(0.13)} & 0.6(0.19) &  & 0.27(0.07) & \textbf{0.13(0.04)} &  & \textbf{3.82(1.02)} & 3.9(0.99) \\
$T=1.5pq$ & \textbf{0.04(0.01)} & 0.13(0.03) &  & \textbf{0.44(0.12)} & 0.55(0.16) &  & 0.22(0.06) & \textbf{0.11(0.03)} &  & 3.29(0.92) & \textbf{3.22(0.8)} \\
$T=2.0pq$ & \textbf{0.04(0.01)} & 0.13(0.04) &  & \textbf{0.45(0.14)} & 0.61(0.22) &  & 0.19(0.06) & \textbf{0.1(0.02)} &  & 2.78(0.77) & \textbf{2.71(0.64)} \\
\multicolumn{12}{l}{$\psi=0.5$} \\
$T=0.5pq$ & \textbf{0.04(0.01)} & 0.13(0.03) &  & \textbf{0.46(0.15)} & 0.58(0.17) &  & 0.38(0.1) & \textbf{0.22(0.05)} &  & \textbf{8.68(1.96)} & 8.7(2.11) \\
$T=1.0pq$ & \textbf{0.04(0.01)} & 0.13(0.03) &  & \textbf{0.45(0.15)} & 0.58(0.17) &  & 0.29(0.08) & \textbf{0.15(0.04)} &  & 6.23(1.48) & \textbf{6.07(1.54)} \\
$T=1.5pq$ & \textbf{0.04(0.01)} & 0.13(0.03) &  & \textbf{0.43(0.13)} & 0.59(0.17) &  & 0.22(0.05) & \textbf{0.13(0.03)} &  & 5.04(1.32) & \textbf{4.99(1.27)} \\
$T=2.0pq$ & \textbf{0.04(0.01)} & 0.13(0.03) &  & \textbf{0.45(0.15)} & 0.59(0.19) &  & 0.2(0.05) & \textbf{0.11(0.03)} &  & \textbf{4.29(1.09)} & 4.39(1.02) \\ \hline
\multicolumn{12}{l}{$(p,q)=(100,50)$} \\
\multicolumn{12}{l}{$\psi=0.1$} \\
$T=0.5pq$ & \textbf{0.44(0.11)} & 0.58(0.16) &  & \textbf{0.04(0.01)} & 0.13(0.03) &  & 0.37(0.1) & \textbf{0.18(0.05)} &  & \textbf{5.44(1.39)} & 5.46(1.3) \\
$T=1.0pq$ & \textbf{0.43(0.13)} & 0.6(0.18) &  & \textbf{0.04(0.01)} & 0.13(0.04) &  & 0.27(0.07) & \textbf{0.14(0.04)} &  & 3.98(0.9) & \textbf{3.96(1.11)} \\
$T=1.5pq$ & \textbf{0.43(0.12)} & 0.58(0.16) &  & \textbf{0.04(0.01)} & 0.13(0.04) &  & 0.21(0.05) & \textbf{0.11(0.03)} &  & \textbf{3.2(0.8)} & 3.22(0.78) \\
$T=2.0pq$ & \textbf{0.43(0.12)} & 0.61(0.18) &  & \textbf{0.04(0.01)} & 0.13(0.03) &  & 0.19(0.05) & \textbf{0.1(0.02)} &  & \textbf{2.83(0.71)} & 2.88(0.68) \\
\multicolumn{12}{l}{$\psi=0.5$} \\
$T=0.5pq$ & \textbf{0.43(0.11)} & 0.59(0.17) &  & \textbf{0.04(0.01)} & 0.13(0.03) &  & 0.39(0.1) & \textbf{0.22(0.05)} &  & \textbf{8.51(2.4)} & 8.61(2.04) \\
$T=1.0pq$ & \textbf{0.43(0.12)} & 0.58(0.18) &  & \textbf{0.04(0.01)} & 0.13(0.03) &  & 0.28(0.07) & \textbf{0.15(0.04)} &  & 6.49(1.36) & \textbf{6.33(1.57)} \\
$T=1.5pq$ & \textbf{0.44(0.13)} & 0.58(0.19) &  & \textbf{0.04(0.01)} & 0.13(0.03) &  & 0.22(0.06) & \textbf{0.12(0.03)} &  & 5.02(1.31) & \textbf{4.99(1.27)} \\
$T=2.0pq$ & \textbf{0.44(0.13)} & 0.56(0.14) &  & \textbf{0.04(0.01)} & 0.13(0.03) &  & 0.2(0.05) & \textbf{0.11(0.02)} &  & \textbf{4.24(1.06)} & 4.31(1.13) \\ \hline
\multicolumn{12}{l}{$(p,q)=(100,100)$} \\
\multicolumn{12}{l}{$\psi=0.1$} \\
$T=0.5pq$ & \textbf{0.05(0.01)} & 0.13(0.03) &  & \textbf{0.05(0.01)} & 0.13(0.03) &  & 0.17(0.05) & \textbf{0.09(0.02)} &  & 4.06(1.07) & \textbf{4.01(1.09)} \\
$T=1.0pq$ & \textbf{0.05(0.01)} & 0.13(0.03) &  & \textbf{0.05(0.01)} & 0.13(0.03) &  & 0.13(0.03) & \textbf{0.06(0.02)} &  & \textbf{2.72(0.71)} & 2.81(0.7) \\
$T=1.5pq$ & \textbf{0.05(0.01)} & 0.13(0.02) &  & \textbf{0.05(0.01)} & 0.13(0.03) &  & 0.1(0.02) & \textbf{0.05(0.01)} &  & 2.27(0.58) & \textbf{2.25(0.58)} \\
$T=2.0pq$ & \textbf{0.05(0.01)} & 0.13(0.03) &  & \textbf{0.05(0.01)} & 0.13(0.03) &  & 0.08(0.02) & \textbf{0.05(0.01)} &  & \textbf{1.92(0.52)} & 2(0.51) \\
\multicolumn{12}{l}{$\psi=0.5$} \\
$T=0.5pq$ & \textbf{0.05(0.01)} & 0.13(0.03) &  & \textbf{0.05(0.01)} & 0.13(0.03) &  & 0.18(0.04) & \textbf{0.1(0.03)} &  & 6.2(1.45) & \textbf{6.11(1.64)} \\
$T=1.0pq$ & \textbf{0.05(0.01)} & 0.13(0.03) &  & \textbf{0.05(0.01)} & 0.13(0.03) &  & 0.12(0.03) & \textbf{0.07(0.02)} &  & \textbf{4.39(1.05)} & 4.43(1.14) \\
$T=1.5pq$ & \textbf{0.05(0.01)} & 0.13(0.03) &  & \textbf{0.05(0.01)} & 0.13(0.03) &  & 0.1(0.02) & \textbf{0.06(0.02)} &  & 3.65(0.93) & \textbf{3.61(0.91)} \\
$T=2.0pq$ & \textbf{0.05(0.01)} & 0.13(0.03) &  & \textbf{0.05(0.01)} & 0.13(0.03) &  & 0.09(0.02) & \textbf{0.05(0.01)} &  & 3.13(0.86) & \textbf{3.09(0.77)} \\ \hline
\end{tabular}
}
\begin{minipage}{\textwidth}
    \footnotesize
    \vspace{0.1cm}
    This table reports the mean and standard deviation (in parentheses) of $e_{\bR}=\frac{1}{p} \| \widehat{\mathbf{R}}-\mathbf{RH_R} \|_2^2$, $e_{\bC}=\frac{1}{q} \| \widehat{\mathbf{C}}-\mathbf{CH_C} \|_2^2$, $e_{\bF}=\frac{1}{T}\sum_{t=1}^T \| \widehat{\mathbf{F}}_t-\mathbf{H}_\bR^{-1}\mathbf{F}_t(\mathbf{H}_\bC^{-1})^\top\|_2^2$, and $e_{\bS}=\frac{1}{T}\sum_{t=1}^T \| \widehat{\mathbf{S}}_t-\mathbf{S}_t \|_2^2$ from our method and $\alpha$-PCA. All values are multiplied by 100. Data are generated from model (\ref{model1}), $\bY_t=\bR\bF_t\bC^\top+\bE_t$, where $\bF_t$ and $\bE_t$ follow the AR(1) processes in Equation (\ref{generatedata1}) with $(k,r)=(3,3)$, and entries of $\bR$ and $\bC$ are drawn independently from $\mathcal{U}(-1, 1)$. We consider four dimension settings, $(p,q)\in\{(50,50),(50,100),(100,50),(100,100)\}$, four sample size settings, $T\in\{0.5pq,1.0pq,1.5pq,2.0pq\}$, and two levels of temporal dependence of $\bE_t$, $\psi\in\{0.1,0.5\}$. Results are based on 200 repetitions, with $\alpha=0$, $s_1=s_2=5$, and even partitions.   
\end{minipage}
\end{table}

\subsubsection{Asymptotic normality of loading matrix estimators}
We next examine the asymptotic normality of $\widehat{\mathbf{R}}-\mathbf{RH_R}$ established in Theorem~\ref{them4}. Specifically, we focus on the first row $\widehat{\bR}^1-\bR^1\bH_\bR\in\mathbb{R}^{1\times k}$. Without loss of generality, we generate $\bF_t$ and $\bE_t$ with independent rows and columns as
\begin{align}\label{generatedata2}
    \bF_t\sim\mathcal{MN}_{3\times 3}\left(3\mathbf{I},\mathbf{I},\mathbf{I}\right),\quad\bE_t\sim\mathcal{MN}_{p\times q}\left(\mathbf{0},\mathbf{I},\mathbf{I}\right).
\end{align}
Under the conditions ${\sqrt{\overline{m}T}}/{p}=\mathbf{o}(1)$ and ${\overline{m}}/{\underline{m}}=\mathcal{O}(1)$, Theorem~\ref{them4} implies that, as $\underline{m},p,T\to \infty$,
$$\sqrt{\overline{m}T}\left(\widehat{\bR}^1-\bR^1\bH_\bR\right)^\top\overset{\mathcal{D}}{\longrightarrow}\mathcal{N}\left(\mathbf{0},\mathbf{\Sigma}_{\bR_1}\right).$$
We estimate the asymptotic covariance $\mathbf{\Sigma}_{\bR_1}$ by $\widehat{\Sigma}_{\bR_1}$ defined in Theorem~\ref{them5}, yielding $$\widehat{\Sigma}_{\bR_1}^{-1/2}\sqrt{\overline{m}T}(\widehat{\bR}^1-\bR^1\bH_\bR)^\top\overset{\mathcal{D}}{\longrightarrow}\mathcal{N}\left(\mathbf{0},\mathbf{I}\right).$$ 
We first consider the special case $s_1=1$ (i.e., single-machine estimation). Figure B2 in Appendix B displays the empirical distribution of the first element of $\widehat{\Sigma}_{\bR_1}^{-1/2}\sqrt{\overline{m}T}(\widehat{\bR}^1-\bR^1\bH_\bR)^\top$ for various sample sizes. As expected, the distribution approaches the standard normal distribution as $T$ increases, consistent with \citet{chen2021}. We next consider the case $s_1=3$. Figure~\ref{figure_asymptotic_3} exhibits similar convergence behavior, where the empirical distribution of its first element also approaches the standard normal distribution as $T$ grows, confirming our theoretical results.

\begin{figure}[h]
\centering
\includegraphics[width=0.75\linewidth]{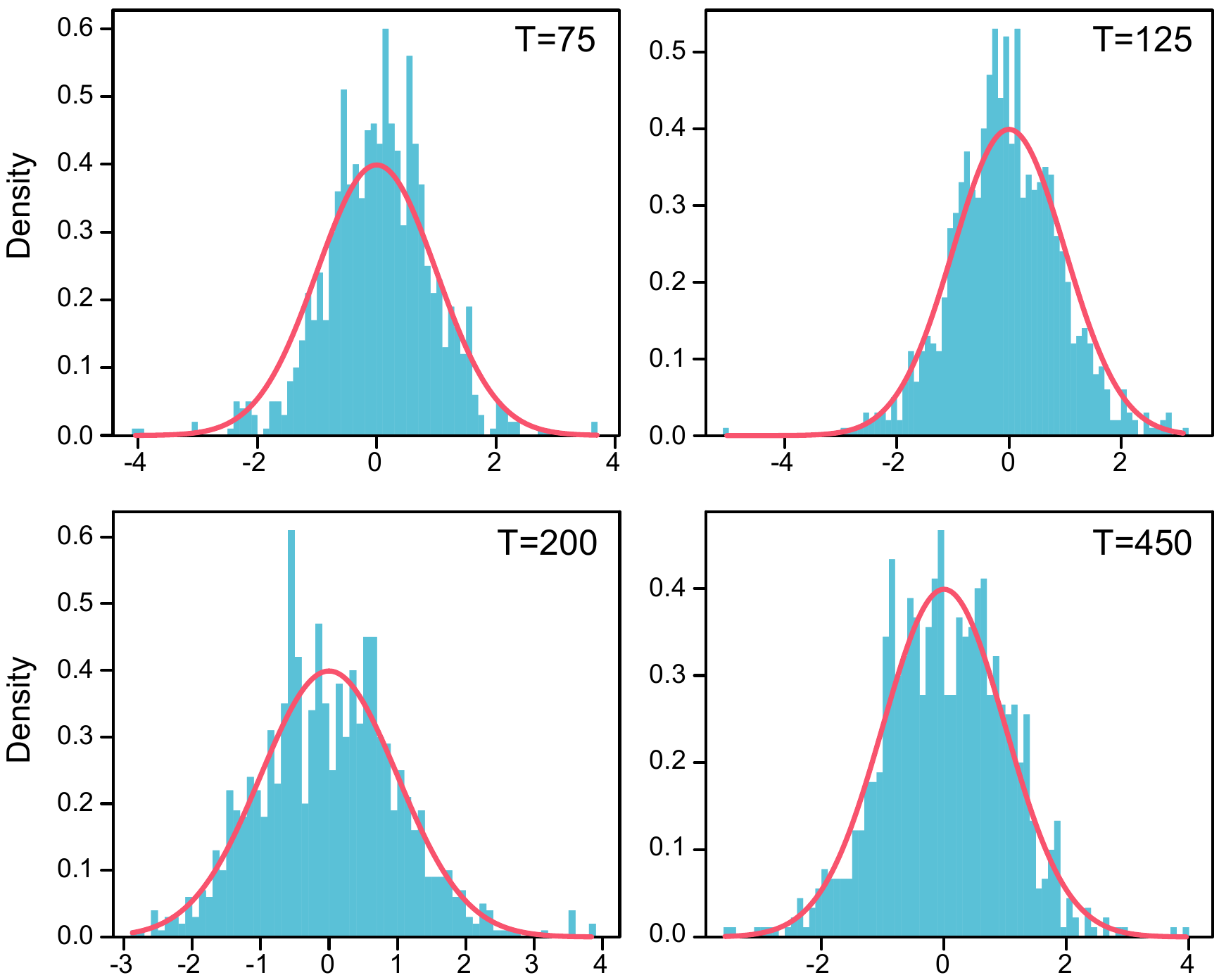}        
\caption{\footnotesize Density histograms of the first element of $\widehat{\Sigma}_{\bR_1}^{-1/2}\sqrt{\overline{m}T}(\widehat{\bR}^1-\bR^1\bH_\bR)^\top$. Red curves denote the standard normal distribution. Data are generated from model (\ref{model1}), $\bY_t=\bR\bF_t\bC^\top+\bE_t$, where $\bF_t$ and $\bE_t$ follow the matrix normal distributions in Equation (\ref{generatedata2}) with $(k,r)=(3,3)$ and $(p,q)=(200,200)$, and entries of $\bR$ and $\bC$ are drawn independently from $\mathcal{U}(-1, 1)$. We consider four sample size settings, $T\in\{0.5pq,1.0pq,1.5pq,2.0pq\}$. Results are based on 1000 repetitions, with $\alpha=0$, $s_1=s_2=3$, and even partitions.   
}
\label{figure_asymptotic_3}
\end{figure}

\subsubsection{Computational efficiency}

We evaluated the convergence rates and asymptotic normality of our estimators, showing accuracy comparable to that of $\alpha$-PCA. To further illustrate the computational savings of our method, we report the average runtime for estimating loading matrices under two scenarios. In setting I, we fix $p=q=100$ and increase $T$ from 100 to 10000; In setting II, we fix $T=100$ and increase $p=q$ from 50 to 1000. Figure~\ref{figure_time} demonstrates the substantial computational gains of our method relative to $\alpha$-PCA, consistent with the discussion in Section~\ref{sec2.3}. 

\begin{figure}[h]
\centering
\includegraphics[width=0.75\linewidth]{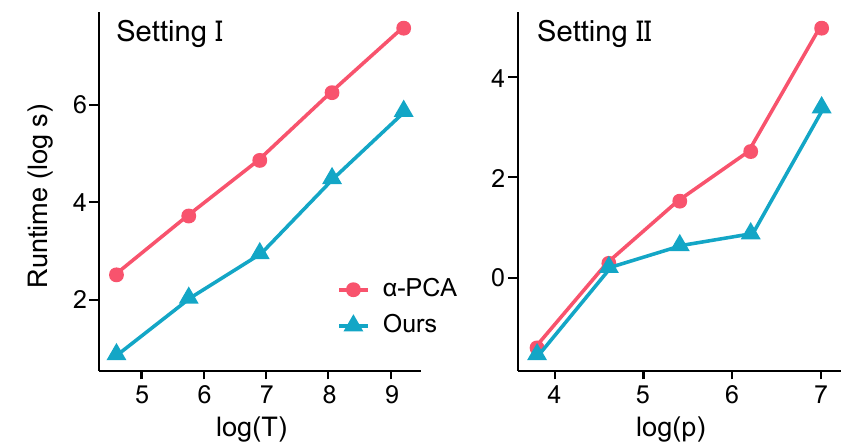}        
\caption{\footnotesize Average runtime (log s) for estimating loading matrices using our method and $\alpha$-PCA. Data are generated from model (\ref{model1}), $\bY_t=\bR\bF_t\bC^\top+\bE_t$, where $\bF_t$ and $\bE_t$ follow the AR(1) processes in Equation (\ref{generatedata1}) with $(k,r)=(3,3)$ and $\psi=0.1$, and entries of $\bR$ and $\bC$ are drawn independently from $\mathcal{U}(-1, 1)$. Two settings are considered: In setting I, we fix $p=q=100$ and increase $T$ from 100 to 10,000; In setting II, we fix $T=100$ and increase $p=q$ from 50 to 1,000. Results are based on 200 repetitions, with $\alpha=0$, $s_1=s_2=5$, and even partitions.   
}
\label{figure_time}
\end{figure}

\subsubsection{Accuracy of estimators in the unit-root nonstationary case}

Lastly, we assess the estimation accuracy of our method for unit-root nonstationary time series. The observations are still generated from model (\ref{model1}), $\bY_t=\bR\bF_t\bC^\top+\bE_t$, where $\bF_t$ is generated following \cite{li2025factormodelsmatrixvaluedtime} and $\bE_t$ is generated as described in Section~\ref{estimatingconvergence}. Specifically, we generate $\bF_t$ and $\bE_t$ as,
\begin{equation}
    \text{vec}(\bF_t)=\text{vec}(\bF_{t-1})+\bw_t,\quad \bw_t=0.3\bw_{t-1}+\boldsymbol{\xi}_t, \label{generatedata3}
\end{equation}
and 
\begin{equation}
    \text{vec}(\bE_t)=0.1\text{vec}(\bE_{t-1})+\boldsymbol{\xi}'_t,\label{generatedata4}
\end{equation}
with initial values $\bF_0=\boldsymbol{0}$, $\bw_0=\boldsymbol{0}$, and $\bE_0=\boldsymbol{0}$. The noise vectors $\boldsymbol{\xi}_t\in\mathbb{R}^{kr}$ and $\boldsymbol{\xi}'_t\in\mathbb{R}^{pq}$ are i.i.d., with each element drawn from the standard normal distribution. 

We evaluate the accuracy of $\wh{\bR}$, $\wh{\bC}$, $\wh{\bF}_t$, and $\wh{\bS}_t$. Table~\ref{tabunitroot} reports means and standard deviations, based on 200 repetitions, of $e_{\bR}=\frac{1}{p} \| \widehat{\mathbf{R}}-\mathbf{RH_R} \|_2^2$, $e_{\bC}=\frac{1}{q} \| \widehat{\mathbf{C}}-\mathbf{CH_C} \|_2^2$, $e_{\bF}=\frac{1}{T}\sum_{t=1}^T \| \widehat{\mathbf{F}}_t-\mathbf{H}_\bR^{-1}\mathbf{F}_t(\mathbf{H}_\bC^{-1})^\top\|_2^2$, and $e_{\bS}=\frac{1}{T}\sum_{t=1}^T \| \widehat{\mathbf{S}}_t-\mathbf{S}_t \|_2^2$. Across all settings, our method achieves substantially higher accuracy than in the stationary case, consistent with our theoretical results in Section \ref{sec3.5}.

\begin{table}[h]
\caption{\raggedright Summary statistics of global estimator accuracy in the unit-root nonstationary case.}
\label{tabunitroot}
    \centering
    \resizebox{\linewidth}{!}{
    \begin{tabular}{@{}p{4cm}llll@{}}
\toprule
 & \multicolumn{1}{c}{$e_{\bR}$} & \multicolumn{1}{c}{$e_{\bC}$} & \multicolumn{1}{c}{$e_{\bF}$} & \multicolumn{1}{c}{$e_{\bS}$} \\ \midrule
\multicolumn{2}{l}{$(p,q)=(50,50)$} &  &   &  \\
$T=0.25pq$ & 4.145e-6(4.306e-6) & 4.063e-6(3.564e-6) & 3.343e-3(6.250e-4) & 6.934e-2(1.493e-3) \\
$T=0.50pq$ & 8.651e-7(7.117e-7) & 9.193e-7(8.002e-7) & 3.008e-3(3.202e-4) & 6.846e-2(9.588e-4) \\
$T=0.75pq$ & 4.853e-7(4.685e-7) & 4.203e-7(3.552e-7) & 2.901e-3(1.839e-4) & 6.802e-2(7.181e-4) \\
$T=1.00pq$ & 2.464e-7(2.212e-7) & 2.458e-7(2.240e-7) & 2.843e-3(1.686e-4) & 6.775e-2(6.425e-4) \\ \midrule
\multicolumn{2}{l}{$(p,q)=(100,100)$} &  &   &  \\
$T=0.25pq$ & 6.119e-8(4.882e-8) & 6.173e-8(5.125e-8) & 7.067e-4(3.121e-5) & 6.834e-2(6.718e-4) \\
$T=0.5pq$ & 1.619e-8(1.322e-8) & 1.550e-8(1.444e-8) & 6.873e-4(1.266e-5) & 6.776e-2(4.613e-4) \\
$T=0.75pq$ & 7.967e-9(5.950e-9) & 7.876e-9(6.166e-9) & 6.850e-4(1.353e-5) & 6.763e-2(3.982e-4) \\
$T=1.00pq$ & 4.413e-9(3.611e-9) & 4.241e-9(2.784e-9) & 6.822e-4(9.091e-6) & 6.756e-2(3.441e-4) \\ \bottomrule
\end{tabular}
    }
    \begin{minipage}{\linewidth}
        \footnotesize
        \vspace{0.1cm}
        This table reports the mean and standard deviation (in parentheses) of $e_{\bR}=\frac{1}{p} \| \widehat{\mathbf{R}}-\mathbf{RH_R} \|_2^2$, $e_{\bC}=\frac{1}{q} \| \widehat{\mathbf{C}}-\mathbf{CH_C} \|_2^2$, $e_{\bF}=\frac{1}{T}\sum_{t=1}^T \| \widehat{\mathbf{F}}_t-\mathbf{H}_\bR^{-1}\mathbf{F}_t(\mathbf{H}_\bC^{-1})^\top\|_2^2$, and $e_{\bS}=\frac{1}{T}\sum_{t=1}^T \| \widehat{\mathbf{S}}_t-\mathbf{S}_t \|_2^2$ from our method. Data are generated from model (\ref{model1}), $\bY_t=\bR\bF_t\bC^\top+\bE_t$, where $\bF_t$ and $\bE_t$ follow the AR(1) processes in Equations (\ref{generatedata3}) and (\ref{generatedata4}) with $(k,r)=(3,3)$, and entries of $\bR$ and $\bC$ are drawn independently from $\mathcal{U}(-1, 1)$. We consider four dimension settings, $(p,q)\in\{(50,50),(50,100),(100,50),(100,100)\}$, and four sample size settings, $T\in\{0.25pq,0.50pq,0.75pq,1.00pq\}$. Results are based on 200 repetitions, with $\alpha=0$, $s_1=s_2=5$, and even partitions.   
    \end{minipage}
\end{table}

\subsection{Real data applications}\label{realexample_sec}

In this section, we apply our method to the Fama–French Stock Return Data and the OECD CPI data to assess its predictive performance.

\subsubsection{Example 1: Fama–French Stock Return Data}

We first apply our method to the Fama–French stock return dataset analyzed in \cite{WANG2019231}. The dataset contains monthly returns for 100 portfolios formed by intersecting 10 size levels with 10 book-to-market levels. The sample spans July 1926 to October 2025, covering 1,192 months. Details are available at \url{http://mba.tuck.dartmouth.edu/pages/faculty/ken.french/data_library.html}.

To compare our method with $\alpha$-PCA, we take the earliest 85\% of the observations as a training set to estimate the loading matrices, and the remaining 15\% as a test set to evaluate their performance. We define out-of-sample $R^2$ on the test set $\bY_t^\text{test}$ of size $N$ as:
\begin{equation}
    \text{out-of-sample }R^2=1-\frac{\sum_{t=1}^{N}\left\|\bY_t^\text{test}-\wh{\bY}_t^{\text{test}}\right\|_\mathsf{F}^2}{\sum_{t=1}^{N}\left\|\bY_t^\text{test}-\overline{\bY}^\text{test}\right\|_\mathsf{F}^2},\label{outofsampler2}
\end{equation}
where $\wh{\bY}_t^{\text{test}}=\wh{\bR}^\top\wh{\bR}\bY_t^\text{test}\wh{\bC}^\top\wh{\bC}$ and $\overline{\bY}^\text{test}=\frac{1}{N}\sum_{t=1}^{N}\bY_t^\text{test}$, with $\wh{\bR}$ and $\wh{\bC}$ estimated from the training set. The denominator is the total sum of squares (TSS), which measures the deviation of $\bY_t^{\text{test}}$ from its sample mean $\overline{\bY}^{\text{test}}$. The numerator is the residual sum of squares (RSS), measuring the deviation of $\bY_t^{\text{test}}$ from its estimator $\wh{\bY}_t^{\text{test}}$ constructed by the training set.

Table \ref{table_stock} reports the training time and out-of-sample $R^2$ of our method and $\alpha$-PCA on the stock return dataset under various assumed numbers of factors, along with $p$-values from the Diebold–Mariano test \citep{diebold1995comparing} used to compare their predictive performance on the RSS. The null hypothesis of interest is that the two methods yield equal forecast errors, $\mathbb{E}[\|\bY_t^\text{test}-\wh{\bY}_t^{\text{test,ours}}\|_\mathsf{F}^2-\|\bY_t^\text{test}-\wh{\bY}_t^{\text{test},\alpha}\|_\mathsf{F}^2]=0$, where $\wh{\bY}_t^{\text{test,ours}}$ and $\wh{\bY}_t^{\text{test},\alpha}$ denote our estimator and $\alpha$-PCA estimator, against the alternative that our method produces a smaller RSS. The eigenvalue ratio-based method suggests that $(k,r)=(4,1)$. The numbers of computing nodes, $s_1$ and $s_2$, are treated as tuning parameters and selected from 1 to 5 to maximize out-of-sample $R^2$, with their chosen values reported in the table. The distributed data allocation follows the strategy in Section \ref{sec2.5}. The observations, which are already sorted, are allocated sequentially and evenly to the nodes.

Across all assumed factor numbers, our method consistently achieves higher out-of-sample $R^2$, with all $p$-values below 0.05, indicating improved forecastability from our distributed estimation procedure. Moreover, despite the low dimensionality of the dataset, our method still exhibits substantial reductions in training time, consistent with our theoretical analysis and simulation results. This finding suggests that our framework achieves significant computational savings in high-dimensional settings.

\begin{table}[h]
\centering
    \begin{minipage}{0.9\linewidth}
    \caption{\raggedright Training time (s) and out-of-sample $R^2$ for the stock return dataset.
    } \label{table_stock}
    \centering
    \resizebox{\linewidth}{!}{
    \begin{tabular}{ccp{0.5cm}ccp{0.25cm}ccc}
\hline
 & & & \multicolumn{2}{c}{Training time (s)} &  & \multicolumn{3}{c}{Out-of-sample $R^2$} \\ \cline{4-5} \cline{7-9} 
\multirow{-2}{*}{$(k,r)$} & \multirow{-2}{*}{\begin{tabular}[c]{@{}c@{}}Selected\\ $(s_1,s_2)$\end{tabular}} &  & Ours & $\alpha$-PCA &  & Ours & $\alpha$-PCA & $p$-value \\
\hline
(2,1) & (5,5)&  & \textbf{0.021} & 0.046 &  & \textbf{0.639} & 0.591 & 8.98e-4 \\
(3,1) & (2,5) & & \textbf{0.008} & 0.017 &  & \textbf{0.680} & 0.637 & 1.59e-5 \\
\textbf{(4,1)} & (4,5) & &  \textbf{0.004} & 0.017 &  & \textbf{0.695} & 0.667 & 3.70e-3 \\
(5,1) & (5,5) & & \textbf{0.004} & 0.021 &  & \textbf{0.716} & 0.690 & 3.42e-3 \\
(3,2) & (2,3) & & \textbf{0.004} & 0.023 &  & \textbf{0.771} & 0.712 & 1.87e-4 \\
(4,2) & (4,3) & & \textbf{0.003} & 0.022 &  & \textbf{0.797} & 0.753 & 2.09e-3 \\
(5,2) & (5,3) & & \textbf{0.003} & 0.020 &  & \textbf{0.848} & 0.798 & 7.26e-4 \\ \hline
\end{tabular}
    }
    \end{minipage}
    \begin{minipage}{0.9\linewidth}
        \footnotesize
        \vspace{0.1cm}
        This table reports the training time (s) and out-of-sample $R^2$ of our method and $\alpha$-PCA on the stock return dataset under various assumed numbers of factors, along with $p$-values based on the Diebold–Mariano test. The dataset contains monthly returns for $10\times 10$ portfolios, covering 1,192 months. We take the earliest 85\% of the observations as a training set, and the remaining 15\% as a test set. The out-of-sample $R^2$ is computed according to Equation (\ref{outofsampler2}), with tuning parameters $s_1$ and $s_2$ selected to maximize out-of-sample $R^2$. In our method, we set $s_1,s_2\in\{1,2,...,5\}$ with even partitions, and in $\alpha$-PCA we use $\alpha=0$. The eigenvalue ratio-based method suggests that $(k,r)=(4,1)$.
    \end{minipage}
\end{table}

\subsubsection{Example 2: OECD CPI Data}

In this example, we analyze a multinational macroeconomic dataset collected from OECD, similar to those used in \cite{chen2021} and \cite{yu2020}. The dataset contains monthly Consumer Price Index (CPI) statistics from 1914 to 2025, covering 27 expenditure items for 52 countries and international organizations over 1343 months. The lists of countries and expenditure items are provided in Table B3 in Appendix B. Further details are available at \url{https://data-explorer.oecd.org/vis?df[ds]=dsDisseminateFinalDMZ&df[id]=DSD_PRICES%40DF_PRICES_ALL&df[ag]=OECD.SDD.TPS}.

We compare the predictive performance of our method with $\alpha$-PCA using the earliest 95\% of the observations used for training and the remaining 5\% for testing. Table~\ref{table_OECD} reports the training time and out-of-sample $R^2$ of both methods on the CPI dataset under various assumed numbers of factors, along with $p$-values from the Diebold–Mariano test \citep{diebold1995comparing} used to compare their predictive performance on the RSS. The eigenvalue ratio-based method suggests that $(k,r)=(3,2)$. The numbers of computing nodes, $s_1$ and $s_2$, are treated as tuning parameters and selected from 1 to 10 to maximize out-of-sample $R^2$, with their chosen values reported in the table. The distributed data allocation also follows the strategy in Section \ref{sec2.5}. The observations are first sorted by GDP rankings along the nation dimension (World Bank, December 16, 2024) and by variance along the item dimension in decreasing order, and are then sequentially and evenly allocated to the nodes. Across all assumed factor numbers, our method consistently achieves shorter training times and higher out-of-sample $R^2$, with all $p$-values below 0.05, indicating both improved computational efficiency and enhanced forecastability. Moreover, the effect of $(s_1,s_2)$ on training times is clearly observed in the table and aligns well with the discussion in Section \ref{sec2.3}.

\begin{table}[h]
\centering
\begin{minipage}{0.9\linewidth}
    \caption{\raggedright Training time (s) and out-of-sample $R^2$ for the CPI dataset. 
    } 
    \label{table_OECD}\centering
    \resizebox{\linewidth}{!}{
    \centering
\begin{tabular}{ccp{0.5cm}ccp{0.25cm}ccc}
\hline
 &  & &\multicolumn{2}{c}{Training time (s)} &  & \multicolumn{3}{c}{Out-of-sample $R^2$} \\ \cline{4-5} \cline{7-9} 
\multirow{-2}{*}{$(k,r)$} & \multirow{-2}{*}{\begin{tabular}[c]{@{}c@{}}Selected\\ $(s_1,s_2)$\end{tabular}} & & Ours & $\alpha$-PCA &  & Ours & $\alpha$-PCA & $p$-value \\
\hline
(3,1) & (4,10)& & \textbf{0.062} & 0.242 &  & \textbf{0.439} & 0.429 & 1.92e-7 \\
(2,2) & (9,9)& & \textbf{0.059} & 0.229 &  & \textbf{0.435} & 0.426 & 1.97e-12 \\
\textbf{(3,2)} & (4,9)& & \textbf{0.070} & 0.214 &  & \textbf{0.460} & 0.437 & 4.32e-8 \\
(4,3) & (1,3)& & \textbf{0.123} & 0.214 &  & \textbf{0.460} & 0.443 & 8.87e-9 \\
(3,3) & (4,6)& & \textbf{0.067} & 0.216 &  & \textbf{0.465} & 0.458 & 1.50e-7 \\
(4,3) & (1,6)& & \textbf{0.110} & 0.213 &  & \textbf{0.471} & 0.467 & 4.84e-7 \\
(4,4) & (1,6)& & \textbf{0.105} & 0.210 &  & \textbf{0.476} & 0.469 & 2.79e-9 \\ \hline
\end{tabular}
    }
\end{minipage}
\begin{minipage}{0.9\linewidth}
    \footnotesize
    \vspace{0.1cm}
    This table reports the training time (s) and out-of-sample $R^2$ of our method and $\alpha$-PCA on the CPI dataset under various assumed numbers of factors, along with $p$-values based on the Diebold–Mariano test. The dataset contains 27 expenditure items for 52 countries over 1343 months. We take the earliest 95\% of the observations as a training set, and the remaining 5\% as a test set. The out-of-sample $R^2$ is computed according to Equation (\ref{outofsampler2}), with tuning parameters $s_1$ and $s_2$ selected to maximize out-of-sample $R^2$. In our method, we set $s_1,s_2\in\{1,2,...,10\}$ with even partitions, and in $\alpha$-PCA we use $\alpha=0$. The eigenvalue ratio-based method suggests that $(k,r)=(3,2)$.
\end{minipage}
\end{table}

Next, we analyze connections among countries and among expenditure items by clustering the loading matrix estimators. Guided by the results in Table~\ref{table_OECD}, we set $(k,r)=(3,2)$ and $(s_1,s_2)=(4,9)$ for the subsequent analysis, and estimate the model using the full dataset. For better illustration, the estimators are normalized and rotated using the VARIMAX criterion \citep{promax}. Clustering our rotated estimators $\widehat{\bR}_\text{ours,rot}$ and $\widehat{\bC}_\text{ours,rot}$, together with the $\alpha$-PCA estimators $\widehat{\bR}_{\alpha\text{,rot}}$ and $\widehat{\bC}_{\alpha\text{,rot}}$), yields groupings of countries and expenditure items, respectively. The rotated estimators and the corresponding clusters are reported in Table B1 and Table B2 in Appendix B, and Table~\ref{table_oecd_cluster_loadingmatrix} summarizes the mean loadings within each cluster for our rotated estimators.

From the rotated estimator $\widehat{\bR}_\text{ours,rot}$, we can classify the countries into four groups: Group 1 (TUR), Group 2 (DEU, JPN, GBR, FRA, KOR, CHE, SWE, IRL, ...), Group 3 (G20, OECD, EU27, G7, CHN, IND, BRA, RUS, ...), and Group 4 (USA, ITA, CAN, MEX, ESP, NLD, POL, BEL, ...). Table B1 and Table \ref{table_oecd_cluster_loadingmatrix} show that Group 1, Turkey, loads most heavily on Column 1, reflecting its distinct high inflation among large economies. Group 2, mainly consisting of developed economies with advanced manufacturing sectors, exhibits negative loadings on Columns 2 and 3. Group 3, which comprises large developing economies, loads primarily on column 3. Finally, Group 4, comprising countries with strong economic ties to the United States but less developed manufacturing sectors, loads negatively on Column 2 and positively on Column 3. Overall, these clusters reflect similarities in domestic economic patterns. Moreover, the $\alpha$-PCA estimators produce groupings for larger economies that are broadly consistent with our clusters, while some differences arise for smaller economies, suggesting the differences in how the two methods capture localized heterogeneity. An analogous analysis applies to the groupings of expenditure items and is therefore omitted.


\begin{table}[h]
\centering
\begin{minipage}{0.8\linewidth}
        \caption{\raggedright Mean loadings by cluster of the estimators for the CPI dataset.
    } 
    \label{table_oecd_cluster_loadingmatrix}\centering
    \resizebox{\linewidth}{!}{
    \centering
\begin{tabular}{ccp{0.25cm}cccc}
\hline
Estimator & Column& & Group1 & Group2 & Group3 & Group4 \\ \hline
\multirow{3}{*}{$\widehat{\bR}_\text{ours,rot}$} & Column1& & 9 & 0 & 0 & 0 \\
 & Column2& & -1 & -2 & 0 & -1 \\
 & Column3& & -2 & -1 & 1 & 1 \\ \hline
\multirow{2}{*}{$\widehat{\bC}_\text{ours,rot}$} & Column1& & -3 & 0 & -2 &  \\
 & Column2& & -2 & 3 & 0 &  \\ \cline{1-6}
\end{tabular}
    }
\end{minipage}
\begin{minipage}{0.8\linewidth}
    \footnotesize
    \vspace{0.1cm}
     This table reports the mean loadings by cluster for our rotated estimators $\widehat{\bR}_\text{ours,rot}$ and $\widehat{\bC}_\text{ours,rot}$ based on the CPI dataset. All values are multiplied by 10. The dataset contains 27 expenditure items for 52 countries over 1343 months. The factor numbers are set to $(k,r)=(3,2)$, with $(s_1,s_2)=(4,9)$ and even partitions in our method, and $\alpha=0$ in $\alpha$-PCA.
\end{minipage}

\end{table}

\section{Concluding Remarks}\label{sec5}
We proposed a distributed estimation procedure for modeling both stationary and unit-root nonstationary high-dimensional matrix-valued time series. The method preserved the intrinsic matrix structure and introduced a novel partitioning regime, which enabled accurate and computationally efficient extraction of latent factors while maintaining the overall common structure and respecting inter-group heterogeneity. We established statistical guarantees for the resulting estimators, including consistency and asymptotic normality. Numerical studies further demonstrated that the proposed method achieved substantial computational gains without sacrificing reliable finite-sample performance. These properties positioned the proposed procedure as a compelling alternative for factor modeling in high-dimensional stationary and nonstationary time series in the big data era.
\vspace*{-12pt}

%

\bibliography{references}

@article{ahn2013,
author = {Ahn, Seung C. and Horenstein, Alex R.},
title = {Eigenvalue Ratio Test for the Number of Factors},
journal = {Econometrica},
volume = {81},
number = {3},
pages = {1203-1227},
doi = {https://doi.org/10.3982/ECTA8968},
url = {https://onlinelibrary.wiley.com/doi/abs/10.3982/ECTA8968},
eprint = {https://onlinelibrary.wiley.com/doi/pdf/10.3982/ECTA8968},
year = {2013}
}

@article{chen2021,
   title={Statistical Inference for High-Dimensional Matrix-Variate Factor Models},
   volume={118},
   ISSN={1537-274X},
   url={http://dx.doi.org/10.1080/01621459.2021.1970569},
   DOI={10.1080/01621459.2021.1970569},
   number={542},
   journal={Journal of the American Statistical Association},
   publisher={Informa UK Limited},
   author={Chen, Elynn Y. and Fan, Jianqing},
   year={2021},
   month=oct, pages={1038–1055} }

@article{fan2018,
author = {Jianqing Fan and Dong Wang and Kaizheng Wang and Ziwei Zhu},
title = {{Distributed estimation of principal eigenspaces}},
volume = {47},
journal = {The Annals of Statistics},
number = {6},
publisher = {Institute of Mathematical Statistics},
pages = {3009 -- 3031},
year = {2019},
doi = {10.1214/18-AOS1713},
URL = {https://doi.org/10.1214/18-AOS1713}
}

@book{fan2003,
author = {Fan, Jianqing and Yao, Qiwei},
year = {2003},
month = {01},
pages = {},
title = {Nonlinear Time Series: Nonparametric and Parametric Methods},
publisher = {Springer New York, NY},
isbn = {978-0-387-26142-3},
doi = {10.1007/978-0-387-69395-8}
}

@book{franklin2012,
  title={Matrix Theory},
  author={Franklin, J.N.},
  isbn={9780486136387},
  series={Dover Books on Mathematics},
  url={https://books.google.com.tw/books?id=eXQXwCDD9agC},
  year={2012},
  publisher={Dover Publications}
}

@article{WANG2019231,
title = {Factor models for matrix-valued high-dimensional time series},
journal = {Journal of Econometrics},
volume = {208},
number = {1},
pages = {231-248},
year = {2019},
note = {Special Issue on Financial Engineering and Risk Management},
issn = {0304-4076},
doi = {https://doi.org/10.1016/j.jeconom.2018.09.013},
url = {https://www.sciencedirect.com/science/article/pii/S0304407618301787},
author = {Dong Wang and Xialu Liu and Rong Chen}
}

@Article{yu2020,
journal={Journal of Econometrics},
author={Yu, Long and He, Yong and Kong, Xinbing and Zhang, Xinsheng},
title={Projected estimation for large-dimensional matrix factor models},
year={2022},
month={None},
pages={201-217},
volume={229},
number={1},
doi={10.1016/j.jeconom.2021.04.001},
url={https://ideas.repec.org/a/eee/econom/v229y2022i1p201-217.html},
}

@article{Lam_2012,
   title={Factor modeling for high-dimensional time series: Inference for the number of factors},
   volume={40},
   ISSN={0090-5364},
   url={http://dx.doi.org/10.1214/12-AOS970},
   DOI={10.1214/12-aos970},
   number={2},
   journal={The Annals of Statistics},
   publisher={Institute of Mathematical Statistics},
   author={Lam, Clifford and Yao, Qiwei},
   year={2012},
   month=apr }

@article{Stock01122002,
author = {James H Stock and Mark W Watson},
title = {Forecasting Using Principal Components From a Large Number of Predictors},
journal = {Journal of the American Statistical Association},
volume = {97},
number = {460},
pages = {1167--1179},
year = {2002},
publisher = {ASA Website},
doi = {10.1198/016214502388618960},
URL = {https://doi.org/10.1198/016214502388618960},
eprint = { https://doi.org/10.1198/016214502388618960}
}

@article{bai2002,
author = {Bai, Jushan and Ng, Serena},
title = {Determining the Number of Factors in Approximate Factor Models},
journal = {Econometrica},
volume = {70},
number = {1},
pages = {191-221},
doi = {https://doi.org/10.1111/1468-0262.00273},
url = {https://onlinelibrary.wiley.com/doi/abs/10.1111/1468-0262.00273},
eprint = {https://onlinelibrary.wiley.com/doi/pdf/10.1111/1468-0262.00273},
year = {2002}
}

@techreport{Newey1986,
 title = "A Simple, Positive Semi-Definite, Heteroskedasticity and AutocorrelationConsistent Covariance Matrix",
 author = "Newey, Whitney K and West, Kenneth D",
 institution = "National Bureau of Economic Research",
 type = "Working Paper",
 series = "Technical Working Paper Series",
 number = "55",
 year = "1986",
 month = "April",
 doi = {10.3386/t0055},
 URL = "http://www.nber.org/papers/t0055",
}

@article{Qu2002,
author = {Qu, Yongming and Ostrouchov, George and Samatova, Nagiza and Geist, Al},
year = {2002},
month = {04},
pages = {},
title = {Principal Component Analysis for Dimension Reduction in Massive Distributed Data Sets},
journal = {Knowledge and Information Systems - KAIS}
}

@article{ROSS1976341,
title = {The arbitrage theory of capital asset pricing},
journal = {Journal of Economic Theory},
volume = {13},
number = {3},
pages = {341-360},
year = {1976},
issn = {0022-0531},
doi = {https://doi.org/10.1016/0022-0531(76)90046-6},
url = {https://www.sciencedirect.com/science/article/pii/0022053176900466},
author = {Stephen A Ross}
}

@article{GREGORY1999423,
title = {Common and country-specific fluctuations in productivity, investment, and the current account},
journal = {Journal of Monetary Economics},
volume = {44},
number = {3},
pages = {423-451},
year = {1999},
issn = {0304-3932},
doi = {https://doi.org/10.1016/S0304-3932(99)00035-5},
url = {https://www.sciencedirect.com/science/article/pii/S0304393299000355},
author = {Allan W. Gregory and Allen C. Head}
}

@article{Lewbel1991,
 ISSN = {00129682, 14680262},
 URL = {http://www.jstor.org/stable/2938225},
 author = {Arthur Lewbel},
 journal = {Econometrica},
 number = {3},
 pages = {711--730},
 publisher = {[Wiley, Econometric Society]},
 title = {The Rank of Demand Systems: Theory and Nonparametric Estimation},
 urldate = {2025-09-02},
 volume = {59},
 year = {1991}
}

@article{Donald1997,
 ISSN = {00129682, 14680262},
 URL = {http://www.jstor.org/stable/2171815},
 author = {Stephen G. Donald},
 journal = {Econometrica},
 number = {1},
 pages = {103--131},
 publisher = {[Wiley, Econometric Society]},
 title = {Inference Concerning the Number of Factors in a Multivariate Nonparametric Relationship},
 urldate = {2025-09-02},
 volume = {65},
 year = {1997}
}

@TechReport{Stock1998,
type={NBER Working Papers},
institution={National Bureau of Economic Research, Inc},
author={James H. Stock and Mark W. Watson},
title={Diffusion Indexes},
year={1998},
month={Aug},
number={6702},
doi={None},
url={https://ideas.repec.org/p/nbr/nberwo/6702.html},
}

@article{forni2000,
 ISSN = {00346535, 15309142},
 URL = {http://www.jstor.org/stable/2646650},
 author = {Mario Forni and Marc Hallin and Marco Lippi and Lucrezia Reichlin},
 journal = {The Review of Economics and Statistics},
 number = {4},
 pages = {540--554},
 publisher = {The MIT Press},
 title = {The Generalized Dynamic-Factor Model: Identification and Estimation},
 urldate = {2025-09-02},
 volume = {82},
 year = {2000}
}

@article{forni2005,
 ISSN = {01621459},
 URL = {http://www.jstor.org/stable/27590616},
 author = {Mario Forni and Marc Hallin and Marco Lippi and Lucrezia Reichlin},
 journal = {Journal of the American Statistical Association},
 number = {471},
 pages = {830--840},
 publisher = {[American Statistical Association, Taylor & Francis, Ltd.]},
 title = {The Generalized Dynamic Factor Model: One-Sided Estimation and Forecasting},
 urldate = {2025-09-02},
 volume = {100},
 year = {2005}
}

@article{FORNI2015359,
title = {Dynamic factor models with infinite-dimensional factor spaces: One-sided representations},
journal = {Journal of Econometrics},
volume = {185},
number = {2},
pages = {359-371},
year = {2015},
issn = {0304-4076},
doi = {https://doi.org/10.1016/j.jeconom.2013.10.017},
url = {https://www.sciencedirect.com/science/article/pii/S0304407614002693},
author = {Mario Forni and Marc Hallin and Marco Lippi and Paolo Zaffaroni}
}

@article{Gao_2021,
   title={Modeling High-Dimensional Time Series: A Factor Model With Dynamically Dependent Factors and Diverging Eigenvalues},
   volume={117},
   ISSN={1537-274X},
   url={http://dx.doi.org/10.1080/01621459.2020.1862668},
   DOI={10.1080/01621459.2020.1862668},
   number={539},
   journal={Journal of the American Statistical Association},
   publisher={Informa UK Limited},
   author={Gao, Zhaoxing and Tsay, Ruey S.},
   year={2021},
   month=feb, pages={1398–1414} }

@article{gao2019,
author = {Gao, Zhaoxing and Tsay, Ruey S.},
title = {A Structural-Factor Approach to Modeling High-Dimensional Time Series and Space-Time Data},
journal = {Journal of Time Series Analysis},
volume = {40},
number = {3},
pages = {343-362},
doi = {https://doi.org/10.1111/jtsa.12466},
url = {https://onlinelibrary.wiley.com/doi/abs/10.1111/jtsa.12466},
eprint = {https://onlinelibrary.wiley.com/doi/pdf/10.1111/jtsa.12466},
year = {2019}
}

@article{he2022matrixfactoranalysissquares,
author = {Yong He and Xinbing Kong and Long Yu and Xinsheng Zhang and Changwei Zhao},
title = {Matrix Factor Analysis: From Least Squares to Iterative Projection},
journal = {Journal of Business \& Economic Statistics},
volume = {42},
number = {1},
pages = {322--334},
year = {2024},
publisher = {ASA Website},
doi = {10.1080/07350015.2023.2191676},


URL = { 
    
        https://doi.org/10.1080/07350015.2023.2191676
    
    

},
eprint = { 
    
        https://doi.org/10.1080/07350015.2023.2191676
    
    

}

}

@article{he2022matrixkendallstauhighdimensions,
author = {Yong He and Yalin Wang and Long Yu and Wang Zhou and Wen-Xin Zhou},
title = {{A new non-parametric Kendall’s tau for matrix-valued elliptical observations}},
volume = {31},
journal = {Bernoulli},
number = {4},
publisher = {Bernoulli Society for Mathematical Statistics and Probability},
pages = {3331 -- 3355},
year = {2025},
doi = {10.3150/24-BEJ1849},
URL = {https://doi.org/10.3150/24-BEJ1849}
}

@article{GAO202383,
title = {A Two-Way Transformed Factor Model for Matrix-Variate Time Series},
journal = {Econometrics and Statistics},
volume = {27},
pages = {83-101},
year = {2023},
issn = {2452-3062},
doi = {https://doi.org/10.1016/j.ecosta.2021.08.008},
url = {https://www.sciencedirect.com/science/article/pii/S2452306221001027},
author = {Zhaoxing Gao and Ruey S. Tsay}
}

@inproceedings{garber2017communicationefficientalgorithmsdistributedstochastic,
author = {Garber, Dan and Shamir, Ohad and Srebro, Nathan},
title = {Communication-efficient algorithms for distributed stochastic principal component analysis},
year = {2017},
publisher = {JMLR.org},
booktitle = {Proceedings of the 34th International Conference on Machine Learning - Volume 70},
pages = {1203–1212},
numpages = {10},
location = {Sydney, NSW, Australia},
series = {ICML'17}
}

@article{chen2021distributedestimationprincipalcomponent,
author = {Xi Chen and Jason D. Lee and He Li and Yun Yang},
title = {Distributed Estimation for Principal Component Analysis: An Enlarged Eigenspace Analysis},
journal = {Journal of the American Statistical Association},
volume = {117},
number = {540},
pages = {1775--1786},
year = {2022},
publisher = {ASA Website},
doi = {10.1080/01621459.2021.1886937},


URL = { 
    
        https://doi.org/10.1080/01621459.2021.1886937
    
    

},
eprint = { 
    
        https://doi.org/10.1080/01621459.2021.1886937
    
    

}

}

@article{Kargupta2000,
author = {Kargupta, Hillol and Huang, Weiyun and Sivakumar, Krishnamoorthy and Johnson, Erik},
year = {2000},
month = {10},
pages = {},
title = {Distributed Clustering Using Collective Principal Component Analysis},
volume = {3},
journal = {Knowledge and Information Systems},
doi = {10.1007/PL00011677}
}

@ARTICLE{Li2011,
  author={Li, Lin and Scaglione, Anna and Manton, Jonathan H.},
  journal={IEEE Journal of Selected Topics in Signal Processing}, 
  title={Distributed Principal Subspace Estimation in Wireless Sensor Networks}, 
  year={2011},
  volume={5},
  number={4},
  pages={725-738},
  doi={10.1109/JSTSP.2011.2118742}}

@article{BERTRAND2014120,
title = {Distributed adaptive estimation of covariance matrix eigenvectors in wireless sensor networks with application to distributed PCA},
journal = {Signal Processing},
volume = {104},
pages = {120-135},
year = {2014},
issn = {0165-1684},
doi = {https://doi.org/10.1016/j.sigpro.2014.03.037},
url = {https://www.sciencedirect.com/science/article/pii/S016516841400142X},
author = {Alexander Bertrand and Marc Moonen}
}

@article{chen2025distributedtensorprincipalcomponent,
author = {Elynn Chen and Xi Chen and Wenbo Jing and Yichen Zhang},
title = {Distributed Tensor Principal Component Analysis with Data Heterogeneity},
journal = {Journal of the American Statistical Association},
pages = {1--13},
year = {2025},
publisher = {ASA Website},
doi = {10.1080/01621459.2025.2483481},


URL = { 
    
        https://doi.org/10.1080/01621459.2025.2483481
    
    

},
eprint = { 
    
        https://doi.org/10.1080/01621459.2025.2483481
    
    

}

}

@article{JIANG2027,
author = {Hangjin Jiang and Baining Shen and Yuzhou Li and Zhaoxing Gao},
title = {Regularized Estimation of High-Dimensional Matrix-Variate Autoregressive Models [in press]},
journal = {Statistica Sinica},
volume = {37},
number = {4},
pages = {1--40},
year = {2027},
publisher = {Academia Sinica and International Chinese Statistical Association},
doi = {10.5705/ss.202024.0341},


URL = { 
    
        https://doi.org/10.5705/ss.202024.0341
    
    

},
eprint = { 
    
        https://doi.org/10.5705/ss.202024.0341
    
    

}

}

@article{promax,
author = {Hendrickson, Alan E. and White, Paul Owen},
title = {PROMAX: A QUICK METHOD FOR ROTATION TO OBLIQUE SIMPLE STRUCTURE},
journal = {British Journal of Statistical Psychology},
volume = {17},
number = {1},
pages = {65-70},
doi = {https://doi.org/10.1111/j.2044-8317.1964.tb00244.x},
url = {https://bpspsychub.onlinelibrary.wiley.com/doi/abs/10.1111/j.2044-8317.1964.tb00244.x},
eprint = {https://bpspsychub.onlinelibrary.wiley.com/doi/pdf/10.1111/j.2044-8317.1964.tb00244.x},
year = {1964}
}

@article{PENA20061237,
title = {Nonstationary dynamic factor analysis},
journal = {Journal of Statistical Planning and Inference},
volume = {136},
number = {4},
pages = {1237-1257},
year = {2006},
issn = {0378-3758},
doi = {https://doi.org/10.1016/j.jspi.2004.08.020},
url = {https://www.sciencedirect.com/science/article/pii/S0378375804003659},
author = {Daniel Peña and Pilar Poncela}
}

@misc{li2025factormodelsmatrixvaluedtime,
      title={Factor Models of Matrix-Valued Time Series: Nonstationarity and Cointegration [preprint, arXiv:2508.11358]}, 
      author={Degui Li and Yayi Yan and Qiwei Yao},
      year={2025},
      eprint={2508.11358},
      archivePrefix={arXiv},
      primaryClass={econ.EM},
      url={https://arxiv.org/abs/2508.11358}, 
}

@article{gao2021divideandconquerdistributedhierarchicalfactor,
author = {Zhaoxing Gao and Ruey S. Tsay},
title = {Divide-and-Conquer: A Distributed Hierarchical Factor Approach to Modeling Large-Scale Time Series Data},
journal = {Journal of the American Statistical Association},
volume = {118},
number = {544},
pages = {2698--2711},
year = {2023},
publisher = {ASA Website},
doi = {10.1080/01621459.2022.2071279},


URL = { 
    
        https://doi.org/10.1080/01621459.2022.2071279
    
    

},
eprint = { 
    
        https://doi.org/10.1080/01621459.2022.2071279
    
    

}

}

@article{BARIGOZZI2021455,
title = {Large-dimensional Dynamic Factor Models: Estimation of Impulse–Response Functions with I(1) cointegrated factors},
journal = {Journal of Econometrics},
volume = {221},
number = {2},
pages = {455-482},
year = {2021},
issn = {0304-4076},
doi = {https://doi.org/10.1016/j.jeconom.2020.05.004},
url = {https://www.sciencedirect.com/science/article/pii/S0304407620302219},
author = {Matteo Barigozzi and Marco Lippi and Matteo Luciani}
}

@techreport{Chamberlain1981,
 title = "Arbitrage and Mean-Variance Analysis on Large Asset Markets",
 author = "Chamberlain, Gary and Rothschild, Michael",
 institution = "National Bureau of Economic Research",
 type = "Working Paper",
 series = "Technical Working Paper Series",
 number = "15",
 year = "1981",
 month = "July",
 doi = {10.3386/t0015},
 URL = "http://www.nber.org/papers/t0015",
}

@article{CAIADO20062668,
title = {A periodogram-based metric for time series classification},
journal = {Computational Statistics \& Data Analysis},
volume = {50},
number = {10},
pages = {2668-2684},
year = {2006},
issn = {0167-9473},
doi = {https://doi.org/10.1016/j.csda.2005.04.012},
url = {https://www.sciencedirect.com/science/article/pii/S0167947305000770},
author = {Jorge Caiado and Nuno Crato and Daniel Peña}
}

@article{diebold1995comparing,
  title={Comparing Predictive Accuracy},
  author={Diebold, Francis X and Mariano, Roberto S},
  journal={Journal of Business \& Economic Statistics},
  volume={13},
  number={3},
  year={1995}
}

\newpage
\section*{Appendix A: Proof of Theorems}
\renewcommand{\thefigure}{A\arabic{figure}}
\renewcommand{\thetable}{A\arabic{table}}
\renewcommand{\theequation}{A\arabic{equation}}

\setcounter{table}{0}
\setcounter{figure}{0}
\setcounter{section}{0}
\setcounter{lemma}{0}
\setcounter{theorem}{0}
\setcounter{proposition}{0}
In this section, we establish the convergence and asymptotic normality of the proposed estimators for both stationary and unit-root nonstationary time series.
\subsection*{A.1 Propositions and lemmas for the stationary case}
We first present some propositions and lemmas for the stationary case. These results will be used to prove Theorems 1-5. We begin by establishing asymptotic properties of the local estimators and the associated auxiliary matrices. Following Lemma B.1 in \cite{yu2020}, we directly have Proposition \ref{prop1}.
\begin{proposition}\label{prop1}
Under Assumptions 1-6(a) with known $k$ and $r$, for any $i\in[s_1]$ and $j\in[s_2]$, we have
\begin{equation*}
    \frac{1}{p}\left\|\widehat{\mathbf{R}}_i-\mathbf{R H}_{\bR_i}\right\|_\mathsf{F}^2=\mathcal{O}_p\left(\frac{1}{m_iT}+\frac{1}{p^2}\right),\quad\text{as }m_i,p,T\to\infty,
\end{equation*}
and
\begin{equation*}
    \frac{1}{q}\left\|\widehat{\mathbf{C}}_j-\mathbf{C H}_{\bC_j}\right\|_\mathsf{F}^2=\mathcal{O}_p\left(\frac{1}{n_jT}+\frac{1}{q^2}\right),\quad\text{as }n_j,q,T\to\infty.
\end{equation*}
For any $l_1\in[p]$ and $l_2\in[q]$, we have
\begin{equation*}
    \left\|\widehat{\mathbf{R}}_i^{l_1}-\mathbf{R}^{l_1}\mathbf{ H}_{\bR_i}\right\|_2^2=\mathcal{O}_p\left(\frac{1}{m_iT}+\frac{1}{p^2}\right),\quad\text{as }m_i,p,T\to\infty,    
\end{equation*}
and
\begin{equation*}
    \left\|\widehat{\mathbf{C}}_j^{l_2}-\mathbf{C}^{l_2}\mathbf{ H}_{\bC_j}\right\|_2^2=\mathcal{O}_p\left(\frac{1}{n_jT}+\frac{1}{q^2}\right),\quad\text{as }n_j,q,T\to\infty.
\end{equation*}
\end{proposition}
By applying Propositions 1-3 and Theorem 2 in \cite{chen2021}, we directly obtain Propositions \ref{prop2}-\ref{prop5}.
\begin{proposition}\label{prop2}
Under Assumptions 1-6(a) with known $k$ and $r$, for any $i\in[s_1]$ and $j\in[s_2]$, we have $\|\bH_{\bR_i}\|_2=\mathcal{O}_p(1)$, and $\|\bH_{\bC_j}\|_2=\mathcal{O}_p(1)$.
\end{proposition}

\begin{proposition}\label{prop5}
    Under Assumptions 1-6(a) with known $k$ and $r$, for any $i\in[s_1]$ and $j\in[s_2]$, we have
    \begin{enumerate}
    \item[(a)] \begin{equation*}
     \widehat{\mathbf{V}}_{\mathbf{R}_i}\overset{\mathcal{P}}{\longrightarrow}  \mathbf{V}_{\mathbf{R}_i},\quad\text{as }m_i,p,T\to\infty,\quad\text{and}\quad
     \widehat{\mathbf{V}}_{\mathbf{C}_j}\overset{\mathcal{P}}{\longrightarrow}  \mathbf{V}_{\mathbf{C}_j},\quad\text{as }n_j,q,T\to\infty,
    \end{equation*}
    with $\|\widehat{\mathbf{V}}_{\mathbf{R}_i}\|_2=\mathcal{O}_p(1)$, $\|\widehat{\mathbf{V}}_{\mathbf{R}_i}^{-1}\|_2=\mathcal{O}_p(1)$, $\|\widehat{\mathbf{V}}_{\mathbf{C}_j}\|_2=\mathcal{O}_p(1)$, and $\|\widehat{\mathbf{V}}_{\mathbf{C}_j}^{-1}\|_2=\mathcal{O}_p(1)$.
    \item[(b)] 
    $$\frac{\widehat{\mathbf{R}}_i^\top\mathbf{R}}{p}\overset{\mathcal{P}}{\longrightarrow}\mathbf{Q}_{\mathbf{R}_i},\quad\text{as }m_i,p,T\to\infty,\quad\text{and}\quad\frac{\widehat{\mathbf{C}}_j^\top\mathbf{C}}{q}\overset{\mathcal{P}}{\longrightarrow}\mathbf{Q}_{\mathbf{C}_j},\quad\text{as }n_j,q,T\to\infty.$$
    \item[(c)] 
    $$ \mathbf{H}_{\mathbf{R}_i}\overset{\mathcal{P}}{\longrightarrow}\mathbf{Q}_{\mathbf{R}_i}^{-1},\quad\text{as }m_i,p,T\to\infty,\quad\text{and}\quad\mathbf{H}_{\mathbf{C}_j}\overset{\mathcal{P}}{\longrightarrow}\mathbf{Q}_{\mathbf{C}_j}^{-1},\quad\text{as }n_j,q,T\to\infty.$$
    \item[(d)] 
    
    When ${\sqrt{m_iT}}/{p}=\mathbf{o}(1)$, for any $l_1\in[p]$,
    \begin{align*}
        \sqrt{m_iT}\left(\widehat{\bR}_i^{l_1}-\bR^{l_1}\bH_{\bR_i}\right)^\top&= \widehat{\mathbf{V}}_{\mathbf{R}_i}^{-1}\frac{\widehat{\bR}_i^\top\bR}{p}\frac{1}{\sqrt{m_iT}}\sum_{t=1}^T\widetilde{\bF}_t\bC_i^\top(\widetilde{\bE}_{it}^{l_1})^\top+\mathbf{o}_p(1) \\
        &\overset{\mathcal{D}}{\longrightarrow}\mathcal{N}\left(\mathbf{0},\boldsymbol{\Sigma}_{\bR_{i,{l_1}}}\right),\quad\text{as }m_i,p,T\to\infty.
    \end{align*}
    When ${\sqrt{n_jT}}/{q}=\mathbf{o}(1)$, for any $l_2\in[q]$,
    \begin{align*}
        \sqrt{n_jT}\left(\widehat{\bC}_j^{l_2}-\bC^{l_2}\bH_{\bC_j}\right)^\top&= \widehat{\mathbf{V}}_{\mathbf{C}_j}^{-1}\frac{\widehat{\bC}_j^\top\bC}{q}\frac{1}{\sqrt{n_jT}}\sum_{t=1}^T\widetilde{\bF}_t^\top\bR_j^\top[(\widetilde{\boldsymbol{\epsilon}}_{jt}^\top)^{l_2}]^\top+\mathbf{o}_p(1) \\
        &\overset{\mathcal{D}}{\longrightarrow}\mathcal{N}\left(\mathbf{0},\boldsymbol{\Sigma}_{\bC_{j,{l_2}}}\right),\quad\text{as }n_j,q,T\to\infty.
    \end{align*}
    The covariance matrices
$$\boldsymbol{\Sigma}_{\bR_{i,{l_1}}}=\bV_{\bR_{i}}^{-1}\bQ_{\bR_{i}}\left(\mathbf{\Phi}_{R_i,{l_1}}^{1,1}+\alpha\mathbf{\Phi}_{R_i,{l_1}}^{1,2}\overline{\bF}^\top+\alpha\overline{\bF}\mathbf{\Phi}_{R_i,{l_1}}^{2,1}+\alpha^2\overline{\bF}\mathbf{\Phi}_{R_i,{l_1}}^{2,2}\overline{\bF}^\top\right)\bQ_{\bR_{i}}^\top\bV_{\bR_{i}}^{-1},$$
    and
    $$\boldsymbol{\Sigma}_{\bC_{j,{l_2}}}=\bV_{\bC_{j}}^{-1}\bQ_{\bC_{j}}\left(\mathbf{\Phi}_{C_j,{l_2}}^{1,1}+\alpha\mathbf{\Phi}_{C_j,{l_2}}^{1,2}\overline{\bF}+\alpha\overline{\bF}^\top\mathbf{\Phi}_{C_j,{l_2}}^{2,1}+\alpha^2\overline{\bF}^\top\mathbf{\Phi}_{C_j,{l_2}}^{2,2}\overline{\bF}\right)\bQ_{\bC_{j}}^\top\bV_{\bC_{j}}^{-1},$$
    where
    $$\mathbf{\Phi}_{R_i,{l_1}}^{1,1}=\underset{q, T \rightarrow \infty}{\operatorname{plim}} \frac{1}{m_i T} \sum_{t=1}^{T} \sum_{s=1}^{T} \mathbb{E}\left[\mathbf{F}_{t} \mathbf{C}_i^\top ({\bE}_{it}^{l_1})^\top {\bE}_{is}^{l_1} \mathbf{C}_i \mathbf{F}_{s}^\top\right],$$
    $$\mathbf{\Phi}_{R_i,{l_1}}^{1,2}=(\mathbf{\Phi}_{R_i,{l_1}}^{2,1})^\top=\underset{q, T \rightarrow \infty}{\operatorname{plim}} \frac{1}{m_i T} \sum_{t=1}^{T} \sum_{s=1}^{T} \mathbb{E}\left[\mathbf{F}_{t} \mathbf{C}_i^\top ({\bE}_{it}^{l_1})^\top {\bE}_{is}^{l_1} \mathbf{C}_i\right], $$
    and
    $$\mathbf{\Phi}_{R_i,{l_1}}^{2,2}=\underset{q, T \rightarrow \infty}{\operatorname{plim}} \frac{1}{m_i T} \sum_{t=1}^{T} \sum_{s=1}^{T} \mathbb{E}\left[\mathbf{C}_i^\top ({\bE}_{it}^{l_1})^\top {\bE}_{is}^{l_1} \mathbf{C}_i\right],$$
    while $\mathbf{\Phi}_{C_j,{l_2}}^{1,1}$, $\mathbf{\Phi}_{C_j,{l_2}}^{1,2}$, $\mathbf{\Phi}_{C_j,{l_2}}^{2,1}$, and $\mathbf{\Phi}_{C_j,{l_2}}^{2,2}$ are defined similarly.
    \end{enumerate}
    Here the matrices $\mathbf{Q}_{\mathbf{R}_i}=\mathbf{V}_{\mathbf{R}_i}^{1/2}\Psi_{\mathbf{R}_i}^\top\widetilde{\boldsymbol{\Sigma}}_{\mathbf{FC}_i}^{-1/2}$ and $\mathbf{Q}_{\mathbf{C}_j}=\mathbf{V}_{\mathbf{C}_j}^{1/2}\Psi_{\mathbf{C}_j}^\top\widetilde{\boldsymbol{\Sigma}}_{\mathbf{FR}_j}^{-1/2}$, where $\mathbf{V}_{\mathbf{R}_i}$ and $\mathbf{V}_{\mathbf{C}_j}$ are diagonal matrices of the eigenvalues (ordered decreasingly) of $\widetilde{\boldsymbol{\Sigma}}_{\mathbf{FC}_i}^{1/2}\boldsymbol{\Omega}_{\mathbf{R}}\widetilde{\boldsymbol{\Sigma}}_{\mathbf{FC}_i}^{1/2}$ and $\widetilde{\boldsymbol{\Sigma}}_{\mathbf{FR}_j}^{1/2}\boldsymbol{\Omega}_{\mathbf{C}}\widetilde{\boldsymbol{\Sigma}}_{\mathbf{FR}_j}^{1/2}$, respectively, with the corresponding eigenvector matrices $\Psi_{\mathbf{R}_i}$ and $\Psi_{\mathbf{C}_j}$ satisfying $\Psi_{\mathbf{R}_i}^\top\Psi_{\mathbf{R}_i}=\mathbf{I}$ and $\Psi_{\mathbf{C}_j}^\top\Psi_{\mathbf{C}_j}=\mathbf{I}$, along with $\widetilde{\boldsymbol{\Sigma}}_{\mathbf{FC}_i}=\frac{1}{m_i}\mathbb{E}[\widetilde{\mathbf{F}}_t{\mathbf{C}_i^\top\mathbf{C}_i}\widetilde{\mathbf{F}}_t^\top]$ and $\widetilde{\boldsymbol{\Sigma}}_{\mathbf{FR}_j}=\frac{1}{n_j}\mathbb{E}[\widetilde{\mathbf{F}}_t^\top{\mathbf{R}_j^\top\mathbf{R}_j}\widetilde{\mathbf{F}}_t]$.
\end{proposition}
We next derive asymptotic properties of the global estimators and the associated auxiliary matrices.
\begin{proposition}\label{prop6}
	Under Assumptions 1-6(a) with known $k$ and $r$, as $\underline{m},\underline{n},T\to \infty$, we have
 \begin{align*}
     \widehat{\mathbf{V}}_\mathbf{R}\overset{\mathcal{P}}{\longrightarrow}  \mathbf{V_R},\quad\text{and}\quad\widehat{\mathbf{V}}_\mathbf{C}\overset{\mathcal{P}}{\longrightarrow} \mathbf{V_C},
 \end{align*}
 with $\|\widehat{\mathbf{V}}_{\mathbf{R}}\|_2=\mathcal{O}_p(1)$, $\|\widehat{\mathbf{V}}_{\mathbf{R}}^{-1}\|_2=\mathcal{O}_p(1)$, $\|\widehat{\mathbf{V}}_{\mathbf{C}}\|_2=\mathcal{O}_p(1)$, and $\|\widehat{\mathbf{V}}_{\mathbf{C}}^{-1}\|_2=\mathcal{O}_p(1)$. Here the matrices $\mathbf{V_R}$ and $\mathbf{V_C}$ are diagonal matrices of eigenvalues (ordered decreasingly) of $\boldsymbol{\Sigma}_\mathbf{H_R}^{1/2}\boldsymbol{\Omega}_\mathbf{R}\boldsymbol{\Sigma}_\mathbf{H_R}^{1/2}$ and  $\boldsymbol{\Sigma}_\mathbf{H_C}^{1/2}\boldsymbol{\Omega}_\mathbf{C}\boldsymbol{\Sigma}_\mathbf{H_C}^{1/2}$, respectively, with $\boldsymbol{\Sigma}_\mathbf{H_R}=\frac{1}{s_1}\sum_{i=1}^{s_1}\mathbf{Q}_{\mathbf{R}_i}^{-1}(\mathbf{Q}_{\mathbf{R}_i}^{-1})^\top$ and $ \boldsymbol{\Sigma}_\mathbf{H_C}=\frac{1}{s_2}\sum_{j=1}^{s_2}\mathbf{Q}_{\mathbf{C}_j}^{-1}(\mathbf{Q}_{\mathbf{C}_j}^{-1})^\top$.
\end{proposition}
\begin{proof}
    We only prove the conclusion for $\hat{\bV}_\bR$, since the proof of results for $\hat{\bV}_\bC$ is similar. According to Equation (2.6) and $\frac{1}{p}\wh{\bR}^\top\wh{\bR}=\bI$, we have
    \begin{equation}
        \widehat{\mathbf{V}}_\mathbf{R}=\frac{1}{p}\widehat{\mathbf{R}}^\top\widehat{\mathbf{M}}_\mathbf{R}\widehat{\mathbf{R}}=\frac{1}{p}\widehat{\mathbf{R}}^\top\left( \frac{1}{ps_1}\sum_{i=1}^{s_1}\widehat{\mathbf{R}}_i\widehat{\mathbf{R}}_i^\top\right)\widehat{\mathbf{R}}. \label{equation1_prop6}
    \end{equation}
    Writing $\wh{\bR}_i$ as $\wh{\bR}_i=\wh{\bR}_i-\bR\bH_{\bR_i}+\bR\bH_{\bR_i}$, we obtain
    \begin{align*}
    \widehat{\mathbf{M}}_\mathbf{R}&=\frac{1}{ps_1}\sum_{i=1}^{s_1}(\widehat{\mathbf{R}}_i-\mathbf{R}\mathbf{H}_{\mathbf{R}_i})(\widehat{\mathbf{R}}_i-\mathbf{R}\mathbf{H}_{\mathbf{R}_i})^\top+\frac{1}{ps_1}\sum_{i=1}^{s_1}\mathbf{R}\mathbf{H}_{\mathbf{R}_i}(\widehat{\mathbf{R}}_i-\mathbf{R}\mathbf{H}_{\mathbf{R}_i})^\top\\
    &\ \ \ +\frac{1}{ps_1}\sum_{i=1}^{s_1}(\widehat{\mathbf{R}}_i-\mathbf{R}\mathbf{H}_{\mathbf{R}_i})\mathbf{H}_{\mathbf{R}_i}^\top\mathbf{R}^\top+\frac{1}{ps_1}\sum_{i=1}^{s_1}\mathbf{R}\mathbf{H}_{\mathbf{R}_i}\mathbf{H}_{\mathbf{R}_i}^\top\mathbf{R}^\top.
    \end{align*}
    By Propositions~\ref{prop1}-\ref{prop2} and Assumption 3, we have
    $$\left \| \frac{1}{ps_1}\sum_{i=1}^{s_1}(\widehat{\mathbf{R}}_i-\mathbf{R}\mathbf{H}_{\mathbf{R}_i})(\widehat{\mathbf{R}}_i-\mathbf{R}\mathbf{H}_{\mathbf{R}_i})^\top\right \|_2 =\mathcal{O}_p\left(\frac{1}{\underline{m}T}+\frac{1}{p^2} \right),$$
    $$\left \| \frac{1}{ps_1}\sum_{i=1}^{s_1}\mathbf{R}\mathbf{H}_{\mathbf{R}_i}(\widehat{\mathbf{R}}_i-\mathbf{R}\mathbf{H}_{\mathbf{R}_i})^\top\right \|_2 =\mathcal{O}_p\left(\left(\frac{1}{\underline{m}T}+\frac{1}{p^2} \right)^{1/2}\right),$$
    and
    $$ \left \| \frac{1}{ps_1}\sum_{i=1}^{s_1}(\widehat{\mathbf{R}}_i-\mathbf{R}\mathbf{H}_{\mathbf{R}_i})\mathbf{H}_{\mathbf{R}_i}^\top\mathbf{R}^\top\right \|_2 =\mathcal{O}_p\left(\left(\frac{1}{\underline{m}T}+\frac{1}{p^2} \right)^{1/2}\right).$$
    In other words,
    $$\left \| \widehat{\mathbf{M}}_\mathbf{R}-\frac{1}{ps_1}\sum_{i=1}^{s_1}\mathbf{R}\mathbf{H}_{\mathbf{R}_i}\mathbf{H}_{\mathbf{R}_i}^\top\mathbf{R}^\top\right \|_2=\mathcal{O}_p\left(\left(\frac{1}{\underline{m}T}+\frac{1}{p^2} \right)^{1/2}\right).$$
    On the other hand, we have
    \begin{align*}
        &\left \| \frac{1}{ps_1}\sum_{i=1}^{s_1}\mathbf{R}\mathbf{H}_{\mathbf{R}_i}\mathbf{H}_{\mathbf{R}_i}^\top\mathbf{R}^\top-\frac{1}{ps_1}\sum_{i=1}^{s_1}\mathbf{R}\mathbf{Q}_{\mathbf{R}_i}^{-1}(\mathbf{Q}_{\mathbf{R}_i}^{-1})^\top\mathbf{R}^\top\right \|_2\\
        &\le \frac{1}{s_1}\sum_{i=1}^{s_1}\left \| \mathbf{H}_{\mathbf{R}_i}\mathbf{H}_{\mathbf{R}_i}^\top-\mathbf{Q}_{\mathbf{R}_i}^{-1}(\mathbf{Q}_{\mathbf{R}_i}^{-1})^\top\right \|_2\frac{1}{p}\left \|\mathbf{R}\right \|_2^2\\ &= \mathbf{o}_p\left(1\right).
    \end{align*}
    The conclusion follows from Assumption 3 and Proposition~\ref{prop5}(c), where $\mathbf{H}_{\mathbf{R}_i}\overset{\mathcal{P}}{\longrightarrow}\mathbf{Q}_{\mathbf{R}_i}^{-1}$, and thus $\mathbf{H}_{\mathbf{R}_i}\mathbf{H}_{\mathbf{R}_i}^\top\overset{\mathcal{P}}{\longrightarrow}\mathbf{Q}_{\mathbf{R}_i}^{-1}(\mathbf{Q}_{\mathbf{R}_i}^{-1})^\top$, as $m_i,p,T\to \infty$. Together, we have
    $$\left \| \widehat{\mathbf{M}}_\mathbf{R}-\frac{1}{ps_1}\sum_{i=1}^{s_1}\mathbf{R}\mathbf{Q}_{\mathbf{R}_i}^{-1}(\mathbf{Q}_{\mathbf{R}_i}^{-1})^\top\mathbf{R}^\top\right \|_2=\left \| \widehat{\mathbf{M}}_\mathbf{R}-\frac{1}{p}\mathbf{R}\boldsymbol{\Sigma}_\mathbf{H_R}\mathbf{R}^\top\right \|_2=\mathbf{o}_p\left(1\right).$$
    By the inequality that for the $j$-th largest eigenvalue of any square matrices $\bA$ and $\wh{\bA}$, $|\lambda_j(\widehat{\mathbf{A}})-\lambda_j(\mathbf{A})|\le \| \widehat{\mathbf{A}}-\mathbf{A}\|_2$, as $\underline{m},p,T\to \infty$, we have
    $$\widehat{\mathbf{V}}_\mathbf{R}\overset{\mathcal{P}}{\longrightarrow}  \mathbf{V_R},$$
    where the eigenvalues of $\frac{1}{p}\mathbf{R}\boldsymbol{\Sigma}_\mathbf{H_R}\mathbf{R}^\top$ and $\frac{1}{p}\mathbf{R}^\top\mathbf{R}\boldsymbol{\Sigma}_\mathbf{H_R}$ are the same, $\frac{1}{p}\mathbf{R}^\top\mathbf{R}\to\boldsymbol{\Omega}_\mathbf{R}$ by Assumption 3, and the eigenvalues of $\boldsymbol{\Omega}_\mathbf{R}\boldsymbol{\Sigma}_\mathbf{H_R}$ and $\boldsymbol{\Sigma}_\mathbf{H_R}^{1/2}\boldsymbol{\Omega}_\mathbf{R}\boldsymbol{\Sigma}_\mathbf{H_R}^{1/2}$ are also the same since $\boldsymbol{\Omega}_\mathbf{R}$ and $\boldsymbol{\Sigma}_\mathbf{H_R}$ are both positive definite. Furthermore, the top $k$ eigenvalues of $\widehat{\mathbf{M}}_\mathbf{R}$ are bounded away from zero and infinity, and therefore we have $\|\widehat{\mathbf{V}}_{\mathbf{R}}\|_2=\mathcal{O}_p(1)$ and $\|\widehat{\mathbf{V}}_{\mathbf{R}}^{-1}\|_2=\mathcal{O}_p(1)$.
\end{proof}

\begin{lemma}\label{lem2}
	Under Assumptions 1-6(a) with known $k$ and $r$, we have $\|\bH_{\bR}\|_2=\mathcal{O}_p(1)$ and $\|\bH_{\bC}\|_2=\mathcal{O}_p(1)$. 
\end{lemma}
\begin{proof}
	We only prove the result for $\|\bH_{\bR}\|_2$, since the proof for $\|\bH_{\bC}\|_2$ is similar.
	Recall that $\bH_\bR=\frac{1}{ps_1}\sum_{i=1}^{s_1} \bH_{\bR_i}\bH_{\bR_i}^\top\bR^\top\hat{\bR}\widehat{\bV}_{\bR}^{-1}$. The conclusion follows directly from Assumption 3, Propositions \ref{prop2} and \ref{prop6}.
\end{proof}

\begin{lemma}\label{lem3}
    Under Assumptions 1-6(a) with known $k$ and $r$, for any $i\in[p]$, as $\underline{m},p,T\to\infty$, we have
    \begin{enumerate}
    \item[(a)] $\frac{1}{ps_1}\sum_{j=1}^{s_1}(\widehat{\mathbf{R}}_j^i-\mathbf{R}^i\mathbf{H}_{\mathbf{R}_j})(\widehat{\mathbf{R}}_{j}-\mathbf{R}\mathbf{H}_{\mathbf{R}_j})^\top\widehat{\mathbf{R}}=\mathcal{O}_p(\frac{1}{\underline{m}T}+\frac{1}{p^2})$,
    \item[(b)] $\frac{1}{ps_1}\sum_{j=1}^{s_1}(\widehat{\mathbf{R}}_j^i-\mathbf{R}^i\mathbf{H}_{\mathbf{R}_j})\mathbf{H}_{\mathbf{R}_j}^\top\mathbf{R}^\top\widehat{\mathbf{R}}=\mathcal{O}_p((\frac{1}{\underline{m}T}+\frac{1}{p^2})^{1/2})$,
    \item[(c)] $\frac{1}{ps_1}\sum_{j=1}^{s_1}\mathbf{R}^i\mathbf{H}_{\mathbf{R}_j}(\widehat{\mathbf{R}}_{j}-\mathbf{R}\mathbf{H}_{\mathbf{R}_j})^\top\widehat{\mathbf{R}}=\mathcal{O}_p((\frac{1}{\underline{m}T}+\frac{1}{p^2})^{1/2})$.
    \end{enumerate}
\end{lemma}
\begin{proof}
(a) First, we expand the last two terms as follows
\begin{align*}
    \mathrm{I}&=\frac{1}{ps_1}\sum_{j=1}^{s_1}(\widehat{\mathbf{R}}_j^i-\mathbf{R}^i\mathbf{H}_{\mathbf{R}_j})(\widehat{\mathbf{R}}_{j}-\mathbf{R}\mathbf{H}_{\mathbf{R}_j})^\top\widehat{\mathbf{R}}\\
    &=\frac{1}{ps_1}\sum_{j=1}^{s_1}(\widehat{\mathbf{R}}_j^i-\mathbf{R}^i\mathbf{H}_{\mathbf{R}_j})(\widehat{\mathbf{R}}_j-\mathbf{R}\mathbf{H}_{\mathbf{R}_j})^\top(\widehat{\mathbf{R}}-\mathbf{R}\mathbf{H}_{\mathbf{R}})\\
    &\ \ \ +\frac{1}{ps_1}\sum_{j=1}^{s_1}\sum_{l=1}^{p}(\widehat{\mathbf{R}}_j^i-\mathbf{R}^i\mathbf{H}_{\mathbf{R}_j})(\widehat{\mathbf{R}}_j-\mathbf{R}\mathbf{H}_{\mathbf{R}_j})^\top\mathbf{R}\mathbf{H}_{\mathbf{R}}\\
    &=\mathrm{I}_1+\mathrm{I}_2.
\end{align*}
We bound $\mathrm{I}_1$ by properties of norms and the Cauchy–Schwarz inequality as
\begin{align*}
    \left \|  \mathrm{I}_1 \right \|_2 &\le   \frac{1}{ps_1}\sum_{j=1}^{s_1}\left \|(\widehat{\mathbf{R}}_j^i-\mathbf{R}^i\mathbf{H}_{\mathbf{R}_j})(\widehat{\mathbf{R}}_j-\mathbf{R}\mathbf{H}_{\mathbf{R}_j})^\top(\widehat{\mathbf{R}}-\mathbf{R}\mathbf{H}_{\mathbf{R}}) \right \|_2\\
    &\le \frac{1}{ps_1}  \sum_{j=1}^{s_1}\left \| \widehat{\mathbf{R}}_{j}^i-\mathbf{R}^i\mathbf{H}_{\mathbf{R}_j}  \right \|_2\left \| \widehat{\mathbf{R}}_j-\mathbf{R}\mathbf{H}_{\mathbf{R}_j}\right \|_2\left \|\widehat{\mathbf{R}}-\mathbf{R}\mathbf{H}_{\mathbf{R}} \right \|_2 \\
    &\le \left(\frac{1}{s_1}\sum_{j=1}^{s_1}\left \| \widehat{\mathbf{R}}_{j}^i-\mathbf{R}^i\mathbf{H}_{\mathbf{R}_j}  \right \|_2^2\right)^{1/2}\left(\frac{1}{ps_1}\sum_{j=1}^{s_1}\left \|  \widehat{\mathbf{R}}_j-\mathbf{R}\mathbf{H}_{\mathbf{R}_j}   \right \|_2^2\right)^{1/2} \frac{1}{\sqrt{p}}\left \|\widehat{\mathbf{R}}-\mathbf{R}\mathbf{H}_{\mathbf{R}} \right \|_2.
\end{align*}
    By Proposition~\ref{prop1} and Theorem 1, we have
    $$\left \|  \mathrm{I}_1 \right \|_2=\mathcal{O}_p\left(\left(\frac{1}{\underline{m}T}+\frac{1}{p^2} \right)^{3/2}\right).$$
    Similarly, we bound $\mathrm{I}_2$ as
    $$\left \|  \mathrm{I}_2 \right \|_2\le \left(\frac{1}{s_1}\sum_{j=1}^{s_1}\left \| \widehat{\mathbf{R}}_{j}^i-\mathbf{R}^i\mathbf{H}_{\mathbf{R}_j}  \right \|_2^2\right)^{1/2}\left(\frac{1}{ps_1}\sum_{j=1}^{s_1}\left \|  \widehat{\mathbf{R}}_j-\mathbf{R}\mathbf{H}_{\mathbf{R}_j}   \right \|_2^2\right)^{1/2} \frac{\left \|\mathbf{R} \right \|_2}{\sqrt{p}}\left\|\mathbf{H}_{\mathbf{R}}\right\|_2.$$
    By Assumption 3, Proposition~\ref{prop1}, and Lemma~\ref{lem2}, we have
    $$\left \|  \mathrm{I}_2 \right \|_2=\mathcal{O}_p\left(\frac{1}{\underline{m}T}+\frac{1}{p^2}\right).$$
(b) First, we expand the last two terms as follows
\begin{align*}
    \mathrm{I}\hspace{-1.2pt}\mathrm{I}&=\frac{1}{ps_1}\sum_{j=1}^{s_1}(\widehat{\mathbf{R}}_j^i-\mathbf{R}^i\mathbf{H}_{\mathbf{R}_j})\mathbf{H}_{\mathbf{R}_j}^\top\mathbf{R}^\top\widehat{\mathbf{R}}\\
    &=\frac{1}{ps_1}\sum_{j=1}^{s_1}(\widehat{\mathbf{R}}_j^i-\mathbf{R}^i\mathbf{H}_{\mathbf{R}_j})\mathbf{H}_{\mathbf{R}_j}^\top\mathbf{R}^\top(\widehat{\mathbf{R}}-\mathbf{R}\mathbf{H}_{\mathbf{R}})\\
    &\ \ \ +\frac{1}{ps_1}\sum_{j=1}^{s_1}(\widehat{\mathbf{R}}_j^i-\mathbf{R}^i\mathbf{H}_{\mathbf{R}_j})\mathbf{H}_{\mathbf{R}_j}^\top\mathbf{R}^\top\mathbf{R}\mathbf{H}_{\mathbf{R}}\\
    &=\mathrm{I}\hspace{-1.2pt}\mathrm{I}_1+\mathrm{I}\hspace{-1.2pt}\mathrm{I}_2.
\end{align*}
We bound $\mathrm{I}\hspace{-1.2pt}\mathrm{I}_1$ by properties of norms and the Cauchy–Schwarz inequality as
\begin{align*}
    \left \|  \mathrm{I}\hspace{-1.2pt}\mathrm{I}_1 \right \|_2&\le \frac{1}{ps_1}\sum_{j=1}^{s_1}\left \|(\widehat{\mathbf{R}}_j^i-\mathbf{R}^i\mathbf{H}_{\mathbf{R}_j})\mathbf{H}_{\mathbf{R}_j}^\top\mathbf{R}^\top(\widehat{\mathbf{R}}-\mathbf{R}\mathbf{H}_{\mathbf{R}}) \right \|_2\\
    &\le \frac{1}{ps_1}  \sum_{j=1}^{s_1}\left \| \widehat{\mathbf{R}}_{j}^i-\mathbf{R}^i\mathbf{H}_{\mathbf{R}_j}  \right \|_2\left\|\mathbf{H}_{\mathbf{R}_j}\right\|_2\left \|\mathbf{R}\right \|_2\left \|\widehat{\mathbf{R}}-\mathbf{R}\mathbf{H}_{\mathbf{R}}\right \|_2 \\
    &= \left(\frac{1}{s_1}\sum_{j=1}^{s_1}\left \| \widehat{\mathbf{R}}_{j}^i-\mathbf{R}^i\mathbf{H}_{\mathbf{R}_j}  \right \|_2^2\right)^{1/2}\left(\frac{1}{s_1}\sum_{j=1}^{s_1}\left \| \mathbf{H}_{\mathbf{R}_j}  \right \|_2^2\right)^{1/2}\frac{\left \|\mathbf{R}\right \|_2}{\sqrt{p}} \frac{1}{\sqrt{p}}\left \|\widehat{\mathbf{R}}-\mathbf{R}\mathbf{H}_{\mathbf{R}} \right \|_2.
\end{align*}
 By Assumption 3, Propositions~\ref{prop1}-\ref{prop2} and Theorem 1, we have
    $$\left \|  \mathrm{I}\hspace{-1.2pt}\mathrm{I}_1 \right \|_2=\mathcal{O}_p\left(\frac{1}{\underline{m}T}+\frac{1}{p^2}\right).$$
Similarly, we bound $\mathrm{I}\hspace{-1.2pt}\mathrm{I}_2$ as
$$\left \|  \mathrm{I}\hspace{-1.2pt}\mathrm{I}_2 \right \|_2\le \left(\frac{1}{s_1}\sum_{j=1}^{s_1}\left \| \widehat{\mathbf{R}}_{j}^i-\mathbf{R}^i\mathbf{H}_{\mathbf{R}_j}  \right \|_2^2\right)^{1/2}\left(\frac{1}{s_1}\sum_{j=1}^{s_1}\left \| \mathbf{H}_{\mathbf{R}_j}  \right \|_2^2\right)^{1/2}\frac{\left \|\mathbf{R}\right \|_2^2}{p}\left\|\mathbf{H}_{\mathbf{R}}\right\|_2.$$
By Assumption 3, Propositions~\ref{prop1}-\ref{prop2} and Lemma~\ref{lem2}, we have
$$\left \|  \mathrm{I}\hspace{-1.2pt}\mathrm{I}_2 \right \|_2=\mathcal{O}_p\left(\left(\frac{1}{\underline{m}T}+\frac{1}{p^2}\right)^{1/2}\right).$$
(c) First, we expand the last two terms as follows
\begin{align*}
    \mathrm{I}\hspace{-1.2pt}\mathrm{I}\hspace{-1.2pt}\mathrm{I}&=\frac{1}{ps_1}\sum_{j=1}^{s_1}\mathbf{R}^i\mathbf{H}_{\mathbf{R}_j}(\widehat{\mathbf{R}}_{j}-\mathbf{R}\mathbf{H}_{\mathbf{R}_j})^\top\widehat{\mathbf{R}}\\
    &=\frac{1}{ps_1}\sum_{j=1}^{s_1}\mathbf{R}^i\mathbf{H}_{\mathbf{R}_j}(\widehat{\mathbf{R}}_j-\mathbf{R}\mathbf{H}_{\mathbf{R}_j})^\top(\widehat{\mathbf{R}}-\mathbf{R}\mathbf{H}_{\mathbf{R}})\\
    &\ \ \ +\frac{1}{ps_1}\sum_{j=1}^{s_1}\mathbf{R}^i\mathbf{H}_{\mathbf{R}_j}(\widehat{\mathbf{R}}_j-\mathbf{R}\mathbf{H}_{\mathbf{R}_j})^\top\mathbf{R}\mathbf{H}_{\mathbf{R}}\\
&=\mathrm{I}\hspace{-1.2pt}\mathrm{I}\hspace{-1.2pt}\mathrm{I}_1+\mathrm{I}\hspace{-1.2pt}\mathrm{I}\hspace{-1.2pt}\mathrm{I}_2.
\end{align*}
We bound $\mathrm{I}\hspace{-1.2pt}\mathrm{I}\hspace{-1.2pt}\mathrm{I}_1$ by properties of norms and the Cauchy–Schwarz inequality as
\begin{align*}
    \left \|  \mathrm{I}\hspace{-1.2pt}\mathrm{I}\hspace{-1.2pt}\mathrm{I}_1 \right \|_2 &\le   \frac{1}{ps_1}\sum_{j=1}^{s_1}\left \|\mathbf{R}^i\mathbf{H}_{\mathbf{R}_j}(\widehat{\mathbf{R}}_j-\mathbf{R}\mathbf{H}_{\mathbf{R}_j})^\top(\widehat{\mathbf{R}}-\mathbf{R}\mathbf{H}_{\mathbf{R}}) \right \|_2\\
    &\le \frac{1}{ps_1}  \sum_{j=1}^{s_1}\left \|  \mathbf{R}^i\right\|_2\left\|\mathbf{H}_{\mathbf{R}_j}  \right \|_2\left \| \widehat{\mathbf{R}}_j-\mathbf{R}\mathbf{H}_{\mathbf{R}_j}\right \|_2\left \| \widehat{\mathbf{R}}-\mathbf{R}\mathbf{H}_{\mathbf{R}} \right \|_2 \\
    &\le  \left \|  \mathbf{R}^i\right\|_2\left(\frac{1}{s_1}\sum_{j=1}^{s_1}\left \| \mathbf{H}_{\mathbf{R}_j}  \right \|_2^2\right)^{1/2}\left(\frac{1}{ps_1}\sum_{j=1}^{s_1}\left \|  \widehat{\mathbf{R}}_j-\mathbf{R}\mathbf{H}_{\mathbf{R}_j}   \right \|_2^2\right)^{1/2} \frac{1}{\sqrt{p}}\left \|\widehat{\mathbf{R}}-\mathbf{R}\mathbf{H}_{\mathbf{R}} \right \|_2.
\end{align*}
By Assumption 3, Propositions~\ref{prop1}-\ref{prop2} and Theorem 1, we have
    $$\left \|  \mathrm{I}\hspace{-1.2pt}\mathrm{I}\hspace{-1.2pt}\mathrm{I}_1 \right \|_2=\mathcal{O}_p\left(\frac{1}{\underline{m}T}+\frac{1}{p^2}\right).$$
Similarly, we bound $\mathrm{I}\hspace{-1.2pt}\mathrm{I}\hspace{-1.2pt}\mathrm{I}_2$ as
$$\left \|  \mathrm{I}\hspace{-1.2pt}\mathrm{I}\hspace{-1.2pt}\mathrm{I}_2 \right \|_2 
    \le \left \|  \mathbf{R}^i\right\|_2\left(\frac{1}{s_1}\sum_{j=1}^{s_1}\left \| \mathbf{H}_{\mathbf{R}_j}  \right \|_2^2\right)^{1/2}\left(\frac{1}{ps_1}\sum_{j=1}^{s_1}\left \|  \widehat{\mathbf{R}}_j-\mathbf{R}\mathbf{H}_{\mathbf{R}_j}   \right \|_2^2\right)^{1/2} \frac{\left \|\mathbf{R}\right \|_2}{\sqrt{p}}\left \| \mathbf{H}_{\mathbf{R}} \right \|_2.$$
By Assumption 3, Propositions~\ref{prop1}-\ref{prop2} and Lemma~\ref{lem2}, we have
    $$\left \|  \mathrm{I}\hspace{-1.2pt}\mathrm{I}\hspace{-1.2pt}\mathrm{I}_2 \right \|_2=\mathcal{O}_p\left(\left(\frac{1}{\underline{m}T}+\frac{1}{p^2}\right)^{1/2}\right).$$
\end{proof}

\begin{proposition}\label{prop7}
Under Assumptions 1-6 with known $k$ and $r$, as $\underline{m},\underline{n},T\to \infty$, we have
 $$ \frac{\widehat{\mathbf{R}}^\top\mathbf{R}}{p}\overset{\mathcal{P}}{\longrightarrow}\mathbf{Q}_{\mathbf{R}},\quad\text{and}\quad\frac{\widehat{\mathbf{C}}^\top\mathbf{C}}{q}\overset{\mathcal{P}}{\longrightarrow}\mathbf{Q}_{\mathbf{C}}.$$
 Here the matrices $\mathbf{Q}_{\mathbf{R}}\in \mathbb{R}^{k\times k}$ and $\mathbf{Q}_{\mathbf{C}}\in \mathbb{R}^{r\times r}$ are given by $\mathbf{Q}_{\mathbf{R}}=\mathbf{V}_{\mathbf{R}}^{1/2}\Psi_{\mathbf{R}}^\top\boldsymbol{\Sigma}_{\mathbf{H_R}}^{-1/2}$ and $\mathbf{Q}_{\mathbf{C}}=\mathbf{V}_{\mathbf{C}}^{1/2}\Psi_{\mathbf{C}}^\top\boldsymbol{\Sigma}_{\mathbf{H_C}}^{-1/2}$, where $\mathbf{V_R}$ and $\mathbf{V_C}$ are diagonal matrices of eigenvalues (ordered decreasingly) of $\boldsymbol{\Sigma}_\mathbf{H_R}^{1/2}\boldsymbol{\Omega}_\mathbf{R}\boldsymbol{\Sigma}_\mathbf{H_R}^{1/2}$ and $\boldsymbol{\Sigma}_\mathbf{H_C}^{1/2}\boldsymbol{\Omega}_\mathbf{C}\boldsymbol{\Sigma}_\mathbf{H_C}^{1/2}$, respectively, with corresponding eigenvector matrices $\Psi_{\mathbf{R}}$ and $\Psi_{\mathbf{C}}$ satisfying $\Psi_{\mathbf{R}}^\top\Psi_{\mathbf{R}}=\mathbf{I}$ and $\Psi_{\mathbf{C}}^\top\Psi_{\mathbf{C}}=\mathbf{I}$, along with $\boldsymbol{\Sigma}_\mathbf{H_R}=\frac{1}{s_1}\sum_{i=1}^{s_1}\mathbf{Q}_{\mathbf{R}_i}^{-1}(\mathbf{Q}_{\mathbf{R}_i}^{-1})^\top$ and $ \boldsymbol{\Sigma}_\mathbf{H_C}=\frac{1}{s_2}\sum_{j=1}^{s_2}\mathbf{Q}_{\mathbf{C}_j}^{-1}(\mathbf{Q}_{\mathbf{C}_j}^{-1})^\top$.
\end{proposition}
\begin{proof}
    We only prove the conclusion for $\widehat{\mathbf{R}}^\top\mathbf{R}$, since proof of results for $\widehat{\mathbf{C}}^\top\mathbf{C}$ is similar. From Equation (\ref{equation1_prop6}), we have
    $$\widehat{\mathbf{R}}\widehat{\mathbf{V}}_\mathbf{R}=\frac{1}{ps_1}\sum_{i=1}^{s_1}\widehat{\mathbf{R}}_i\widehat{\mathbf{R}}_i^\top\widehat{\mathbf{R}}.$$
    By left-multiplying $\frac{1}{p}(\frac{1}{s_1}\sum_{i=1}^{s_1}\mathbf{H}_{\mathbf{R}_i}\mathbf{H}_{\mathbf{R}_i}^\top)^{1/2}\mathbf{R}^\top$ on the both sides and writing $\wh{\bR}_i$ as $\wh{\bR}_i=\wh{\bR}_i-\bR\bH_{\bR_i}+\bR\bH_{\bR_i}$, we obtain
    \begin{equation}
        \bB\widehat{\mathbf{V}}_\mathbf{R}=(\bA+\bd\bB^{-1})\bB,\label{equation1_prop7}
    \end{equation}
    where $k\times k$ matrices $\bA$, $\bB$, and $\bd$ are defined as
    $$\bA=\left(\frac{1}{s_1}\sum_{i=1}^{s_1}\mathbf{H}_{\mathbf{R}_i}\mathbf{H}_{\mathbf{R}_i}^\top\right)^{1/2}\frac{\mathbf{R}^\top\mathbf{R}}{p}\left(\frac{1}{s_1}\sum_{i=1}^{s_1}\mathbf{H}_{\mathbf{R}_i}\mathbf{H}_{\mathbf{R}_i}^\top\right)^{1/2},$$
    $$\bB=\frac{1}{p}\left(\frac{1}{s_1}\sum_{i=1}^{s_1}\mathbf{H}_{\mathbf{R}_i}\mathbf{H}_{\mathbf{R}_i}^\top\right)^{1/2}\mathbf{R}^\top\widehat{\bR},$$
    and    
    \begin{align*}
        \bd&=\frac{1}{p}\left(\frac{1}{s_1}\sum_{i=1}^{s_1}\mathbf{H}_{\mathbf{R}_i}\mathbf{H}_{\mathbf{R}_i}^\top\right)^{1/2}\mathbf{R}^\top\left[\frac{1}{ps_1}\sum_{i=1}^{s_1}(\widehat{\mathbf{R}}_i-\mathbf{R}\mathbf{H}_{\mathbf{R}_i})(\widehat{\mathbf{R}}_i-\mathbf{R}\mathbf{H}_{\mathbf{R}_i})^\top\widehat{\mathbf{R}}\right.\\
        &\left.\ \ \ +\frac{1}{ps_1}\sum_{i=1}^{s_1}(\widehat{\mathbf{R}}_i-\mathbf{R}\mathbf{H}_{\mathbf{R}_i})\mathbf{H}_{\mathbf{R}_i}^\top\mathbf{R}^\top\widehat{\mathbf{R}}+\frac{1}{ps_1}\sum_{i=1}^{s_1}\mathbf{R}\mathbf{H}_{\mathbf{R}_i}(\widehat{\mathbf{R}}_i-\mathbf{R}\mathbf{H}_{\mathbf{R}_i})^\top\widehat{\mathbf{R}} \right]\\
        &=\frac{1}{p}\left(\frac{1}{s_1}\sum_{i=1}^{s_1}\mathbf{H}_{\mathbf{R}_i}\mathbf{H}_{\mathbf{R}_i}^\top\right)^{1/2}\mathbf{R}^\top\left(\mathrm{I}_1+\mathrm{I}_2+\mathrm{I}_3\right).
    \end{align*}
    By Proposition~\ref{prop2} and Assumption 3, we have $\|\frac{1}{\sqrt{p}}(\frac{1}{s_1}\sum_{i=1}^{s_1}\mathbf{H}_{\mathbf{R}_i}\mathbf{H}_{\mathbf{R}_i}^\top)^{1/2}\mathbf{R}^\top\|_2=\mathcal{O}_p(1)$. Following a similar argument to that in Theorem 1, term $\mathrm{I}_1$ is dominated by $\mathrm{I}_2$ and $\mathrm{I}_3$, with $\frac{1}{\sqrt{p}}\|\mathrm{I}_2\|_2=\mathbf{o}_p(1)$ and $\frac{1}{\sqrt{p}}\|\mathrm{I}_3\|_2=\mathbf{o}_p(1)$. Together, we have $$\|\bd\|_2=\mathbf{o}_p(1).$$
    By Proposition~\ref{prop6}, as $\underline{m},p,T\to \infty$, we have
    $$\bB^\top\bB=\frac{1}{p}\widehat{\mathbf{R}}^\top\mathbf{R}\left( \frac{1}{s_1}\sum_{i=1}^{s_1}\mathbf{H}_{\mathbf{R}_i}\mathbf{H}_{\mathbf{R}_i}^\top\right)\frac{1}{p}\mathbf{R}^\top\widehat{\mathbf{R}}\overset{\mathcal{P}}{\longrightarrow}  \frac{1}{p}\widehat{\mathbf{R}}^\top \widehat{\bM}_\bR \widehat{\mathbf{R}}=\widehat{\bV}_\bR\overset{\mathcal{P}}{\longrightarrow}\mathbf{V_R}.$$
    In other words, the eigenvalues of $\bB^\top\bB$ are not close to zero nor infinity, and then $\|\bB^{-1}\|_2=\|\bB^{-1}(\bB^{-1})^\top\|_2^{1/2}=\mathcal{O}_p(1)$.
    
    Let $\bV_{\bB}$ be the diagonal matrix consisting of the diagonal elements of $\bB^\top\bB$. As $\underline{m},p,T\to \infty$, we have $\bV_{\bB}\overset{\mathcal{P}}{\longrightarrow}  \mathbf{V_R}$. Define the normalized $k\times k$ matrix $\bB^\ast=\bB\bV_{\bB}^{-\frac{1}{2}}$. Then, $\left\|\bB^\ast\right\|_2=\sqrt{\lambda_{\max}(\bV_\bB^{-1}\bB^\top\bB)}\overset{\mathcal{P}}{\longrightarrow}1$. By Equation (\ref{equation1_prop7}), we have
    $$\bB^\ast\widehat{\mathbf{V}}_\mathbf{R}=(\bA+\bd\bB^{-1})\bB^\ast.$$
    Since $\|\bd\|_2=\mathbf{o}_p(1)$ and $\|\bB^{-1}\|_2=\mathcal{O}_p(1)$, by Assumption 3 and Proposition~\ref{prop5}(c), as $\underline{m},p,T\to \infty$, we have $$\bA+\bd\bB^{-1}=\bA+\mathbf{o}_p(1)\overset{\mathcal{P}}{\longrightarrow}\boldsymbol{\Sigma}_\mathbf{H_R}^{1/2}\boldsymbol{\Omega}_\mathbf{R}\boldsymbol{\Sigma}_\mathbf{H_R}^{1/2}.$$ Since the eigenvalues of $\boldsymbol{\Sigma}_\mathbf{H_R}^{1/2}\boldsymbol{\Omega}_\mathbf{R}\boldsymbol{\Sigma}_\mathbf{H_R}^{1/2}$ are distinct under Assumption 6(b), the eigenvector perturbation theory (\citeauthor{franklin2012}, \citeyear{franklin2012}, Section 6.12) implies the existence of a unique orthogonal eigenvector matrix $\Psi_{\mathbf{R}}$ of $\boldsymbol{\Sigma}_\mathbf{H_R}^{1/2}\boldsymbol{\Omega}_\mathbf{R}\boldsymbol{\Sigma}_\mathbf{H_R}^{1/2}$, such that $$\left\|\Psi_{\mathbf{R}}-\bB^\ast\right\|_2=\mathbf{o}_p(1).$$
    Finally, we complete our proof as
    $$\frac{\mathbf{R}^\top\widehat{\mathbf{R}}}{p}=\left(\frac{1}{s_1}\sum_{i=1}^{s_1}\mathbf{H}_{\mathbf{R}_i}\mathbf{H}_{\mathbf{R}_i}^\top\right)^{-1/2}\bB^\ast\bV_{\bB}^{1/2}\overset{\mathcal{P}}{\longrightarrow}\boldsymbol{\Sigma}_\mathbf{H_R}^{-1/2}\Psi_{\mathbf{R}}\mathbf{V}_\bR^{1/2}=\bQ_\bR^\top.$$
\end{proof}

\begin{proposition}\label{prop8}
    Under Assumptions 1-6 with known $k$ and $r$, as $\underline{m},\underline{n},T\to \infty$, we have
     $$ \mathbf{H}_{\mathbf{R}}\overset{\mathcal{P}}{\longrightarrow}\mathbf{Q}_{\mathbf{R}}^{-1},\quad\text{and}\quad\mathbf{H}_{\mathbf{C}}\overset{\mathcal{P}}{\longrightarrow}\mathbf{Q}_{\mathbf{C}}^{-1}.$$
     Here the matrices $\mathbf{Q}_{\mathbf{R}}\in \mathbb{R}^{k\times k}$ and $\mathbf{Q}_{\mathbf{C}}\in \mathbb{R}^{r\times r}$ are given by $\mathbf{Q}_{\mathbf{R}}=\mathbf{V}_{\mathbf{R}}^{1/2}\Psi_{\mathbf{R}}^\top\boldsymbol{\Sigma}_{\mathbf{H_R}}^{-1/2}$ and $\mathbf{Q}_{\mathbf{C}}=\mathbf{V}_{\mathbf{C}}^{1/2}\Psi_{\mathbf{C}}^\top\boldsymbol{\Sigma}_{\mathbf{H_C}}^{-1/2}$, where $\mathbf{V_R}$ and $\mathbf{V_C}$ are diagonal matrices of eigenvalues (ordered decreasingly) of $\boldsymbol{\Sigma}_\mathbf{H_R}^{1/2}\boldsymbol{\Omega}_\mathbf{R}\boldsymbol{\Sigma}_\mathbf{H_R}^{1/2}$ and $\boldsymbol{\Sigma}_\mathbf{H_C}^{1/2}\boldsymbol{\Omega}_\mathbf{C}\boldsymbol{\Sigma}_\mathbf{H_C}^{1/2}$, respectively, with corresponding eigenvector matrices $\Psi_{\mathbf{R}}$ and $\Psi_{\mathbf{C}}$ satisfying $\Psi_{\mathbf{R}}^\top\Psi_{\mathbf{R}}=\mathbf{I}$ and $\Psi_{\mathbf{C}}^\top\Psi_{\mathbf{C}}=\mathbf{I}$, along with $\boldsymbol{\Sigma}_\mathbf{H_R}=\frac{1}{s_1}\sum_{i=1}^{s_1}\mathbf{Q}_{\mathbf{R}_i}^{-1}(\mathbf{Q}_{\mathbf{R}_i}^{-1})^\top$ and $ \boldsymbol{\Sigma}_\mathbf{H_C}=\frac{1}{s_2}\sum_{j=1}^{s_2}\mathbf{Q}_{\mathbf{C}_j}^{-1}(\mathbf{Q}_{\mathbf{C}_j}^{-1})^\top$.
\end{proposition}
\begin{proof}
    We only prove the conclusion for $\mathbf{H}_{\mathbf{R}}$, since proof of results for $\mathbf{H}_{\mathbf{C}}$ is similar.\\
    By the definition of $\bH_\bR$, we have
    $$\bH_\bR=\frac{1}{ps_1}\sum_{i=1}^{s_1} \bH_{\bR_i}\bH_{\bR_i}^\top\bR^\top\hat{\bR}\widehat{\bV}_{\bR}^{-1}\overset{\mathcal{P}}{\longrightarrow}\boldsymbol{\Sigma}_\mathbf{H_R}\mathbf{Q}_{\mathbf{R}}^\top\mathbf{V}_\bR^{-1}=\boldsymbol{\Sigma}_\mathbf{H_R}^{1/2}\Psi_{\mathbf{R}}\mathbf{V}_\bR^{-1/2}=\mathbf{Q}_{\mathbf{R}}^{-1},$$
    where the second step follows from Propositions \ref{prop5}(c), \ref{prop6}, and \ref{prop7}, and the last step follows from $\Psi_{\mathbf{R}}^\top\Psi_{\mathbf{R}}=\mathbf{I}$.
\end{proof}
The following lemma extends the asymptotic normality results in \cite{chen2021} to establish the joint asymptotic distribution of all rows of $\{\widehat{\mathbf{R}}_i-\mathbf{R}\mathbf{H}_{\mathbf{R}_i}\}_{i=1}^{s_1}$, which form the technical basis for the proof of Theorem 4.
\begin{lemma}\label{lem4}
    Let $\{\mathbf{\Phi}_{\bR_{i_1,i_2,j_1,j_2}}^{a,b}\}_{a,b\in[2]}$ as in Theorem 4 for $i_1,i_2\in[s_1]$ and $j_1,j_2\in[p]$, and $\{\mathbf{\Phi}_{\bC_{i_1',i_2',j_1',j_2'}}^{a,b}\}_{a,b\in[2]}$ similarly for $i_1',i_2'\in[s_2]$ and $j_1',j_2'\in[q]$. Under Assumptions 1-6 with known $k$ and $r$, we have
    \begin{enumerate}
    \item[(a)] When ${\sqrt{\overline{m}T}}/{p}=\mathbf{o}(1)$, for bounded matrices $\{\bA^{i,j}\}_{i,j=1}^{s_1,p}\subset\mathbb{R}^{k\times k}$, any linear combination of the row vectors of $\{\widehat{\mathbf{R}}_i-\mathbf{R}\mathbf{H}_{\mathbf{R}_i}\}_{i=1}^{s_1}$ satisfies
    \begin{align*}
        &\frac{1}{ps_1}\sum_{i=1}^{s_1}\sum_{j=1}^p\sqrt{m_iT}\bA^{i,j}\left(\widehat{\mathbf{R}}^j_i-\mathbf{R}^j\mathbf{H}_{\mathbf{R}_i}\right)^\top\\
        &=\frac{1}{ps_1}\sum_{i=1}^{s_1}\sum_{j=1}^p\bA^{i,j}\widehat{\mathbf{V}}_{\mathbf{R}_i}^{-1}\frac{\widehat{\bR}_i^\top\bR}{p}\frac{1}{\sqrt{m_iT}}\sum_{t=1}^T\widetilde{\bF}_t\bC_i^\top(\widetilde{\bE}_{it}^j)^\top+\mathbf{o}_p(1)\\
        &\overset{\mathcal{D}}{\longrightarrow}\mathcal{N}\left(\mathbf{0},\mathbf{\Sigma}_{\bR,\bA}\right),\quad\text{as }\underline{m},p,T\to\infty,
    \end{align*}
    where
    \begin{align*}
        \mathbf{\Sigma}_{\bR,\bA}&=\frac{1}{p^2s_1^2}\sum_{i_1=1}^{s_1}\sum_{i_2=1}^{s_1}\sum_{j_1=1}^p\sum_{j_2=1}^p\bA^{i_1,j_1}\bV_{\bR_{i_1}}^{-1}\bQ_{\bR_{i_1}}\left(\mathbf{\Phi}_{\bR_{i_1,i_2,j_1,j_2}}^{1,1}+\alpha\mathbf{\Phi}_{\bR_{i_1,i_2,j_1,j_2}}^{1,2}\overline{\bF}^\top\right.\\
        & \quad \quad \left.+\alpha\overline{\bF}\mathbf{\Phi}_{\bR_{i_1,i_2,j_1,j_2}}^{2,1}+\alpha^2\overline{\bF}\mathbf{\Phi}_{\bR_{i_1,i_2,j_1,j_2}}^{2,2}\overline{\bF}^\top\right)\bQ_{\bR_{i_2}}^\top\bV_{\bR_{i_2}}^{-1}(\bA^{i_2,j_2})^\top.
    \end{align*}
    The matrix $\mathbf{Q}_{\mathbf{R}_i}=\mathbf{V}_{\mathbf{R}_i}^{1/2}\Psi_{\mathbf{R}_i}^\top\widetilde{\boldsymbol{\Sigma}}_{\mathbf{FC}_i}^{-1/2}$, where $\mathbf{V}_{\mathbf{R}_i}$ is the diagonal matrix of eigenvalues (ordered decreasingly) of $\widetilde{\boldsymbol{\Sigma}}_{\mathbf{FC}_i}^{1/2}\boldsymbol{\Omega}_{\mathbf{R}}\widetilde{\boldsymbol{\Sigma}}_{\mathbf{FC}_i}^{1/2}$, with the corresponding eigenvector matrix $\Psi_{\mathbf{R}_i}$ satisfying $\Psi_{\mathbf{R}_i}^\top\Psi_{\mathbf{R}_i}=\mathbf{I}$, along with $\widetilde{\boldsymbol{\Sigma}}_{\mathbf{FC}_i}=\frac{1}{m_i}\mathbb{E}[\widetilde{\mathbf{F}}_t\mathbf{C}_i^\top\mathbf{C}_i\widetilde{\mathbf{F}}_t^\top]$.
    \item[(b)] When ${\sqrt{\overline{n}T}}/{q}=\mathbf{o}(1)$, for bounded matrices $\{\bB^{i',j'}\}_{i',j'=1}^{s_2,q}\subset\mathbb{R}^{r\times r}$, any linear combination of the row vectors of $\{\widehat{\mathbf{C}}_{i'}-\mathbf{C}\mathbf{H}_{\mathbf{C}_{i'}}\}_{i'=1}^{s_2}$ satisfies
    \begin{align*}
        &\frac{1}{qs_2}\sum_{i'=1}^{s_2}\sum_{j'=1}^q\sqrt{n_{i'}T}\bB^{i',j'}\left(\widehat{\mathbf{C}}^{j'}_{i'}-\mathbf{C}^{j'}\mathbf{H}_{\mathbf{C}_{i'}}\right)^\top\\
        &=\frac{1}{qs_2}\sum_{i'=1}^{s_2}\sum_{j'=1}^q\bB^{i',j'}\widehat{\mathbf{V}}_{\mathbf{C}_{i'}}^{-1}\frac{\widehat{\bC}_{i'}^\top\bC}{q}\frac{1}{\sqrt{n_{i'}T}}\sum_{t=1}^T\widetilde{\bF}_t^\top\bR_{i'}^\top[(\widetilde{\boldsymbol{\epsilon}}_{i't}^\top)^{j'}]^\top+\mathbf{o}_p(1)\\
        &\overset{\mathcal{D}}{\longrightarrow}\mathcal{N}\left(\mathbf{0},\mathbf{\Sigma}_{\bC,\bB}\right),\quad\text{as }\underline{n},q,T\to\infty,
    \end{align*}
    where
    \begin{align*}
        \mathbf{\Sigma}_{\bC,\bB}&=\frac{1}{q^2s_2^2}\sum_{i_1'=1}^{s_2}\sum_{i_2'=1}^{s_2}\sum_{j_1'=1}^q\sum_{j_2'=1}^q\bB^{i_1',j_1'}\bV_{\bC_{i_1'}}^{-1}\bQ_{\bC_{i_1'}}\left(\mathbf{\Phi}_{\bC_{i_1',i_2',j_1',j_2'}}^{1,1}+\alpha\mathbf{\Phi}_{\bC_{i_1',i_2',j_1',j_2'}}^{1,2}\overline{\bF}\right.\\
        &\quad\quad\left.+\alpha\overline{\bF}^\top\mathbf{\Phi}_{\bC_{i_1',i_2',j_1',j_2'}}^{2,1}+\alpha^2\overline{\bF}^\top\mathbf{\Phi}_{\bC_{i_1',i_2',j_1',j_2'}}^{2,2}\overline{\bF}\right)\bQ_{\bC_{i_2'}}^\top\bV_{\bC_{i_2'}}^{-1}(\bB^{i_2',j_2'})^\top.
    \end{align*}
    The matrix $\mathbf{Q}_{\mathbf{C}_{i'}}=\mathbf{V}_{\mathbf{C}_{i'}}^{1/2}\Psi_{\mathbf{C}_{i'}}^\top\widetilde{\boldsymbol{\Sigma}}_{\mathbf{FR}_{i'}}^{-1/2}$, where $\mathbf{V}_{\mathbf{C}_{i'}}$ is the diagonal matrix of eigenvalues (ordered decreasingly) of $\widetilde{\boldsymbol{\Sigma}}_{\mathbf{FR}_{i'}}^{1/2}\boldsymbol{\Omega}_{\mathbf{C}}\widetilde{\boldsymbol{\Sigma}}_{\mathbf{FR}_{i'}}^{1/2}$, with the corresponding eigenvector matrix $\Psi_{\mathbf{C}_{i'}}$ satisfying $\Psi_{\mathbf{C}_{i'}}^\top\Psi_{\mathbf{C}_{i'}}=\mathbf{I}$, along with $\widetilde{\boldsymbol{\Sigma}}_{\mathbf{FR}_{i'}}=\frac{1}{n_{i'}}\mathbb{E}[\widetilde{\mathbf{F}}_t^\top\mathbf{R}_{i'}^\top\mathbf{R}_{i'}\widetilde{\mathbf{F}}_t]$.
    \end{enumerate}
\end{lemma}
\begin{proof}
    We only prove (a), since the proof of (b) is similar. For each $i\in[s_1]$ and $j\in[p]$, when ${\sqrt{m_iT}}/{p}=\mathbf{o}(1)$, $\widehat{\mathbf{R}}^j_i-\mathbf{R}^j\mathbf{H}_{\mathbf{R}_i}$ can be replaced by its dominant term introduced in Proposition~\ref{prop5}(d). Specifically, 
    \begin{align*}
        &\frac{1}{ps_1}\sum_{i=1}^{s_1}\sum_{j=1}^p\sqrt{m_iT}\bA^{i,j}\left(\widehat{\mathbf{R}}^j_i-\mathbf{R}^j\mathbf{H}_{\mathbf{R}_i}\right)^\top\\
        &=\frac{1}{ps_1}\sum_{i=1}^{s_1}\sum_{j=1}^p\bA^{i,j}\widehat{\mathbf{V}}_{\mathbf{R}_i}^{-1}\frac{\widehat{\bR}_i^\top\bR}{p}\frac{1}{\sqrt{m_iT}}\sum_{t=1}^T\widetilde{\bF}_t\bC_i^\top(\widetilde{\bE}_{it}^j)^\top+\mathbf{o}_p(1)\\
        &=\frac{1}{\sqrt{T}}\sum_{t=1}^T\frac{1}{ps_1}\sum_{i=1}^{s_1}\sum_{j=1}^p\bA^{i,j}\widehat{\mathbf{V}}_{\mathbf{R}_i}^{-1}\frac{\widehat{\bR}_i^\top\bR}{p}\widetilde{\bF}_t\frac{\bC_i^\top(\widetilde{\bE}_{it}^j)^\top}{\sqrt{m_i}}+\mathbf{o}_p(1).
    \end{align*}
    For convenience, let $\hat{\bU}^{i,j}=\bA^{i,j}\widehat{\mathbf{V}}_{\mathbf{R}_i}^{-1}{\widehat{\bR}_i^\top\bR}/p\in\mathbb{R}^{k\times k}$. As defined in Section 2.2, $\widetilde{\bF}_t=\bF_t+\widetilde{\alpha}\overline{\bF}$, $\widetilde{\bE}_t=\bE_t+\widetilde{\alpha}\overline{\bE}$, and $\widetilde{\alpha}^2+2\widetilde{\alpha}=\alpha$. Then, we obtain the following decomposition,
    \begin{align}
        &\frac{1}{\sqrt{T}}\sum_{t=1}^T\frac{1}{ps_1}\sum_{i=1}^{s_1}\sum_{j=1}^p\hat{\bU}^{i,j}\widetilde{\bF}_t\frac{1}{\sqrt{m_i}}\bC_i^\top(\widetilde{\bE}_{it}^j)^\top \nonumber \\
        &=\frac{1}{\sqrt{T}}\sum_{t=1}^T\frac{1}{ps_1}\sum_{i=1}^{s_1}\sum_{j=1}^p\hat{\bU}^{i,j}\left(\bF_t+\widetilde{\alpha}\overline{\bF}\right)\frac{1}{\sqrt{m_i}}\bC_i^\top\left[(\bE_{it}^j)^\top+\widetilde{\alpha}(\overline{\bE}_{i}^j)^\top\right] \nonumber \\
        &=\frac{1}{\sqrt{T}}\sum_{t=1}^T\frac{1}{ps_1}\sum_{i=1}^{s_1}\sum_{j=1}^p\hat{\bU}^{i,j}\left[\bF_t\frac{1}{\sqrt{m_i}}\bC_i^\top(\bE_{it}^j)^\top+(\widetilde{\alpha}^2+2\widetilde{\alpha})\overline{\bF}\frac{1}{\sqrt{m_i}}\bC_i^\top(\overline{\bE}_{i}^j)^\top\right] \nonumber \\
        &=\frac{1}{\sqrt{T}}\sum_{t=1}^T\frac{1}{ps_1}\sum_{i=1}^{s_1}\sum_{j=1}^p\hat{\bU}^{i,j}(\bF_t+\alpha\overline{\bF})\frac{1}{\sqrt{m_i}}\bC_i^\top(\bE_{it}^j)^\top\overset{\triangle }{=}\frac{1}{\sqrt{T}}\sum_{t=1}^T\bx_t,\nonumber 
    \end{align}
    where we define the $k\times 1$ random vector $\bx_t=\frac{1}{ps_1}\sum_{i=1}^{s_1}\sum_{j=1}^p\hat{\bU}^{i,j}(\bF_t+\alpha\overline{\bF})\frac{1}{\sqrt{m_i}}\bC_i^\top(\bE_{it}^j)^\top$. According to Proposition~\ref{prop5}(a)-(b), both $\|\widehat{\mathbf{V}}_{\mathbf{R}_i}^{-1}\|_2$ and $\|{p}^{-1}{\widehat{\bR}_i^\top\bR}\|_2$ are bounded, implying that $\|\hat{\bU}^{i,j}\|_2$ is bounded, for any $i \in [s_1]$ and $j\in[p]$. By Assumptions 2(a), $\mathbb{E}\|{\bF}_t\|^4_2$ is bounded, and hence $\|\overline{\bF}\|_2$ is bounded by the ergodic theorem for $\alpha$-mixing processes in \cite{fan2003}. By Assumptions 3(a), $\frac{1}{\sqrt{m_i}}\|\bC_i\|_2$ is also bounded. Consequently, $\frac{1}{ps_1}\sum_{i=1}^{s_1}\sum_{j=1}^p\hat{\bU}^{i,j}\alpha\overline{\bF}\frac{1}{\sqrt{m_i}}\bC_i^\top(\bE_{it}^j)^\top$ inherits the $\alpha$-mixing property from vec($\bE_t$). Moreover, by Assumption 1, the process \{vec($\bF_t$),vec($\bE_t$)\} is $\alpha$-mixing, implying that $\frac{1}{ps_1}\sum_{i=1}^{s_1}\sum_{j=1}^p\hat{\bU}^{i,j}\bF_t\frac{1}{\sqrt{m_i}}\bC_i^\top(\bE_{it}^j)^\top$ is also $\alpha$-mixing. Combining these two components, we conclude that $\bx_t$ is $\alpha$-mixing. 
    
    By Assumption 2, $\mathbb{E}[\bE_t]=\textbf{0}$ and $\bE_t$ is uncorrelated with $\bF_t$, which implies that $\mathbb{E}[\bx_t]=\textbf{0}$. In addition, for $m>2$ given in Assumption 5, by the Holder's inequality and the Cauchy–Schwarz inequality, for some $a,b$ such that $1<a, b<\infty$ and $1/a+1/b=1$, we obtain the following bound:
    \begin{align*}
        \mathbb{E}[\|\bx_t\|^m_2]&\le\max_{i,j}\mathbb{E}\left\|\hat{\bU}^{i,j}(\bF_t+\alpha\overline{\bF})\frac{1}{\sqrt{m_i}}\bC_i^\top(\bE_{it}^j)^\top\right\|_2^m\\
        &\le\left(\max_{i,j}\left\|\hat{\bU}^{i,j}\right\|_2^m\right)\left(\max_{i,j}\mathbb{E}\left\|(\bF_t+\alpha\overline{\bF})\frac{1}{\sqrt{m_i}}\bC_i^\top(\bE_{it}^j)^\top\right\|_2^m\right)\\
        &\le k^m\left(\max_{i,j}\left\|\hat{\bU}^{i,j}\right\|_2^m\right)\left(\max_{i,j,l}\mathbb{E}\left[(\bF_t^l+\alpha\overline{\bF}^l)\frac{1}{\sqrt{m_i}}\bC_i^\top(\bE_{it}^j)^\top\right]^m\right)\\
        &\le k^m\left(\max_{i,j}\left\|\hat{\bU}^{i,j}\right\|_2^m\right)\left(\max_{i,j,l}\mathbb{E}\left[\left\|\bF_t^l+\alpha\overline{\bF}^l\right\|_2^m\left\|\frac{1}{\sqrt{m_i}}\bC_i^\top(\bE_{it}^j)^\top\right\|_2^m\right]\right)\\
        &\le (1+|\alpha|)^mk^m\left(\max_{i,j}\left\|\hat{\bU}^{i,j}\right\|_2^m\right)\left[\max_{l}\left(\mathbb{E}\left[\|\bF_t^l\|_2^{ma}\right]\right)^{1/a}\right]\\
        &\ \ \ \left[\max_{i,j}\left(\mathbb{E}\left[\left\|\frac{1}{\sqrt{m_i}}\bC_i^\top(\bE_{it}^j)^\top\right\|_2^{mb}\right]\right)^{1/b}\right]\\
        &= (1+|\alpha|)^mk^m\left(\max_{i,j}\left\|\hat{\bU}^{i,j}\right\|_2^m\right)\left[\max_{l}\left(\mathbb{E}\left[\|\bF_t^l\|_2^{ma}\right]\right)^{1/a}\right]\\
        &\ \ \ \left[\max_{i,j}\left(\mathbb{E}\left[\left\|\frac{1}{\sqrt{m_i}}\sum_{\tau=1}^{m_i}(\bC_i^\tau)^\top e_{it,j\tau}\right\|_2^{mb}\right]\right)^{1/b}\right].
    \end{align*}
    Here $\{(\bC_i^\tau)^\top\}_{\tau=1}^{m_i}$ are columns of  $\bC_i^\top$  and $\{e_{it,j\tau}\}_{\tau=1}^{m_i}$ are elements of $\bE_{it}^j$, i.e,
    $$\bC_i^\top=\begin{pmatrix} (\bC_i^1)^\top & (\bC_i^2)^\top &...&(\bC_i^{m_i})^\top \end{pmatrix},$$
    and 
    $$\bE_{it}^j=\begin{pmatrix} e_{it,j1} & e_{it,j2} &...&e_{it,jm_i} \end{pmatrix}.$$
    By Assumption 5(a), for any $l\in[k]$, $j\in[p]$, and $i\in[s_1]$, there exists $C>0$, s.t.,
    $$\mathbb{E}\left[\|\bF_t^l\|_2^{ma}\right]\le C,\quad\text{and}\quad\mathbb{E}\left[\left\|\frac{1}{\sqrt{m_i}}\sum_{\tau=1}^{m_i}(\bC_i^\tau)^\top e_{it,j\tau}\right\|_2^{mb}\right]=\mathcal{O}(1).$$
    Thus, we have $\mathbb{E}[\|\bx_t\|^m_2]<\infty$. Since $\bx_t$ is $\alpha$-mixing with $\mathbb{E}[\bx_t]=\textbf{0}$, and $\mathbb{E}[\|\bx_t\|^m_2]<\infty$ for some $m>2$, the central limit theorem for $\alpha$-mixing processes (\citeauthor{fan2003}, \citeyear{fan2003}, Theorem 2.21(i)) implies that, as $T\to \infty$,
    $$\frac{1}{\sqrt{T}}\sum_{t=1}^T\bx_t\overset{\mathcal{D}}{\longrightarrow}\mathcal{N}\left(\mathbf{0},\mathbf{\Sigma}_{\bR,\bA}^{p,q}\right),$$
    where
    \begin{align*}
        \mathbf{\Sigma}_{\bR,\bA}^{p,q}&=\underset{T \rightarrow \infty}{\operatorname{plim}}\frac{1}{T}\sum_{t=1}^T\sum_{s=1}^T\text{cov}(\bx_{t},\bx_{s})\\
        &=\underset{T \rightarrow \infty}{\operatorname{plim}}\frac{1}{T}\sum_{t,s=1}^T\sum_{i_1,i_2=1}^{s_1}\sum_{j_1,j_2=1}^p\frac{1}{p^2s_1^2}\begin{pmatrix}\widehat{\bU}^{i_1,j_1} & \alpha\widehat{\bU}^{i_1,j_1}\overline{\bF}\end{pmatrix}\begin{pmatrix} \bF_t\frac{1}{\sqrt{m_{i_1}}}\bC_{i_1}^\top(\bE_{{i_1}t}^{j_1})^\top\\ \frac{1}{\sqrt{m_{i_1}}}\bC_{i_1}^\top(\bE_{{i_1}t}^{j_1})^\top\end{pmatrix}\\
        &\ \ \ \ \ \ \ \ \ \ \begin{pmatrix} \bF_s\frac{1}{\sqrt{m_{i_2}}}\bC_{i_2}^\top(\bE_{{i_2}s}^{j_2})^\top\\ \frac{1}{\sqrt{m_{i_2}}}\bC_{i_2}^\top(\bE_{{i_2}s}^{j_2})^\top\end{pmatrix}^\top\begin{pmatrix}(\widehat{\bU}^{i_2,j_2})^\top \\ \alpha\overline{\bF}^\top(\widehat{\bU}^{i_2,j_2})^\top\end{pmatrix}\\
        &=\frac{1}{p^2s_1^2}\sum_{i_1,i_2=1}^{s_1}\sum_{j_1,j_2=1}^p\left[\underset{T \rightarrow \infty}{\operatorname{plim}}\begin{pmatrix}\widehat{\bU}^{i_1,j_1} & \alpha\widehat{\bU}^{i_1,j_1}\overline{\bF}\end{pmatrix}\right]\left[\underset{T \rightarrow \infty}{\operatorname{plim}}\frac{1}{T}\sum_{t,s=1}^T\begin{pmatrix} \bF_t\frac{1}{\sqrt{m_{i_1}}}\bC_{i_1}^\top(\bE_{{i_1}t}^{j_1})^\top\\ \frac{1}{\sqrt{m_{i_1}}}\bC_{i_1}^\top(\bE_{{i_1}t}^{j_1})^\top\end{pmatrix}\right.\\
        &\left.\ \ \ \ \ \ \ \ \ \ \begin{pmatrix} \bF_s\frac{1}{\sqrt{m_{i_2}}}\bC_{i_2}^\top(\bE_{{i_2}s}^{j_2})^\top\\ \frac{1}{\sqrt{m_{i_2}}}\bC_{i_2}^\top(\bE_{{i_2}s}^{j_2})^\top\end{pmatrix}^\top\right]\left[\underset{T \rightarrow \infty}{\operatorname{plim}}\begin{pmatrix}(\widehat{\bU}^{i_2,j_2})^\top \\ \alpha\overline{\bF}^\top(\widehat{\bU}^{i_2,j_2})^\top\end{pmatrix}\right]\\
        &=\frac{1}{p^2s_1^2}\sum_{i_1,i_2=1}^{s_1}\sum_{j_1,j_2=1}^p\left[\underset{T \rightarrow \infty}{\operatorname{plim}}\begin{pmatrix}\widehat{\bU}^{i_1,j_1} & \alpha\widehat{\bU}^{i_1,j_1}\overline{\bF}\end{pmatrix}\right]\begin{pmatrix} \mathbf{\Phi}_{\bR_{i_1,i_2,j_1,j_2}}^{1,1} & \mathbf{\Phi}_{\bR_{i_1,i_2,j_1,j_2}}^{1,2}\\ \mathbf{\Phi}_{\bR_{i_1,i_2,j_1,j_2}}^{2,1} &\mathbf{\Phi}_{\bR_{i_1,i_2,j_1,j_2}}^{2,2}\end{pmatrix}\\
        &\ \ \ \ \ \ \ \ \ \ \left[\underset{T \rightarrow \infty}{\operatorname{plim}}\begin{pmatrix}(\widehat{\bU}^{i_2,j_2})^\top \\ \alpha\overline{\bF}^\top(\widehat{\bU}^{i_2,j_2})^\top\end{pmatrix}\right].
    \end{align*}
    By Proposition~\ref{prop5}(a)-(b), as $m_i,p,T\to \infty$, we have
    $$\hat{\bU}^{i,j}\overset{\mathcal{P}}{\longrightarrow}\bA^{i,j}\bV_{\bR_{i}}^{-1}\bQ_{\bR_{i}}\overset{\triangle }{=} {\bU}^{i,j}.$$
    Then, as $\underline{m},p\to \infty$, we have
    \begin{align*}
        \mathbf{\Sigma}_{\bR,\bA}^{p,q}
        &\overset{\mathcal{P}}{\longrightarrow}\frac{1}{p^2s_1^2}\sum_{i_1,i_2=1}^{s_1}\sum_{j_1,j_2=1}^p\left[\underset{m_i,p,T \rightarrow \infty}{\operatorname{plim}}\begin{pmatrix}\widehat{\bU}^{i_1,j_1} & \alpha\widehat{\bU}^{i_1,j_1}\overline{\bF}\end{pmatrix}\right]\begin{pmatrix} \mathbf{\Phi}_{\bR_{i_1,i_2,j_1,j_2}}^{1,1} & \mathbf{\Phi}_{\bR_{i_1,i_2,j_1,j_2}}^{1,2}\\ \mathbf{\Phi}_{\bR_{i_1,i_2,j_1,j_2}}^{2,1} &\mathbf{\Phi}_{\bR_{i_1,i_2,j_1,j_2}}^{2,2}\end{pmatrix}\\
        &\ \ \ \ \ \ \ \ \ \ \left[\underset{m_i,p,T \rightarrow \infty}{\operatorname{plim}}\begin{pmatrix}(\widehat{\bU}^{i_2,j_2})^\top \\ \alpha\overline{\bF}^\top(\widehat{\bU}^{i_2,j_2})^\top\end{pmatrix}\right]\\
        &=\frac{1}{p^2s_1^2}\sum_{i_1,i_2=1}^{s_1}\sum_{j_1,j_2=1}^p\begin{pmatrix}{\bU}^{i_1,j_1} & \alpha{\bU}^{i_1,j_1}\overline{\bF}\end{pmatrix}\begin{pmatrix} \mathbf{\Phi}_{\bR_{i_1,i_2,j_1,j_2}}^{1,1} & \mathbf{\Phi}_{\bR_{i_1,i_2,j_1,j_2}}^{1,2}\\ \mathbf{\Phi}_{\bR_{i_1,i_2,j_1,j_2}}^{2,1} &\mathbf{\Phi}_{\bR_{i_1,i_2,j_1,j_2}}^{2,2}\end{pmatrix}\\
        &\ \ \ \ \ \ \ \ \ \ \begin{pmatrix}({\bU}^{i_2,j_2})^\top \\ \alpha\overline{\bF}^\top({\bU}^{i_2,j_2})^\top\end{pmatrix}\\
        &=\mathbf{\Sigma}_{\bR,\bA}.
    \end{align*}
    In other words, the final results are derived. As $\underline{m},p,T\to \infty$,
    \begin{align*}
        \frac{1}{ps_1}\sum_{i=1}^{s_1}\sum_{j=1}^p\sqrt{m_iT}\bA^{i,j}\left(\widehat{\mathbf{R}}^j_i-\mathbf{R}^j\mathbf{H}_{\mathbf{R}_i}\right)^\top\overset{\mathcal{D}}{\longrightarrow}\mathcal{N}\left(\mathbf{0},\mathbf{\Sigma}_{\bR,\bA}\right).
    \end{align*}
\end{proof}
\subsection*{A.2 Propositions and lemmas for the unit-root nonstationary case}
We next present some propositions and lemmas for the nonstationary case. These results will be used to prove Theorems 6-8. We begin by establishing asymptotic properties of the local estimators and the associated auxiliary matrices. Following Lemma B.3 in \cite{li2025factormodelsmatrixvaluedtime}, we have Proposition \ref{prop8.5}.
\begin{proposition}\label{prop8.5}
    Under Assumptions 3-4 and 7-9 with known $k$ and $r$, for any $i\in[s_1]\text{ and }i'\in[s_2]$, we have
    $$\left\|\frac{1}{m_iT^2}\sum_{t=1}^{T}\widetilde{\bF}_{t}\bC_i^\top\bC_i\widetilde{\bF}_{t}^\top-\widetilde{\boldsymbol{\Sigma}}_{\mathbf{FC}_i}^*\right\|_\mathsf{F}=\mathbf{o}_p\left(1\right),$$
    and
    $$\left\|\frac{1}{n_jT^2}\sum_{t=1}^{T}\widetilde{\bF}_{t}^\top\bR_j^\top\bR_j\widetilde{\bF}_{t}-\widetilde{\boldsymbol{\Sigma}}_{\mathbf{FR}_j}^*\right\|_\mathsf{F}=\mathbf{o}_p\left(1\right).$$
    The matrices $\widetilde{\boldsymbol{\Sigma}}_{\mathbf{FC}_i}^*=\int_0^1 \widetilde{\bW}(u)\boldsymbol{\Omega}_{\bC_i}\widetilde{\bW}(u)^\top\text{d}u$ and $\widetilde{\boldsymbol{\Sigma}}_{\mathbf{FR}_j}^*=\int_0^1 \widetilde{\bW}(u)^\top\boldsymbol{\Omega}_{\bR_j}\widetilde{\bW}(u)\text{d}u$, where $\widetilde{\bW}(u)=\bW(u)+\frac{\widetilde{\alpha}}{T}\sum_{t=1}^T\bW(t/T)$ and $\bW(\cdot)$ is an $k\times r$ matrix of Brownian motions with the covariance of $\text{vec}(\bW(\cdot))$ being $\overline{\bG}\boldsymbol{\Sigma}_\bu\overline{\bG}^\top$. \end{proposition}
\begin{proof}
    We only prove the conclusion for $\sum_{t=1}^{T}\widetilde{\bF}_{t}\bC_i^\top\bC_i\widetilde{\bF}_{t}^\top$, since the proof of the results for $\sum_{t=1}^{T}\widetilde{\bF}_{t}^\top\bR_j^\top\bR_j\widetilde{\bF}_{t}$ is similar. As defined in Section 2.2, $\widetilde{\bF}_t=\bF_t+\widetilde{\alpha}\overline{\bF}$. Then, we obtain the following decomposition,
    $$\sum_{t=1}^{T}\widetilde{\bF}_{t}\bC_i^\top\bC_i\widetilde{\bF}_{t}^\top=\sum_{t=1}^{T}\left({\bF}_t+\widetilde{\alpha}\overline{\bF}\right)\bC_i^\top\bC_i\left({\bF}_t+\widetilde{\alpha}\overline{\bF}\right)^\top.$$
    Following Lemma B.3 in \cite{li2025factormodelsmatrixvaluedtime}, 
    $$\max_{t\in[T]}\frac{1}{\sqrt{T}}\left\|\bF_t-\bW(t/T)\right\|_\mathsf{F}=\mathbf{o}_p\left(1\right).$$
    Then,
    \begin{align*}
        \left\|\overline{\bF}-\frac{1}{T}\sum_{t=1}^T\bW(t/T)\right\|_\mathsf{F}&=\left\|\frac{1}{T}\sum_{t=1}^T{\bF}_t-\bW(t/T)\right\|_\mathsf{F}\\
        &\le\frac{1}{T}\sum_{t=1}^T\left\|\bF_t-\bW(t/T)\right\|_\mathsf{F}\\
        &=\mathbf{o}_p\left(1\right).
    \end{align*}
    Combine two parts together, we have
    $$\max_{t\in[T]}\frac{1}{\sqrt{T}}\left\|\widetilde{\bF}_t-\widetilde{\bW}(t/T)\right\|_\mathsf{F}=\mathbf{o}_p\left(1\right).$$
    By Assumption 3 and the continuous mapping theorem, the final results are derived.
\end{proof}
\begin{proposition}\label{prop9}
    Under Assumptions 3-4 and 7-8 with known $k$ and $r$, for any $i\in[s_1]\text{ and }i'\in[s_2]$, we have
    \begin{equation*}
    \mathbb{E}\left[\left\|\sum_{t=1}^{T}\widetilde{\bF}_t\bC_i^\top\widetilde{\bE}_{it}^\top\right\|_\mathsf{F}^2\right]=\mathcal{O}_p\left(pm_iT^2\right),\quad\text{as }m_i,p,T\to\infty,
\end{equation*}
and \begin{equation*}
    \mathbb{E}\left[\left\|\sum_{t=1}^{T}\widetilde{\bF}_t^\top\bR_{i'}^\top\widetilde{\boldsymbol{\epsilon}}_{i't}\right\|_\mathsf{F}^2\right]=\mathcal{O}_p\left(qn_{i'} T^2\right),\quad\text{as }n_{i'},q,T\to\infty.
\end{equation*}
\end{proposition}
\begin{proof}
    We only prove the conclusion for $\sum_{t=1}^T\widetilde{\bF}_t\bC_i^\top\widetilde{\bE}_{it}^\top$, since the proof of the results for $\sum_{t=1}^T\widetilde{\bF}_t^\top\bR_{i'}^\top\widetilde{\boldsymbol{\epsilon}}_{i't}$ is similar. As defined in Section 2.2, $\widetilde{\bF}_t=\bF_t+\widetilde{\alpha}\overline{\bF}$, $\widetilde{\bE}_t=\bE_t+\widetilde{\alpha}\overline{\bE}$, and $\widetilde{\alpha}^2+2\widetilde{\alpha}=\alpha$. Then, we obtain the following decomposition,
    \begin{align*}
        \mathbb{E}\left\|\sum_{t=1}^{T}\widetilde{\bF}_t\bC_i^\top\widetilde{\bE}_{it}^\top\right\|_\mathsf{F}^2&=\mathbb{E}\left\|\sum_{t=1}^{T}\left({\bF}_t+\widetilde{\alpha}\overline{\bF}\right)\bC_i^\top\left({\bE}_{it}^\top+\widetilde{\alpha}\overline{{\bE}}_{i}^\top\right)\right\|_\mathsf{F}^2\\
        &=\mathbb{E}\left\|\sum_{t=1}^{T}{\bF}_t\bC_i^\top{\bE}_{it}^\top+(2\widetilde{\alpha}+\widetilde{\alpha}^2)\overline{\bF}\bC_i^\top{{\bE}}_{it}^\top\right\|_\mathsf{F}^2\\
        &=\mathbb{E}\left\|\sum_{t=1}^{T}({\bF}_t+\alpha\overline{\bF})\bC_i^\top{\bE}_{it}^\top\right\|_\mathsf{F}^2\\
        &=\sum_{j=1}^p\mathbb{E}\left\|\sum_{t=1}^{T}({\bF}_t+\alpha\overline{\bF})\bC_i^\top({\bE}_{it}^j)^\top\right\|_2^2\\
        &=\sum_{j=1}^p\mathbb{E}\left\|\sum_{t=1}^{T}({\bF}_t+\alpha\overline{\bF})\Big[\sum_{l=1}^{m_i}(\bC_i^l)^\top e_{it,jl}\Big]\right\|_2^2.
    \end{align*}
    Let $\bx_t^{ij}=\frac{1}{\sqrt{m_i}}\sum_{l=1}^{m_i}(\bC_i^l)^\top e_{it,jl}\in\mathbb{R}^{r\times 1}$. By Assumptions 7-8, following Lemma B.2. in \cite{li2025factormodelsmatrixvaluedtime}, we have
    $$\frac{1}{T^2}\mathbb{E}\left\|\sum_{t=1}^T{\bF}_t\bx_t^{ij}\right\|_2=\mathcal{O}(1).$$
    Additionally, by Assumption 8, we have
    \begin{align*}
        \frac{1}{T^2}\mathbb{E}\left\|\sum_{t=1}^T\overline{\bF}\bx_t^{ij}\right\|_2&=\frac{1}{T^3}\mathbb{E}\left\|\sum_{t=1}^T\sum_{s=1}^T{\bF}_s\bx_t^{ij}\right\|_2\\
        &\le\frac{1}{T^3}\mathbb{E}\left[\left\|\sum_{s=1}^T\bF_s\right\|_2\left\|\sum_{t=1}^T\bx_t^{ij}\right\|_2\right]\\
        &=\frac{1}{T^3}\mathbb{E}\left\|\sum_{s=1}^T\bF_s\right\|_2\mathbb{E}\left\|\sum_{t=1}^T\bx_t^{ij}\right\|_2\\
        &\le\frac{1}{T}\max_{s\in[T]}\mathbb{E}\left\|\bF_s\right\|_2\max_{t\in[T]}\mathbb{E}\left\|\bx_t^{ij}\right\|_2\\
        &=\mathcal{O}(1),
    \end{align*}
    where $\max_{t\in[T]}\frac{1}{\sqrt{T}}\left\|\bF_t-\bW(t/T)\right\|_\mathsf{F}=\mathbf{o}_p\left(1\right)$ by Lemma B.3 in \cite{li2025factormodelsmatrixvaluedtime}. Together, we have
    \begin{align*}
        \mathbb{E}\left\|\sum_{t=1}^{T}\widetilde{\bF}_t\bC_i^\top\widetilde{\bE}_{it}^\top\right\|_\mathsf{F}^2&=\sum_{j=1}^p\mathbb{E}\left\|\sum_{t=1}^{T}({\bF}_t+\alpha\overline{\bF})\Big[\sum_{l=1}^{m_i}(\bC_i^l)^\top e_{it,jl}\Big]\right\|_2^2\\
        &=m_i\sum_{j=1}^p\mathbb{E}\left\|\sum_{t=1}^{T}({\bF}_t+\alpha\overline{\bF})\bx_t^{ij}\right\|_2^2\\
        &\le m_i\sum_{j=1}^p\mathbb{E}\left\|\sum_{t=1}^{T}{\bF}_t\bx_t^{ij}\right\|_2^2+m_i\alpha\sum_{j=1}^p\mathbb{E}\left\|\sum_{t=1}^{T}\overline{\bF}\bx_t^{ij}\right\|_2^2\\
        &=\mathcal{O}_p\left(pm_iT^2\right).
    \end{align*}
\end{proof}

\begin{proposition}\label{prop10}
    Under Assumptions 3-4 and 7-9 with known $k$ and $r$, for any $i\in[s_1]$ and $j\in[s_2]$, as $m_i,n_j,T\to\infty$, we have
    $$\widehat{\mathbf{V}}_{\mathbf{R}_i}^*\overset{\mathcal{P}}{\longrightarrow}\bV_{\bR_i}^*\quad\text{and}\quad\widehat{\mathbf{V}}_{\mathbf{C}_j}^*\overset{\mathcal{P}}{\longrightarrow}\bV_{\bC_j}^*,$$
    with $\|\widehat{\mathbf{V}}_{\mathbf{R}_i}^*\|_2=\mathcal{O}_p(1)$, $\|(\widehat{\mathbf{V}}_{\mathbf{R}_i}^*)^{-1}\|_2=\mathcal{O}_p(1)$, $\|\widehat{\mathbf{V}}_{\mathbf{C}_j}^*\|_2=\mathcal{O}_p(1)$, and $\|(\widehat{\mathbf{V}}_{\mathbf{C}_j}^*)^{-1}\|_2=\mathcal{O}_p(1)$. Here $\mathbf{V}_{\mathbf{R}_i}^*$ and $\mathbf{V}_{\mathbf{C}_j}^*$ are diagonal matrices of eigenvalues (ordered decreasingly) of $(\widetilde{\boldsymbol{\Sigma}}_{\mathbf{FC}_i}^*)^{1/2}\boldsymbol{\Omega}_\mathbf{R}(\widetilde{\boldsymbol{\Sigma}}_{\mathbf{FC}_i}^*)^{1/2}$ and $(\widetilde{\boldsymbol{\Sigma}}_{\mathbf{FR}_j}^*)^{1/2}\boldsymbol{\Omega}_\mathbf{C}(\widetilde{\boldsymbol{\Sigma}}_{\mathbf{FR}_j}^*)^{1/2}$, respectively. The matrices $\widetilde{\boldsymbol{\Sigma}}_{\mathbf{FC}_i}^*=\int_0^1 \widetilde{\bW}(u)\boldsymbol{\Omega}_{\bC_i}\widetilde{\bW}(u)^\top\text{d}u$ and $\widetilde{\boldsymbol{\Sigma}}_{\mathbf{FR}_j}^*=\int_0^1 \widetilde{\bW}(u)^\top\boldsymbol{\Omega}_{\bR_j}\widetilde{\bW}(u)\text{d}u$, where $\widetilde{\bW}(u)=\bW(u)+\frac{\widetilde{\alpha}}{T}\sum_{t=1}^T\bW(t/T)$ and $\bW(\cdot)$ is an $k\times r$ matrix of Brownian motions with the covariance of $\text{vec}(\bW(\cdot))$ being $\overline{\bG}\boldsymbol{\Sigma}_\bu\overline{\bG}^\top$.
\end{proposition}
\begin{proof}
    We only prove the conclusion for $\hat{\bV}_{\bR_i}^*$, since the proof of results for $\hat{\bV}_{\bC_j}^*$ is similar. According to Equation (2.5) and $\frac{1}{p}\wh{\bR}_i^\top\wh{\bR}_i=\bI$, we have
    \begin{equation}
        \widehat{\mathbf{V}}_{\mathbf{R}_i}^*=\frac{1}{pT}\widehat{\mathbf{R}}_i^\top\widehat{\mathbf{M}}_{\mathbf{R}_i}\widehat{\mathbf{R}}_i=\frac{1}{p}\widehat{\mathbf{R}}_i^\top\left( \frac{1}{pm_iT^2}\sum_{t=1}^{T}\widetilde{\bY}_{it}\widetilde{\bY}_{it}^\top\right)\widehat{\mathbf{R}}_i. 
    \end{equation}
    According to model (2.4), $\widetilde{\bY}_{it}=\bR\widetilde{\bF}_{t}\bC_i^\top+\widetilde{\bE}_{it}^\top$, we obtain
    \begin{align*}
    \frac{1}{T}\widehat{\mathbf{M}}_{\mathbf{R}_i}&=\frac{1}{pm_iT^2}\sum_{t=1}^{T}\bR\widetilde{\bF}_{t}\bC_i^\top\bC_i\widetilde{\bF}_{t}^\top\bR^\top+\frac{1}{pm_iT^2}\sum_{t=1}^{T}\bR\widetilde{\bF}_{t}\bC_i^\top\widetilde{\bE}_{it}^\top\\
    &\ \ \ +\frac{1}{pm_iT^2}\sum_{t=1}^{T}\widetilde{\bE}_{it}\bC_i\widetilde{\bF}_{t}^\top\bR^\top+\frac{1}{pm_iT^2}\sum_{t=1}^{T}\widetilde{\bE}_{it}\widetilde{\bE}_{it}^\top.
    \end{align*}
    By Assumption 3 and Proposition~\ref{prop9}, $\frac{1}{\sqrt{p}}\|\bR\|_2=\mathcal{O}(1)$ and $\mathbb{E}[\|\sum_{t=1}^{T}\widetilde{\bF}_t\bC_i^\top\widetilde{\bE}_{it}^\top\|_\mathsf{F}^2]=\mathcal{O}_p\left(pm_iT^2\right)$. In addition, by Lemma 8 in \cite{chen2021}, $\mathbb{E}[\|\sum_{t=1}^{T}\widetilde{\bE}_{it}\widetilde{\bE}_{it}^\top\|_\mathsf{F}^2]=\mathcal{O}_p\left(p^2m_iT+pm_i^2T^2\right)$. Therefore, we obtain bounds of the last three terms of $\frac{1}{T}\widehat{\mathbf{M}}_{\mathbf{R}_i}$,
    $$\left \| \frac{1}{pm_iT^2}\sum_{t=1}^{T}\bR\widetilde{\bF}_{t}\bC_i^\top\widetilde{\bE}_{it}^\top\right \|_2 =\mathcal{O}_p\left(\frac{1}{\sqrt{m_i}T} \right),$$
    $$\left \| \frac{1}{pm_iT^2}\sum_{t=1}^{T}\widetilde{\bE}_{it}\bC_i\widetilde{\bF}_{t}^\top\bR^\top\right \|_2 =\mathcal{O}_p\left(\frac{1}{\sqrt{m_i}T} \right),$$
    and 
    $$\left \| \frac{1}{pm_iT^2}\sum_{t=1}^{T}\widetilde{\bE}_{it}\widetilde{\bE}_{it}^\top\right \|_2 =\mathcal{O}_p\left(\frac{1}{\sqrt{m_i}T^{3/2}}+\frac{1}{\sqrt{p}T} \right).$$
    In other words,
    $$\left \| \frac{1}{T}\widehat{\mathbf{M}}_{\mathbf{R}_i}-\frac{1}{pm_iT^2}\sum_{t=1}^{T}\bR\widetilde{\bF}_{t}\bC_i^\top\bC_i\widetilde{\bF}_{t}^\top\bR^\top\right \|_2=\mathcal{O}_p\left(\frac{1}{\sqrt{m_i}T}+\frac{1}{\sqrt{p}T} \right).$$
    On the other hand, by Assumption 3 and Proposition \ref{prop8.5}, we obtain
    $$\left\|\frac{1}{pm_iT^2}\sum_{t=1}^{T}\bR\widetilde{\bF}_{t}\bC_i^\top\bC_i\widetilde{\bF}_{t}^\top\bR^\top-\frac{1}{p}\bR\widetilde{\boldsymbol{\Sigma}}_{\mathbf{FC}_i}^*\bR^\top\right\|_\mathsf{F}=\mathbf{o}_p\left(1\right).$$
    Together, we have
    $$\left \| \frac{1}{T}\widehat{\mathbf{M}}_{\mathbf{R}_i}-\frac{1}{p}\bR\widetilde{\boldsymbol{\Sigma}}_{\mathbf{FC}_i}^*\bR^\top\right \|_2=\mathbf{o}_p\left(1\right).$$
    By the inequality that for the $j$-th largest eigenvalue of any square matrices $\bA$ and $\wh{\bA}$, $|\lambda_j(\widehat{\mathbf{A}})-\lambda_j(\mathbf{A})|\le \| \widehat{\mathbf{A}}-\mathbf{A}\|_2$, as $m_i,p,T\to \infty$, we have
    $$\widehat{\mathbf{V}}_{\mathbf{R}_i}^*\overset{\mathcal{P}}{\longrightarrow}  \mathbf{V}_{\mathbf{R}_i}^*,$$
    where the eigenvalues of $\frac{1}{p}\mathbf{R}\widetilde{\boldsymbol{\Sigma}}_{\mathbf{FC}_i}^*\mathbf{R}^\top$ and $\frac{1}{p}\mathbf{R}^\top\mathbf{R}\widetilde{\boldsymbol{\Sigma}}_{\mathbf{FC}_i}^*$ are the same, $\frac{1}{p}\mathbf{R}^\top\mathbf{R}\to\boldsymbol{\Omega}_\mathbf{R}$ by Assumption 3, and the eigenvalues of $\boldsymbol{\Omega}_\mathbf{R}\widetilde{\boldsymbol{\Sigma}}_{\mathbf{FC}_i}^*$ and $(\widetilde{\boldsymbol{\Sigma}}_{\mathbf{FC}_i}^*)^{1/2}\boldsymbol{\Omega}_\mathbf{R}(\widetilde{\boldsymbol{\Sigma}}_{\mathbf{FC}_i}^*)^{1/2}$ are also the same since $\boldsymbol{\Omega}_\mathbf{R}$ and $\widetilde{\boldsymbol{\Sigma}}_{\mathbf{FC}_i}^*$ are both positive definite. Furthermore, the top $k$ eigenvalues of $\frac{1}{T}\widehat{\mathbf{M}}_{\mathbf{R}_i}$ are bounded away from zero and infinity, and therefore we obtain $\|\widehat{\mathbf{V}}_{\mathbf{R}_i}^*\|_2=\mathcal{O}_p(1)$ and $\|(\widehat{\mathbf{V}}_{\mathbf{R}_i}^*)^{-1}\|_2=\mathcal{O}_p(1)$.
\end{proof}

\begin{lemma}\label{lem4.5}
	Under Assumptions 3-4 and 7-9 with known $k$ and $r$, for any $i\in[s_1]$ and $j\in[s_2]$, we have $\|\bH_{\bR_i}^*\|_2=\mathcal{O}_p(1)$ and $\|\bH_{\bC_j}^*\|_2=\mathcal{O}_p(1)$. 
\end{lemma}
\begin{proof}
	We only prove the result for $\|\bH_{\bR_i}^*\|_2$, since the proof for $\|\bH_{\bC_j}^*\|_2$ is similar.
	Recall that 
    $$\bH_\bR^*=\frac{1}{pm_iT^2}\sum_{t=1}^T\widetilde{\bF}_t\bC_i^\top\bC_i\widetilde{\bF}_t^\top\bR^\top\wh{\bR}_i(\widehat{\bV}_{\bR_i}^*)^{-1}.$$ 
    By Proposition \ref{prop8.5}, as $p,m_i,T\to\infty$, we have
    $$\frac{1}{m_iT^2}\sum_{t=1}^T\widetilde{\bF}_t\bC_i^\top\bC_i\widetilde{\bF}_t^\top\overset{\mathcal{P}}{\longrightarrow}\widetilde{\boldsymbol{\Sigma}}_{\mathbf{FC}_i}^*.$$ 
    Moreover, $\frac{1}{\sqrt{p}}\bR$ and $(\wh{\bV}_{\bR_i}^*)^{-1}$ are bounded according to Assumption 3 and Proposition \ref{prop10}, while the estimator $\frac{1}{\sqrt{p}}\wh{\bR}_i$ is also bounded. Therefore, the final results are derived.
\end{proof}

\begin{proposition}\label{prop11}
    Under Assumptions 3-4 and 7-9 with known $k$ and $r$, for any $i\in[s_1]$ and $j\in[s_2]$, as $m_i,n_j,T\to\infty$, we have
    $$\frac{\widehat{\bR}_i^\top\bR}{p}\overset{\mathcal{P}}{\longrightarrow}\bQ_{\bR_i}^*\quad\text{and}\quad\frac{\widehat{\bC}_j^\top\bC}{q}\overset{\mathcal{P}}{\longrightarrow}\bQ_{\bC_j}^*.$$
    The matrices $\bQ_{\bR_i}^*=(\bV_{\bR_i}^*)^{1/2}(\boldsymbol{\Psi}_{\mathbf{R}_i}^*)^\top(\widetilde{\boldsymbol{\Sigma}}_{\mathbf{FC}_i}^*)^{-1/2}$ and $\bQ_{\bC_j}^*=(\bV_{\bC_j}^*)^{1/2}(\boldsymbol{\Psi}_{\mathbf{C}_j}^*)^\top(\widetilde{\boldsymbol{\Sigma}}_{\mathbf{FR}_j}^*)^{-1/2}$, where the matrices $\mathbf{V}_{\mathbf{R}_i}^*$ and $\mathbf{V}_{\mathbf{C}_j}^*$ are diagonal matrices of eigenvalues (ordered decreasingly) of $(\widetilde{\boldsymbol{\Sigma}}_{\mathbf{FC}_i}^*)^{1/2}\boldsymbol{\Omega}_\mathbf{R}(\widetilde{\boldsymbol{\Sigma}}_{\mathbf{FC}_i}^*)^{1/2}$ and $(\widetilde{\boldsymbol{\Sigma}}_{\mathbf{FR}_j}^*)^{1/2}\boldsymbol{\Omega}_\mathbf{C}(\widetilde{\boldsymbol{\Sigma}}_{\mathbf{FR}_j}^*)^{1/2}$, respectively, with the corresponding eigenvectors $\boldsymbol{\Psi}_{\mathbf{R}_i}^*$ and $\boldsymbol{\Psi}_{\mathbf{C}_j}^*$ satisfying $(\boldsymbol{\Psi}_{\mathbf{R}_i}^*)^\top\boldsymbol{\Psi}_{\mathbf{R}_i}^*=\mathbf{I}$ and $(\boldsymbol{\Psi}_{\mathbf{C}_j}^*)^\top\boldsymbol{\Psi}_{\mathbf{C}_j}^*=\mathbf{I}$. Here $\widetilde{\boldsymbol{\Sigma}}_{\mathbf{FC}_i}^*=\int_0^1 \widetilde{\bW}(u)\boldsymbol{\Omega}_{\bC_i}\widetilde{\bW}(u)^\top\text{d}u$ and $\widetilde{\boldsymbol{\Sigma}}_{\mathbf{FR}_j}^*=\int_0^1 \widetilde{\bW}(u)^\top\boldsymbol{\Omega}_{\bR_j}\widetilde{\bW}(u)\text{d}u$, where $\widetilde{\bW}(u)=\bW(u)+\frac{\widetilde{\alpha}}{T}\sum_{t=1}^T\bW(t/T)$ and $\bW(\cdot)$ is an $k\times r$ matrix of Brownian motions with the covariance of $\text{vec}(\bW(\cdot))$ being $\overline{\bG}\boldsymbol{\Sigma}_\bu\overline{\bG}^\top$.
\end{proposition}
\begin{proof}
    We only prove the conclusion for $\widehat{\bR}_i^\top\bR$, since the proof of results for $\widehat{\bC}_j^\top\bC$ is similar. According to Equation (2.5) and the definition of $\widehat{\bV}_{\bR_i}^*$ in Section 4, we have
    $$\widehat{\bR}_i\widehat{\bV}_{\bR_i}^*=\frac{1}{pm_iT^2}\sum_{t=1}^T\widetilde{\bY}_{it}\widetilde{\bY}_{it}^\top\widehat{\bR}_i.$$
    By left-multiplying $\frac{1}{p}(\frac{1}{m_iT^2}\sum_{t=1}^{T}\widetilde{\bF}_{t}\bC_i^\top\bC_i\widetilde{\bF}_{t}^\top)^{1/2}\bR^\top$ on the both sides and expanding $\widetilde{\bY}_{it}$ as $\widetilde{\bY}_{it}=\bR\widetilde{\bF}_{t}\bC_i^\top+\widetilde{\bE}_{it}^\top$, we have
    \begin{equation}
        \bB_i^*\widehat{\bV}_{\bR_i}^*=\left[\bA_i^*+\bd_i^*(\bB_i^*)^{-1}\right]\bB_i^*,\label{prop9eq1}
    \end{equation}
    where $k\times k$ matrices $\bA_i^*$, $\bB_i^*$, and $\bd_i^*$ are defined as
    $$\bA_i^*=\left(\frac{1}{m_iT^2}\sum_{t=1}^{T}\widetilde{\bF}_{t}\bC_i^\top\bC_i\widetilde{\bF}_{t}^\top\right)^{1/2}\frac{\bR^\top\bR}{p}\left(\frac{1}{m_iT^2}\sum_{t=1}^{T}\widetilde{\bF}_{t}\bC_i^\top\bC_i\widetilde{\bF}_{t}^\top\right)^{1/2},$$
    $$\bB_i^*=\left(\frac{1}{m_iT^2}\sum_{t=1}^{T}\widetilde{\bF}_{t}\bC_i^\top\bC_i\widetilde{\bF}_{t}^\top\right)^{1/2}\frac{\bR^\top\widehat{\bR}_i}{p},$$
    and 
    \begin{align*}
        \bd_i^*&=\frac{1}{p}\left(\frac{1}{m_iT^2}\sum_{t=1}^{T}\widetilde{\bF}_{t}\bC_i^\top\bC_i\widetilde{\bF}_{t}^\top\right)^{1/2}\bR^\top\left[\frac{1}{pm_iT^2}\sum_{t=1}^{T}\bR\widetilde{\bF}_{t}\bC_i^\top\widetilde{\bE}_{it}^\top\widehat{\bR}_i\right.\\
        &\ \ \ \left.+\frac{1}{pm_iT^2}\sum_{t=1}^{T}\widetilde{\bE}_{it}\bC_i\widetilde{\bF}_{t}^\top\bR^\top\widehat{\bR}_i+\frac{1}{pm_iT^2}\sum_{t=1}^{T}\widetilde{\bE}_{it}\widetilde{\bE}_{it}^\top\widehat{\bR}_i\right]\\
        &=\frac{1}{p}\left(\frac{1}{m_iT^2}\sum_{t=1}^{T}\widetilde{\bF}_{t}\bC_i^\top\bC_i\widetilde{\bF}_{t}^\top\right)^{1/2}\bR^\top\left(\mathrm{I}_1+\mathrm{I}_2+\mathrm{I}_3\right)\widehat{\bR}_i.
    \end{align*}
    According to Assumption 3 and Proposition \ref{prop8.5}, we have
    $$\left\|\frac{1}{\sqrt{p}}\left(\frac{1}{m_iT^2}\sum_{t=1}^{T}\widetilde{\bF}_{t}\bC_i^\top\bC_i\widetilde{\bF}_{t}^\top\right)^{1/2}\bR^\top\right\|_2=\mathcal{O}_p(1).$$
    Following Proposition~\ref{prop10}, we have $\|\mathrm{I}_1\|_2=\mathbf{o}_p\left(1\right)$, $\|\mathrm{I}_2\|_2=\mathbf{o}_p\left(1\right)$, and $\|\mathrm{I}_3\|_2=\mathbf{o}_p\left(1\right)$. Moreover, $\|\frac{1}{\sqrt{p}}\widehat{\bR}_i\|_2=\mathcal{O}_p(1)$. Together, we have $$\|\bd^*_i\|_2=\mathbf{o}_p\left(1\right).$$
    According to Proposition~\ref{prop10}, as $p,m_i,T\to\infty$, we have
    $$(\bB_i^*)^\top\bB_i^*=\frac{\widehat{\bR}_i^\top\bR}{p}\left(\frac{1}{m_iT^2}\sum_{t=1}^{T}\widetilde{\bF}_{t}\bC_i^\top\bC_i\widetilde{\bF}_{t}^\top\right)\frac{\bR^\top\widehat{\bR}_i}{p}\overset{\mathcal{P}}{\longrightarrow}\frac{1}{pT}\widehat{\bR}_i^\top\widehat{\bM}_{\bR_i}\widehat{\bR}_i=\widehat{\bV}_{\bR_i}^*\overset{\mathcal{P}}{\longrightarrow}{\bV}_{\bR_i}^*.$$
    In other words, the eigenvalues of $(\bB_i^*)^\top\bB_i^*$ are not close to zero nor infinity, and therefore $\|(\bB_i^*)^{-1}\|_2=\|(\bB_i^*)^{-1}[(\bB_i^*)^{-1}]^\top\|_2^{1/2}=\mathcal{O}_p(1)$. 
    
    Let $\bV_{\bB_i}^*$ denote the diagonal matrix consisting of diagonal elements of $(\bB_i^*)^\top\bB_i^*$, satisfying $\bV_{\bB_i}^*\overset{\mathcal{P}}{\longrightarrow}{\bV}_{\bR_i}^*$, as $\underline{m},p,T\to\infty$. Denote the normalized $k\times k$ matrix $\boldsymbol{\Psi}_{\bB_i}^*=\bB_i^*(\bV_{\bB_i}^*)^{-1/2}$. Then, $\|\boldsymbol{\Psi}_{\bB_i}^*\|_2=\sqrt{\lambda_{\max} {[(\bV_{\bB_i}^*)^{-1}(\bB_i^*)^\top\bB_i^*]}} \overset{\mathcal{P}}{\longrightarrow} 1$. According to Equation~(\ref{prop9eq1}), we have
    $$\boldsymbol{\Psi}_{\bB_i}^*\widehat{\bV}_{\bR_i}^*=\left[\bA_i^*+\bd_i^*(\bB_i^*)^{-1}\right]\boldsymbol{\Psi}_{\bB_i}^*.$$
    Since $\|\bd^*_i\|_2=\mathbf{o}_p\left(1\right)$ and $\|(\bB_i^*)^{-1}\|_2=\mathcal{O}_p(1)$, by Assumption 3 and Proposition \ref{prop8.5}, as $\underline{m},p,T\to \infty$, we have
    $$\bA_i^*+\bd_i^*(\bB_i^*)^{-1}=\bA_i^*+\mathbf{o}_p\left(1\right)\overset{\mathcal{P}}{\longrightarrow}(\widetilde{\boldsymbol{\Sigma}}_{\mathbf{FC}_i}^*)^{1/2}\boldsymbol{\Omega}_\mathbf{R}(\widetilde{\boldsymbol{\Sigma}}_{\mathbf{FC}_i}^*)^{1/2}.$$
    According to the eigenvector perturbation theory (\citeauthor{franklin2012}, \citeyear{franklin2012}) and Assumption 9, there exists a unique orthogonal eigenvector matrix $\Psi_{\mathbf{R}_i}^*$ of $(\widetilde{\boldsymbol{\Sigma}}_{\mathbf{FC}_i}^*)^{1/2}\boldsymbol{\Omega}_\mathbf{R}(\widetilde{\boldsymbol{\Sigma}}_{\mathbf{FC}_i}^*)^{1/2}$, such that $$\left\|\boldsymbol{\Psi}_{\mathbf{R}_i}^*-\boldsymbol{\Psi}_{\bB_i}^*\right\|_2=\mathbf{o}_p(1).$$
    Finally, we complete our proof as
    \begin{align*}
        \frac{\mathbf{R}^\top\widehat{\mathbf{R}}_i}{p}&=\left(\frac{1}{m_iT^2}\sum_{t=1}^{T}\widetilde{\bF}_{t}\bC_i^\top\bC_i\widetilde{\bF}_{t}^\top\right)^{-1/2}\boldsymbol{\Psi}_{\bB_i}^*(\bV_{\bB_i}^*)^{1/2}\\&\overset{\mathcal{P}}{\longrightarrow}(\widetilde{\boldsymbol{\Sigma}}_{\mathbf{FC}_i}^*)^{-1/2}\boldsymbol{\Psi}_{\mathbf{R}_i}^*(\mathbf{V}_{\bR_i}^*)^{1/2}=(\bQ_{\bR_i}^*)^\top.
    \end{align*}
\end{proof}

\begin{proposition}\label{prop12}
    Under Assumptions 3-4 and 7-9 with known $k$ and $r$, for any $i\in[s_1]$ and $j\in[s_2]$, as $m_i,n_j,T\to\infty$, we have
    $$\bH_{\bR_i}^*\overset{\mathcal{P}}{\longrightarrow}(\bQ_{\bR_i}^*)^{-1}\quad\text{and}\quad\bH_{\bC_j}^*\overset{\mathcal{P}}{\longrightarrow}(\bQ_{\bC_j}^*)^{-1}.$$
    The matrices $\bQ_{\bR_i}^*=(\bV_{\bR_i}^*)^{1/2}(\boldsymbol{\Psi}_{\mathbf{R}_i}^*)^\top(\widetilde{\boldsymbol{\Sigma}}_{\mathbf{FC}_i}^*)^{-1/2}$ and $\bQ_{\bC_j}^*=(\bV_{\bC_j}^*)^{1/2}(\boldsymbol{\Psi}_{\mathbf{C}_j}^*)^\top(\widetilde{\boldsymbol{\Sigma}}_{\mathbf{FR}_j}^*)^{-1/2}$, where the matrices $\mathbf{V}_{\mathbf{R}_i}^*$ and $\mathbf{V}_{\mathbf{C}_j}^*$ are diagonal matrices of eigenvalues (ordered decreasingly) of $(\widetilde{\boldsymbol{\Sigma}}_{\mathbf{FC}_i}^*)^{1/2}\boldsymbol{\Omega}_\mathbf{R}(\widetilde{\boldsymbol{\Sigma}}_{\mathbf{FC}_i}^*)^{1/2}$ and $(\widetilde{\boldsymbol{\Sigma}}_{\mathbf{FR}_j}^*)^{1/2}\boldsymbol{\Omega}_\mathbf{C}(\widetilde{\boldsymbol{\Sigma}}_{\mathbf{FR}_j}^*)^{1/2}$, respectively, with the corresponding eigenvectors $\boldsymbol{\Psi}_{\mathbf{R}_i}^*$ and $\boldsymbol{\Psi}_{\mathbf{C}_j}^*$ satisfying $(\boldsymbol{\Psi}_{\mathbf{R}_i}^*)^\top\boldsymbol{\Psi}_{\mathbf{R}_i}^*=\mathbf{I}$ and $(\boldsymbol{\Psi}_{\mathbf{C}_j}^*)^\top\boldsymbol{\Psi}_{\mathbf{C}_j}^*=\mathbf{I}$. Here $\widetilde{\boldsymbol{\Sigma}}_{\mathbf{FC}_i}^*=\int_0^1 \widetilde{\bW}(u)\boldsymbol{\Omega}_{\bC_i}\widetilde{\bW}(u)^\top\text{d}u$ and $\widetilde{\boldsymbol{\Sigma}}_{\mathbf{FR}_j}^*=\int_0^1 \widetilde{\bW}(u)^\top\boldsymbol{\Omega}_{\bR_j}\widetilde{\bW}(u)\text{d}u$, where $\widetilde{\bW}(u)=\bW(u)+\frac{\widetilde{\alpha}}{T}\sum_{t=1}^T\bW(t/T)$ and $\bW(\cdot)$ is an $k\times r$ matrix of Brownian motions with the covariance of $\text{vec}(\bW(\cdot))$ being $\overline{\bG}\boldsymbol{\Sigma}_\bu\overline{\bG}^\top$.
\end{proposition}
\begin{proof}
    We only prove the conclusion for $\mathbf{H}_{\mathbf{R}_i}^*$, since proof of results for $\mathbf{H}_{\mathbf{C}_j}^*$ is similar. By the definition of $\mathbf{H}_{\mathbf{R}_i}^*$ in Section 4, we have
    \begin{align*}
        \bH_{\bR_i}^*&=\frac{1}{pm_iT^2}\sum_{t=1}^T\widetilde{\bF}_t\bC_i^\top\bC_i\widetilde{\bF}_t^\top\bR^\top\wh{\bR}_i(\widehat{\bV}_{\bR_i}^*)^{-1}\\&\overset{\mathcal{P}}{\longrightarrow}\widetilde{\boldsymbol{\Sigma}}_{\mathbf{FC}_i}^*(\bQ_{\bR_i}^*)^\top({\bV}_{\bR_i}^*)^{-1}=(\widetilde{\boldsymbol{\Sigma}}_{\mathbf{FC}_i}^*)^{1/2}\boldsymbol{\Psi}_{\mathbf{R}_i}^*(\mathbf{V}_{\bR_i}^*)^{-1/2}=(\bQ_{\bR_i}^*)^{-1},
    \end{align*}
    where the second step follows from Lemma B.3. in \cite{li2025factormodelsmatrixvaluedtime} and Propositions~\ref{prop10}-\ref{prop11}, and the last step follows from $(\boldsymbol{\Psi}_{\mathbf{R}_i}^*)^\top\boldsymbol{\Psi}_{\mathbf{R}_i}^*=\mathbf{I}$.
\end{proof}

\begin{lemma}\label{lem5}
    Under Assumptions 3-4 and 7-9 with known $k$ and $r$, for any $i\in[s_1]\text{ and }j\in[s_2]$, we have
\begin{equation*}
    \frac{1}{p}\left\|\widehat{\mathbf{R}}_i-\mathbf{R H}_{\bR_i}^*\right\|_\mathsf{F}^2=\mathcal{O}_p\left(\frac{1}{m_iT^2}+\frac{1}{p^2T^2}\right),\quad\text{as }m_i,p,T\to\infty,
\end{equation*}
and
\begin{equation*}
    \frac{1}{q}\left\|\widehat{\mathbf{C}}_j-\mathbf{C H}_{\bC_j}^*\right\|_\mathsf{F}^2=\mathcal{O}_p\left(\frac{1}{n_jT^2}+\frac{1}{q^2T^2}\right),\quad\text{as }n_j,q,T\to\infty.
\end{equation*}
For any $l_1\in[p]$ and $l_2\in[q]$, we have
\begin{equation*}
    \left\|\widehat{\mathbf{R}}_i^{l_1}-\mathbf{R}^{l_1}\bH_{\bR_i}^*\right\|_2^2=\mathcal{O}_p\left(\frac{1}{m_iT^2}+\frac{1}{p^2T^2}\right),\quad\text{as }m_i,p,T\to\infty,
\end{equation*}
and
\begin{equation*}
    \left\|\widehat{\mathbf{C}}_j^{l_2}-\mathbf{C}^{l_2}\bH_{\bC_j}^*\right\|_2^2=\mathcal{O}_p\left(\frac{1}{n_jT^2}+\frac{1}{q^2T^2}\right),\quad\text{as }n_j,q,T\to\infty.
\end{equation*}
\end{lemma}
\begin{proof}
    We only prove the conclusion for $\widehat{\bR}_i-\bR\bH_{\bR_i}^*$, since the proof of the results for $\widehat{\bC}_j-\bC\bH_{\bC_j}^*$ is similar. From the estimation procedure of $\widehat{\bR}_i$ and the definition of $\mathbf{H}_{\mathbf{R}_i}^*$ in Section 4, we have
    \begin{align*}
        \widehat{\bR}_i-\bR\bH_{\bR_i}^*&=\frac{1}{pm_iT^2}\sum_{t=1}^{T}\widetilde{\mathbf{Y}}_{it}\widetilde{\mathbf{Y}}_{it}^\top\widehat{\bR}_i(\widehat{\bV}_{\bR_i}^*)^{-1}-\frac{1}{pm_iT^2}\sum_{t=1}^T\bR\widetilde{\bF}_t\bC_i^\top\bC_i\widetilde{\bF}_t^\top\bR^\top\wh{\bR}_i(\widehat{\bV}_{\bR_i}^*)^{-1}\\
        &=\frac{1}{pm_iT^2}\sum_{t=1}^{T}\left(\bR\widetilde{\bF}_t\bC_i^\top+\widetilde{\bE}_{it}\right)\left(\bR\widetilde{\bF}_t\bC_i^\top+\widetilde{\bE}_{it}\right)^\top\widehat{\bR}_i(\widehat{\bV}_{\bR_i}^*)^{-1}\\
        &\ \ \ -\frac{1}{pm_iT^2}\sum_{t=1}^{T}\bR\widetilde{\bF}_t\bC_i^\top\bC_i\widetilde{\bF}_t^\top\bR^\top\widehat{\bR}_i(\widehat{\bV}_{\bR_i}^*)^{-1}\\
        &=\frac{1}{pm_iT^2}\sum_{t=1}^{T}\left(\bR\widetilde{\bF}_t\bC_i^\top\widetilde{\bE}_{it}^\top+\widetilde{\bE}_{it}\bC_i\widetilde{\bF}_t^\top\bR^\top+\widetilde{\bE}_{it}\widetilde{\bE}_{it}^\top\right)\widehat{\bR}_i(\widehat{\bV}_{\bR_i}^*)^{-1}\\
        &=\mathrm{I}_1+\mathrm{I}_2+\mathrm{I}_3.
    \end{align*}
    According to Assumption 3 and Propositions~\ref{prop9}-\ref{prop10}, $\text{as }m_i,p,T\to\infty$, we have,
    \begin{align*}
        \frac{1}{p}\left\|\mathrm{I}_1\right\|_\mathsf{F}^2&=\frac{1}{p}\left\|\frac{1}{pm_iT^2}\sum_{t=1}^{T}\bR\widetilde{\bF}_t\bC_i^\top\widetilde{\bE}_{it}^\top\widehat{\bR}_i(\widehat{\bV}_{\bR_i}^*)^{-1}\right\|_\mathsf{F}^2\\
        &\le\frac{1}{m_iT^2}\frac{\|\bR\|_\mathsf{F}^2}{p}\frac{\left\|\sum_{t=1}^{T}\widetilde{\bF}_t\bC_i^\top\widetilde{\bE}_{it}^\top\right\|_\mathsf{F}^2}{pm_iT^2}\frac{\|\widehat{\bR}_i\|_\mathsf{F}^2}{p}\left\|(\widehat{\bV}_{\bR_i}^*)^{-1}\right\|_\mathsf{F}^2\\
        &=\mathcal{O}_p\left(\frac{1}{m_iT^2}\right),
    \end{align*}
    and 
    \begin{align*}
        \frac{1}{p}\left\|\mathrm{I}_2\right\|_\mathsf{F}^2&=\frac{1}{p}\left\|\frac{1}{pm_iT^2}\sum_{t=1}^{T}\widetilde{\bE}_{it}\bC_i\widetilde{\bF}_t^\top\bR^\top\widehat{\bR}_i(\widehat{\bV}_{\bR_i}^*)^{-1}\right\|_\mathsf{F}^2\\
        &\le\frac{1}{m_iT^2}\frac{\left\|\sum_{t=1}^{T}\widetilde{\bE}_{it}\bC_i\widetilde{\bF}_t^\top\right\|_\mathsf{F}^2}{pm_iT^2}\frac{\|\bR^\top\|_\mathsf{F}^2}{p}\frac{\|\widehat{\bR}_i\|_\mathsf{F}^2}{p}\left\|(\widehat{\bV}_{\bR_i}^*)^{-1}\right\|_\mathsf{F}^2\\
        &=\mathcal{O}_p\left(\frac{1}{m_iT^2}\right).
    \end{align*}
    Since $\bE_t$ is stationary and $\alpha$-mixing, by Lemma B.2. in \cite{yu2020}, as $m_i,p,T\to\infty$, we have $\frac{1}{p}\|\frac{1}{pm_iT}\sum_{t=1}^{T}\widetilde{\bE}_{it}\widetilde{\bE}_{it}^\top\widehat{\bR}_i\|_\mathsf{F}^2=\mathcal{O}_p(\frac{1}{p^2}+\frac{1}{pm_iT})+\mathbf{o}_p(1)\times\frac{1}{p}\|\widehat{\mathbf{R}}_i-\mathbf{R H}_{\bR_i}^*\|_\mathsf{F}^2$. Accordingly, we have
    $$\frac{1}{p}\left\|\mathrm{I}_3\right\|_\mathsf{F}^2=\mathcal{O}_p\left(\frac{1}{p^2T^2}+\frac{1}{pm_iT^3}\right)+\mathbf{o}_p\left(\frac{1}{T^2}\right)\times\frac{1}{p}\left\|\widehat{\mathbf{R}}_i-\mathbf{R H}_{\bR_i}^*\right\|_\mathsf{F}^2.$$
    Together, as $m_i,p,T\to\infty$, we have
    $$\frac{1}{p}\left\|\widehat{\mathbf{R}}_i-\mathbf{R H}_{\bR_i}^*\right\|_\mathsf{F}^2=\mathcal{O}_p\left(\frac{1}{m_iT^2}+\frac{1}{p^2T^2}\right).$$
    In addition, we consider the $l_1$-th row vector,
    \begin{align*}
        \widehat{\bR}_i^{l_1}-\bR^{l_1}\bH_{\bR_i}^*
        &=\frac{1}{pm_iT^2}\sum_{t=1}^{T}\left(\bR^{l_1}\widetilde{\bF}_t\bC_i^\top\widetilde{\bE}_{it}^\top+\widetilde{\bE}_{it}^{l_1}\bC_i\widetilde{\bF}_t^\top\bR^\top+\widetilde{\bE}_{it}^{l_1}\widetilde{\bE}_{it}^\top\right)\widehat{\bR}_i(\widehat{\bV}_{\bR_i}^*)^{-1}\\
        &=\mathrm{I}\hspace{-1.2pt}\mathrm{I}_1+\mathrm{I}\hspace{-1.2pt}\mathrm{I}_2+\mathrm{I}\hspace{-1.2pt}\mathrm{I}_3.
    \end{align*}
    According to Assumption 3 and Propositions~\ref{prop9}-\ref{prop10}, $\text{as }m_i,p,T\to\infty$, we have,
    \begin{align*}
        \left\|\mathrm{I}\hspace{-1.2pt}\mathrm{I}_1\right\|_2^2&=\left\|\frac{1}{pm_iT^2}\sum_{t=1}^{T}\bR^{l_1}\widetilde{\bF}_t\bC_i^\top\widetilde{\bE}_{it}^\top\widehat{\bR}_i(\widehat{\bV}_{\bR_i}^*)^{-1}\right\|_2^2\\
        &\le\frac{1}{m_iT^2}\|\bR^{l_1}\|_2^2\frac{\left\|\sum_{t=1}^{T}\widetilde{\bF}_t\bC_i^\top\widetilde{\bE}_{it}^\top\right\|_2^2}{pm_iT^2}\frac{\|\widehat{\bR}_i\|_2^2}{p}\left\|(\widehat{\bV}_{\bR_i}^*)^{-1}\right\|_2^2\\
        &=\mathcal{O}_p\left(\frac{1}{m_iT^2}\right),
    \end{align*}
    and 
    \begin{align*}
        \left\|\mathrm{I}\hspace{-1.2pt}\mathrm{I}_2\right\|_2^2&=\left\|\frac{1}{pm_iT^2}\sum_{t=1}^{T}\widetilde{\bE}_{it}^{l_1}\bC_i\widetilde{\bF}_t^\top\bR^\top\widehat{\bR}_i(\widehat{\bV}_{\bR_i}^*)^{-1}\right\|_2^2\\
        &\le\frac{1}{m_iT^2}\frac{\left\|\sum_{t=1}^{T}\widetilde{\bE}_{it}^{l_1}\bC_i\widetilde{\bF}_t^\top\right\|_2^2}{m_iT^2}\frac{\|\bR^\top\|_2^2}{p}\frac{\|\widehat{\bR}_i\|_2^2}{p}\left\|(\widehat{\bV}_{\bR_i}^*)^{-1}\right\|_2^2\\
        &=\mathcal{O}_p\left(\frac{1}{m_iT^2}\right),
    \end{align*}
    where Proposition~\ref{prop9} implies that
    $$\left\|\sum_{t=1}^{T}\widetilde{\bE}_{it}^{l_1}\bC_i\widetilde{\bF}_t^\top\right\|_2^2=\left\|\sum_{t=1}^{T}({\bF}_t+\alpha\overline{\bF})\Big[\sum_{\tau=1}^{m_i}(\bC_i^{\tau})^\top e_{it,l_1\tau}\Big]\right\|_2^2=\mathcal{O}_p\left({m_iT^2}\right).$$
    Moreover, Equation (B.4) in \cite{yu2020} implies that $\|\frac{1}{pm_iT}\sum_{t=1}^{T}\widetilde{\bE}_{it}^{l_1}\widetilde{\bE}_{it}^\top\widehat{\bR}_i\|_2^2=\mathcal{O}_p(\frac{1}{p^2}+\frac{1}{pm_iT})+\mathbf{o}_p(1)\times\frac{1}{p}\|\widehat{\mathbf{R}}_i-\mathbf{R H}_{\bR_i}^*\|_2^2$. Accordingly, we have
    $$\left\|\mathrm{I}\hspace{-1.2pt}\mathrm{I}_3\right\|_2^2=\mathcal{O}_p\left(\frac{1}{p^2T^2}+\frac{1}{pm_iT^3}\right)+\mathbf{o}_p\left(\frac{1}{T^2}\right)\times\frac{1}{p}\left\|\widehat{\mathbf{R}}_i-\mathbf{R H}_{\bR_i}^*\right\|_2^2.$$
    Together, for any $l_1\in[p]$, as $m_i,p,T\to\infty$, we have
    $$\left\|\widehat{\mathbf{R}}_i^{l_1}-\mathbf{R}^{l_1}\bH_{\bR_i}^*\right\|_2^2=\mathcal{O}_p\left(\frac{1}{m_iT^2}+\frac{1}{p^2T^2}\right).$$
\end{proof}
We next derive asymptotic properties of the global estimators and the associated auxiliary matrices.
\begin{proposition}\label{prop13}
	Under Assumptions 3-4 and 7-9 with known $k$ and $r$, as $\underline{m},\underline{n},T\to \infty$, we have
 \begin{align*}
     \widehat{\mathbf{V}}_\mathbf{R}\overset{\mathcal{P}}{\longrightarrow}  \mathbf{V}_\bR^*\quad\text{and}\quad
     \widehat{\mathbf{V}}_\mathbf{C}\overset{\mathcal{P}}{\longrightarrow} \mathbf{V}_\bC^*,
 \end{align*}
 with $\|\widehat{\mathbf{V}}_{\mathbf{R}}\|_2=\mathcal{O}_p(1)$, $\|\widehat{\mathbf{V}}_{\mathbf{R}}^{-1}\|_2=\mathcal{O}_p(1)$, $\|\widehat{\mathbf{V}}_{\mathbf{C}}\|_2=\mathcal{O}_p(1)$, and $\|\widehat{\mathbf{V}}_{\mathbf{C}}^{-1}\|_2=\mathcal{O}_p(1)$. Here $\mathbf{V}_\bR^*$ and $\mathbf{V}_\bC^*$ are diagonal matrices of eigenvalues (ordered decreasingly) of $(\boldsymbol{\Sigma}_\mathbf{H_R}^*)^{1/2}\boldsymbol{\Omega}_\mathbf{R}(\boldsymbol{\Sigma}_\mathbf{H_R}^*)^{1/2}$ and $(\boldsymbol{\Sigma}_\mathbf{H_C}^*)^{1/2}\boldsymbol{\Omega}_\mathbf{C}(\boldsymbol{\Sigma}_\mathbf{H_C}^*)^{1/2}$, respectively, with $\boldsymbol{\Sigma}_\mathbf{H_R}^*=\frac{1}{s_1}\sum_{i=1}^{s_1}(\mathbf{Q}_{\mathbf{R}_i}^*)^{-1}[(\mathbf{Q}_{\mathbf{R}_i}^*)^{-1}]^\top$ and $ \boldsymbol{\Sigma}_\mathbf{H_C}^*=\frac{1}{s_2}\sum_{j=1}^{s_2}(\mathbf{Q}_{\mathbf{C}_j}^*)^{-1}[(\mathbf{Q}_{\mathbf{C}_j}^*)^{-1}]^\top$.
\end{proposition}
\begin{proof}
    We only prove the conclusion for $\hat{\bV}_\bR$, since the proof of results for $\hat{\bV}_\bC$ is similar. According to Equation (2.6) and $\frac{1}{p}\wh{\bR}^\top\wh{\bR}=\bI$, we have
    \begin{equation}
        \widehat{\mathbf{V}}_\mathbf{R}=\frac{1}{p}\widehat{\mathbf{R}}^\top\widehat{\mathbf{M}}_\mathbf{R}\widehat{\mathbf{R}}=\frac{1}{p}\widehat{\mathbf{R}}^\top\left( \frac{1}{ps_1}\sum_{i=1}^{s_1}\widehat{\mathbf{R}}_i\widehat{\mathbf{R}}_i^\top\right)\widehat{\mathbf{R}}. 
    \end{equation}
    Writing $\wh{\bR}_i$ as $\wh{\bR}_i=\wh{\bR}_i-\bR\bH_{\bR_i}^*+\bR\bH_{\bR_i}^*$, we obtain
    \begin{align*}
    \widehat{\mathbf{M}}_\mathbf{R}&=\frac{1}{ps_1}\sum_{i=1}^{s_1}(\widehat{\mathbf{R}}_i-\mathbf{R}\mathbf{H}_{\mathbf{R}_i}^*)(\widehat{\mathbf{R}}_i-\mathbf{R}\mathbf{H}_{\mathbf{R}_i}^*)^\top+\frac{1}{ps_1}\sum_{i=1}^{s_1}\mathbf{R}\mathbf{H}_{\mathbf{R}_i}^*(\widehat{\mathbf{R}}_i-\mathbf{R}\mathbf{H}_{\mathbf{R}_i}^*)^\top\\
    &\ \ \ +\frac{1}{ps_1}\sum_{i=1}^{s_1}(\widehat{\mathbf{R}}_i-\mathbf{R}\mathbf{H}_{\mathbf{R}_i}^*)(\mathbf{H}_{\mathbf{R}_i}^*)^\top\mathbf{R}^\top+\frac{1}{ps_1}\sum_{i=1}^{s_1}\mathbf{R}\mathbf{H}_{\mathbf{R}_i}^*(\mathbf{H}_{\mathbf{R}_i}^*)^\top\mathbf{R}^\top.
    \end{align*}
    By Assumption 3 and Lemmas \ref{lem4.5}-\ref{lem5}, we have $\frac{1}{p}\|\widehat{\bR}_i-\bR\bH_{\bR_i}^*\|_\mathsf{F}^2=\mathcal{O}_p(\frac{1}{m_iT^2}+\frac{1}{p^2T^2})$ and $\frac{1}{p}\|\mathbf{R}\mathbf{H}_{\mathbf{R}_i}^*\|_\mathsf{F}^2=\mathcal{O}_p(1)$, for any $i\in[s_1]$. We then bound the first three terms of $\widehat{\mathbf{M}}_\mathbf{R}$ as
    $$\left \| \frac{1}{ps_1}\sum_{i=1}^{s_1}(\widehat{\mathbf{R}}_i-\mathbf{R}\mathbf{H}_{\mathbf{R}_i}^*)(\widehat{\mathbf{R}}_i-\mathbf{R}\mathbf{H}_{\mathbf{R}_i}^*)^\top\right \|_2 =\mathcal{O}_p\left(\frac{1}{\underline{m}T^2}+\frac{1}{p^2T^2} \right),$$
    $$\left \| \frac{1}{ps_1}\sum_{i=1}^{s_1}\mathbf{R}\mathbf{H}_{\mathbf{R}_i}^*(\widehat{\mathbf{R}}_i-\mathbf{R}\mathbf{H}_{\mathbf{R}_i}^*)^\top\right \|_2 =\mathcal{O}_p\left(\frac{1}{\sqrt{\underline{m}}T}+\frac{1}{pT}\right),$$
    and 
    $$\left \| \frac{1}{ps_1}\sum_{i=1}^{s_1}(\widehat{\mathbf{R}}_i-\mathbf{R}\mathbf{H}_{\mathbf{R}_i}^*)(\mathbf{H}_{\mathbf{R}_i}^*)^\top\mathbf{R}^\top\right \|_2 =\mathcal{O}_p\left(\frac{1}{\sqrt{\underline{m}}T}+\frac{1}{pT}\right).$$
    In other words,
    $$\left \| \widehat{\mathbf{M}}_\mathbf{R}-\frac{1}{ps_1}\sum_{i=1}^{s_1}\mathbf{R}\mathbf{H}_{\mathbf{R}_i}^*(\mathbf{H}_{\mathbf{R}_i}^*)^\top\mathbf{R}^\top\right \|_2=\mathcal{O}_p\left(\frac{1}{\sqrt{\underline{m}}T}+\frac{1}{pT}\right).$$
    On the other hand, according to Proposition~\ref{prop12}, we have $\mathbf{H}_{\mathbf{R}_i}^*\overset{\mathcal{P}}{\longrightarrow}(\mathbf{Q}_{\mathbf{R}_i}^*)^{-1}$ and thus $ \mathbf{H}_{\mathbf{R}_i}^*(\mathbf{H}_{\mathbf{R}_i}^*)^\top\overset{\mathcal{P}}{\longrightarrow}(\mathbf{Q}_{\mathbf{R}_i}^*)^{-1}[(\mathbf{Q}_{\mathbf{R}_i}^*)^{-1}]^\top$, as $m_i,p,T\to \infty$, for any $i\in[s_1]$. Then, by Assumption 3, we have
    \begin{align*}
        &\left \| \frac{1}{ps_1}\sum_{i=1}^{s_1}\mathbf{R}\mathbf{H}_{\mathbf{R}_i}^*(\mathbf{H}_{\mathbf{R}_i}^*)^\top\mathbf{R}^\top-\frac{1}{ps_1}\sum_{i=1}^{s_1}\mathbf{R}(\mathbf{Q}_{\mathbf{R}_i}^*)^{-1}[(\mathbf{Q}_{\mathbf{R}_i}^*)^{-1}]^\top\mathbf{R}^\top\right \|_2\\
        &\le \frac{1}{s_1}\sum_{i=1}^{s_1}\left \| \mathbf{H}_{\mathbf{R}_i}^*(\mathbf{H}_{\mathbf{R}_i}^*)^\top-(\mathbf{Q}_{\mathbf{R}_i}^*)^{-1}[(\mathbf{Q}_{\mathbf{R}_i}^*)^{-1}]^\top\right \|_2\frac{1}{p}\left \|\mathbf{R}\right \|_2^2\\
        &= \mathbf{o}_p\left(1\right).
    \end{align*}
    Together, we have
    $$\left \| \widehat{\mathbf{M}}_\mathbf{R}-\frac{1}{ps_1}\sum_{i=1}^{s_1}\mathbf{R}(\mathbf{Q}_{\mathbf{R}_i}^*)^{-1}[(\mathbf{Q}_{\mathbf{R}_i}^*)^{-1}]^\top\mathbf{R}^\top\right \|_2=\left \| \widehat{\mathbf{M}}_\mathbf{R}-\frac{1}{p}\mathbf{R}\boldsymbol{\Sigma}_\mathbf{H_R}^*\mathbf{R}^\top\right \|_2=\mathbf{o}_p\left(1\right).$$
    By the inequality that for the $j$-th largest eigenvalue of any square matrices $\bA$ and $\wh{\bA}$, $|\lambda_j(\widehat{\mathbf{A}})-\lambda_j(\mathbf{A})|\le \| \widehat{\mathbf{A}}-\mathbf{A}\|_2$, as $\underline{m},p,T\to \infty$, we have
    $$\widehat{\mathbf{V}}_\mathbf{R}\overset{\mathcal{P}}{\longrightarrow}  \mathbf{V}_\bR^*,$$
    where the eigenvalues of $\frac{1}{p}\mathbf{R}\boldsymbol{\Sigma}_\mathbf{H_R}^*\mathbf{R}^\top$ and $\frac{1}{p}\mathbf{R}^\top\mathbf{R}\boldsymbol{\Sigma}_\mathbf{H_R}^*$ are the same, $\frac{1}{p}\mathbf{R}^\top\mathbf{R}\to\boldsymbol{\Omega}_\mathbf{R}$ by Assumption 3, and the eigenvalues of $\boldsymbol{\Omega}_\mathbf{R}\boldsymbol{\Sigma}_\mathbf{H_R}^*$ and $(\boldsymbol{\Sigma}_\mathbf{H_R}^*)^{1/2}\boldsymbol{\Omega}_\mathbf{R}(\boldsymbol{\Sigma}_\mathbf{H_R}^*)^{1/2}$ are also the same since $\boldsymbol{\Omega}_\mathbf{R}$ and $\boldsymbol{\Sigma}_\mathbf{H_R}^*$ are both positive definite. Furthermore, the top $k$ eigenvalues of $\widehat{\mathbf{M}}_\mathbf{R}$ are bounded away from zero and infinity, and therefore we have $\|\widehat{\mathbf{V}}_{\mathbf{R}}\|_2=\mathcal{O}_p(1)$ and $\|\widehat{\mathbf{V}}_{\mathbf{R}}^{-1}\|_2=\mathcal{O}_p(1)$.
\end{proof}

\begin{lemma}\label{lem6}
	Under Assumptions 3-4 and 7-9 with known $k$ and $r$, we have $\|\bH_{\bR}^*\|_2=\mathcal{O}_p(1)$ and $\|\bH_{\bC}^*\|_2=\mathcal{O}_p(1)$. 
\end{lemma}
\begin{proof}
	We only prove the result for $\|\bH_{\bR}^*\|_2$, since the proof for $\|\bH_{\bC}^*\|_2$ is similar.
	Recall that $\bH_\bR^*=\frac{1}{ps_1}\sum_{i=1}^{s_1} \bH_{\bR_i}^*(\bH_{\bR_i}^*)^\top\bR^\top\hat{\bR}\widehat{\bV}_{\bR}^{-1}$. The conclusion follows directly from Assumption 3, Proposition~\ref{prop13}, and Lemma~\ref{lem4.5}.
\end{proof}

\subsection*{A.3 Proof of  Theorem 1}

\begin{proof}
We only prove the conclusion for $\hat{\bR}$, since the proof of results for $\hat{\bC}$ is similar.
As $\widehat{\bV}_{\bR} \in \mathbb{R}^{k\times k}$ is consisting of the $k$ largest eigenvalues of $\widehat{\mathbf{M}}_{\mathbf{R}}$ (2.6), we have,
\begin{equation*}
    \hat{\bR}=\widehat{\mathbf{M}}_{\mathbf{R}} \hat{\bR}\widehat{\bV}_{\bR}^{-1}=\frac{1}{ps_1}\sum_{i=1}^{s_1} \wh{\bR}_i\wh{\bR}_i^\top\hat{\bR}\widehat{\bV}_{\bR}^{-1}.
\end{equation*}
Writing $\wh{\bR}_i$ as $\wh{\bR}_i=\wh{\bR}_i-\bR\bH_{\bR_i}+\bR\bH_{\bR_i}$, we obtain
\begin{equation}\label{r-rh}
	\begin{split}
		\hat{\bR}-\bR\bH_\bR&=\frac{1}{ps_1}\sum_{i=1}^{s_1} (\wh{\bR}_i-\bR\bH_{\bR_i})(\wh{\bR}_i-\bR\bH_{\bR_i})^\top\hat{\bR}\widehat{\bV}_{\bR}^{-1} \\
		&\ \ \ +\frac{1}{ps_1}\sum_{i=1}^{s_1} \bR\bH_{\bR_i}(\wh{\bR}_i-\bR\bH_{\bR_i})^\top\hat{\bR}\widehat{\bV}_{\bR}^{-1} \\
		&\ \ \ + \frac{1}{ps_1}\sum_{i=1}^{s_1} (\wh{\bR}_i-\bR\bH_{\bR_i})\bH_{\bR_i}^\top\bR^\top\hat{\bR}\widehat{\bV}_{\bR}^{-1} \\
		&= \mathrm{I}_1+\mathrm{I}_2+\mathrm{I}_3.
	\end{split}
\end{equation}
Next, we bound $\mathrm{I}_1$, $\mathrm{I}_2$, and $\mathrm{I}_3$. We observe that $\mathrm{I}_1$ is dominated by $\mathrm{I}_2$ and $\mathrm{I}_3$, and $\|\mathrm{I}_2\|_\mathsf{F}\asymp \|\mathrm{I}_3\|_\mathsf{F}$. For $\mathrm{I}_2$, by Assumption 3, Propositions~\ref{prop1}-\ref{prop2} and \ref{prop6}, we have
$$\frac{1}{p} \|\mathrm{I}_2\|_\mathsf{F}^2 \leq \frac{C}{s_1}\sum_{i=1}^{s_1} \frac{\|\bR\|_\mathsf{F}^2}{p} \| \bH_{\bR_i}\|_\mathsf{F}^2 \frac{\|\wh{\bR}_i-\bR\bH_{\bR_i}\|_\mathsf{F}^2}{p}\frac{\|\hat{\bR}\|_\mathsf{F}^2}{p}\| \widehat{\bV}_{\bR}^{-1}\|_\mathsf{F}^2=\mathcal{O}_p\left(\frac{1}{\underline{m}T}+\frac{1}{p^2}\right),$$
where $C$ is a positive constant. Thus, as $\underline{m},p,T\to\infty$,
$$\frac{1}{p} \left\|\hat{\bR}-\bR\bH_\bR\right\|_\mathsf{F}^2=\mathcal{O}_p\left(\frac{1}{\underline{m}T}+\frac{1}{p^2}\right).$$

\end{proof}

\subsection*{A.4 Proof of Theorem 2}

\begin{proof}
 From the estimation procedure (2.10), we have
$$
\widehat{\mathbf{F}}_t=\frac{1}{p q} \widehat{\bR}^{\top} \mathbf{Y}_t \widehat{\mathbf{C}}=\frac{1}{p q} \widehat{\mathbf{R}}^{\top} \mathbf{R} \mathbf{F}_t \mathbf{C}^{\top} \widehat{\mathbf{C}}+\frac{1}{p q} \widehat{\mathbf{R}}^{\top} \mathbf{E}_t \widehat{\mathbf{C}} .
$$
Writing $\mathbf{R}=\mathbf{R}-\widehat{\mathbf{R}} \mathbf{H}_\bR^{-1}+\widehat{\mathbf{R}} \mathbf{H}_\bR^{-1}$ and $\mathbf{C}={\mathbf{C}}-\widehat{\mathbf{C}} \mathbf{H}_\bC^{-1}+\widehat{\mathbf{C}} \mathbf{H}_\bC^{-1}$, we obtain
$$
\begin{aligned}
	\widehat{\mathbf{F}}_t-\mathbf{H}_\bR^{-1} \mathbf{F}_t (\mathbf{H}_\bC^{-1})^{\top} & =\frac{1}{p q} \widehat{\mathbf{R}}^{\top}(\mathbf{R}-\widehat{\mathbf{R}} \mathbf{H}_\bR^{-1}) \mathbf{F}_t(\mathbf{C}-\widehat{\mathbf{C}} \mathbf{H}_\bC^{-1})^{\top} \widehat{\mathbf{C}} \\
	&\ \ \  +\frac{1}{p} \widehat{\mathbf{R}}^{\top}(\mathbf{R}-\widehat{\mathbf{R}} \mathbf{H}_\bR^{-1}) \mathbf{F}_t (\mathbf{H}_\bC^{-1})^{\top} \\
	&\ \ \  +\frac{1}{q} \mathbf{H}_\bR^{-1} \mathbf{F}_t(\mathbf{C}-\widehat{\mathbf{C}} \mathbf{H}_\bC^{-1})^{\top} \widehat{\mathbf{C}}  +\frac{1}{p q} \widehat{\mathbf{R}}^{\top} \mathbf{E}_t \widehat{\mathbf{C}} .
\end{aligned}
$$
We further decompose $\widehat{\mathbf{R}}=\widehat{\mathbf{R}}-\mathbf{R H}_\bR+\mathbf{R H}_\bR$ and $\widehat{\mathbf{C}}=\widehat{\mathbf{C}}-\mathbf{C H}_\bC+\mathbf{C H}_\bC$ in the last term of the above equation and rearrange it. We have
\begin{equation}\label{extendFt}
    \begin{split}
        \widehat{\mathbf{F}}_t-\mathbf{H}_\bR^{-1} \mathbf{F}_t (\mathbf{H}_\bC^{-1})^{\top}&=  \frac{1}{p q} \widehat{\mathbf{R}}^{\top}(\mathbf{R}-\widehat{\mathbf{R}} \mathbf{H}_\bR^{-1}) \mathbf{F}_t(\mathbf{C}-\widehat{\mathbf{C}} \mathbf{H}_\bC^{-1})^{\top} \widehat{\mathbf{C}} \\
	&\ \ \ +\frac{1}{p} \widehat{\mathbf{R}}^{\top}(\mathbf{R}-\widehat{\mathbf{R}} \mathbf{H}_\bR^{-1}) \mathbf{F}_t (\mathbf{H}_\bC^{-1})^{\top}  +\frac{1}{q} \mathbf{H}_\bR^{-1} \mathbf{F}_t(\mathbf{C}-\widehat{\mathbf{C}} \mathbf{H}_\bC^{-1})^{\top} \widehat{\mathbf{C}} \\
	&\ \ \  +\frac{1}{p q}(\widehat{\mathbf{R}}-\mathbf{R} \mathbf{H}_\bR)^{\top} \mathbf{E}_t(\widehat{\mathbf{C}}-\mathbf{C H}_\bC)  +\frac{1}{p q}(\widehat{\mathbf{R}}-\mathbf{R} \mathbf{H}_\bR)^{\top} \mathbf{E}_t \mathbf{C H}_\bC \\
	&\ \ \  +\frac{1}{p q} \bH_\bR^\top{\mathbf{R}}^{\top} \mathbf{E}_t(\widehat{\mathbf{C}}-\mathbf{C} \mathbf{H}_\bC) +\frac{1}{p q} \mathbf{H}_\bR^{\top} \mathbf{R}^{\top} \mathbf{E}_t \mathbf{C} \mathbf{H}_\bC \\
	& =\sum_{i=1}^7 \mathrm{I}_i .
    \end{split}	
\end{equation}
Since $\frac{1}{\sqrt{p}}\|\mathbf{R}-\widehat{\mathbf{R}} \mathbf{H}_\bR^{-1}\|_2=\mathbf{o}_p(1)$ and $\frac{1}{\sqrt{q}}\|\mathbf{C}-\widehat{\mathbf{C}} \mathbf{H}_\bC^{-1}\|_2=\mathbf{o}_p(1)$ as $\underline{m},\underline{n},T\to\infty$ by Theorem 1, term $\mathrm{I}_1$ is dominated by $\mathrm{I}_2$ and $\mathrm{I}_3$, and term $\mathrm{I}_4$ is dominated by $\mathrm{I}_5$ and $\mathrm{I}_6$. Now we bound $\mathrm{I}_2, \mathrm{I}_3, \mathrm{I}_5, \mathrm{I}_6$, and $\mathrm{I}_7$.
$$
\begin{aligned}
	\|\mathrm{I}_2\|_\text{2} & =\left\|\frac{1}{p}(\widehat{\mathbf{R}}-\mathbf{R} \mathbf{H}_\bR)^{\top}(\mathbf{R}-\widehat{\mathbf{R}} \mathbf{H}_\bR^{-1}) \mathbf{F}_t (\mathbf{H}_\bC^{-1})^{\top}+\frac{1}{p} \mathbf{H}_\bR^{\top} \mathbf{R}^{\top}(\mathbf{R}-\widehat{\mathbf{R}} \mathbf{H}_\bR^{-1}) \mathbf{F}_t (\mathbf{H}_\bC^{-1})^{\top}\right\|_2 \\
	& =\mathcal{O}_p \left(\left(\frac{1}{\underline{m}T}+\frac{1}{p^2}\right)^{1/2}\right),
\end{aligned}
$$
by Assumptions 2-3, Lemma~\ref{lem2}, and Theorem 1. Similarly, adding Assumption 5, we have
\begin{align*}
    \|\mathrm{I}_3\|_2&=\mathcal{O}_p \left(\left(\frac{1}{\underline{n}T}+\frac{1}{q^2}\right)^{1/2}\right),\quad  \|\mathrm{I}_5\|_2=\mathcal{O}_p\left(\frac{1}{\sqrt{q}}\left(\frac{1}{\underline{m}T}+\frac{1}{p^2}\right)^{1/2} \right), \\
    \|\mathrm{I}_6\|_2&=\mathcal{O}_p\left(\frac{1}{\sqrt{p}}\left(\frac{1}{\underline{n}T}+\frac{1}{q^2}\right)^{1/2}\right), \quad \|\mathrm{I}_7\|_2=\mathcal{O}_p(\frac{1}{\sqrt{pq}}).
\end{align*}
Finally, we have, as $\underline{m},\underline{n},T\to\infty$,
$$
\left\|\widehat{\mathbf{F}}_t-\mathbf{H}_\bR^{-1} \mathbf{F}_t (\mathbf{H}_\bC^{-1})^{\top}\right\|_2=\mathcal{O}_p\left(\frac{1}{\sqrt{\underline{m}T}}+\frac{1}{\sqrt{\underline{n}T}}+\frac{1}{p}+\frac{1}{q}\right).
$$
\end{proof}

\subsection*{A.5 Proof of Theorem 3}
\begin{proof}
From the estimation procedure (2.11), we have $	\wh{\bS}_t=\wh{\bR}\wh{\bF}_t\wh{\bC}^\top$ and $\bS_t=\bR\bF_t\bC^\top$. Define $\bR^\text{H}=\bR\bH_{\bR}$, ${\bC}^\text{H}=\bC\bH_{\bC}$, and $\bF_t^\text{H}=\bH_{\bR}^{-1}\bF_t(\bH_{\bC}^{-1})^\top$. Then, $\bS_t={\bR}^\text{H}{\bF}_t^\text{H}({\bC}^\text{H})^\top$. Under this definition, we have
\begin{equation*}
	\begin{split}
		\wh{\bS}_t-\bS_t&=\wh{\bR}\wh{\bF}_t\wh{\bC}^\top- {\bR}^\text{H}{\bF}_t^\text{H}({\bC}^\text{H})^\top \\
		&=(\wh{\bR}- {\bR}^\text{H})(\wh{\bF}_t-{\bF}_t^\text{H})(\wh{\bC}-{\bC}^\text{H})^\top+(\wh{\bR}- {\bR}^\text{H}){\bF}_t^\text{H}(\wh{\bC}-{\bC}^\text{H})^\top \\
		&\ \ \ + {\bR}^\text{H}(\wh{\bF}_t-{\bF}_t^\text{H})(\wh{\bC}-{\bC}^\text{H})^\top +(\wh{\bR}- {\bR}^\text{H})(\wh{\bF}_t-{\bF}_t^\text{H})({\bC}^\text{H})^\top\\
		&\ \ \ +(\wh{\bR}- {\bR}^\text{H}){\bF}_t^\text{H}({\bC}^\text{H})^\top+{\bR}^\text{H}(\wh{\bF}_t-{\bF}_t^\text{H})({\bC}^\text{H})^\top+{\bR}^\text{H}{\bF}_t^\text{H}(\wh{\bC}-{\bC}^\text{H} )^\top.
	\end{split}
\end{equation*} 
Then, for any $i\in[p]$ and $j\in[q]$, we have
\begin{equation}\label{extendst}
    \begin{split}
        \wh{\bS}_{t,ij}-\bS_{t,ij}&=[\wh{\bR}^i- ({\bR}^\text{H})^i](\wh{\bF}_t-{\bF}_t^\text{H})[\wh{\bC}^j-({\bC}^\text{H})^j]^\top+[\wh{\bR}^i- ({\bR}^\text{H})^i]{\bF}_t^\text{H}[\wh{\bC}^j-({\bC}^\text{H})^j]^\top\\
		&\ \ \ + ({\bR}^\text{H})^i(\wh{\bF}_t-{\bF}_t^\text{H})[\wh{\bC}^j-({\bC}^\text{H})^j]^\top +[\wh{\bR}^i- ({\bR}^\text{H})^i](\wh{\bF}_t-{\bF}_t^\text{H})[({\bC}^\text{H})^j]^\top\\
        &\ \ \ +({\bR}^\text{H})^i(\wh{\bF}_t-{\bF}_t^\text{H})[({\bC}^\text{H})^j]^\top+({\bR}^\text{H})^i{\bF}_t^\text{H}[\wh{\bC}^j-({\bC}^\text{H})^j]^\top\\
        &\ \ \ +[\wh{\bR}^i- ({\bR}^\text{H})^i]{\bF}_t^\text{H}[({\bC}^\text{H})^j]^\top.
    \end{split}
\end{equation}
Since $\|\widehat{\mathbf{R}}^i-{\mathbf{R}}^i \mathbf{H}_\bR\|_2=\mathbf{o}_p(1)$, $\|\widehat{\mathbf{C}}^j-{\mathbf{C}}^j \mathbf{H}_\bC\|_2=\mathbf{o}_p(1)$, and $\|\widehat{\mathbf{F}}_t-\mathbf{H}_\bR^{-1} \mathbf{F}_t (\mathbf{H}_\bC^{-1})^{\top}\|_2=\mathbf{o}_p(1)$ as $\underline{m},\underline{n},T\to\infty$ by Theorems 1-2, the last three terms dominate $\wh{\bS}_{t,ij}-\bS_{t,ij}$. By the Cauchy-Schwarz inequality,
$$({\bR}^\text{H})^i(\wh{\bF}_t-{\bF}_t^\text{H})[({\bC}^\text{H})^j]^\top\le\|({\bR}^\text{H})^i\|_2\|\wh{\bF}_t-{\bF}_t^\text{H}\|_2\|({\bC}^\text{H})^j\|_2,$$
$$({\bR}^\text{H})^i{\bF}_t^\text{H}[\wh{\bC}^j-({\bC}^\text{H})^j]^\top\le\|({\bR}^\text{H})^i\|_2\|{\bF}_t^\text{H}\|_2\|\wh{\bC}^j-({\bC}^\text{H})^j\|_2,$$
and
$$[\wh{\bR}^i- ({\bR}^\text{H})^i]{\bF}_t^\text{H}[({\bC}^\text{H})^j]^\top\le\|\wh{\bR}^i- ({\bR}^\text{H})^i\|_2\|{\bF}_t^\text{H}\|_2\|({\bC}^\text{H})^j\|_2.$$
By Assumptions 2-3, Lemma~\ref{lem2}, and Theorems 1-2, we have
$$
\wh{\bS}_{t,ij}-\bS_{t,ij}=\mathcal{O}_p\left(\frac{1}{\sqrt{\underline{m}T}}+\frac{1}{\sqrt{\underline{n}T}}+\frac{1}{p}+\frac{1}{q}\right).
$$
\end{proof}

\subsection*{A.6 Proof of Theorem 4}

\begin{proof}
We only prove the conclusion for $\widehat{\bR}^{i_1}-\bR^{i_1}\bH_\bR$, since the proof of results for $\widehat{\bC}^{i_2}-\bC^{i_2}\bH_\bC$ is similar.
    First, we expand $\widehat{\bR}^{i_1}-\bR^{i_1}\bH_\bR$ as Equation~(\ref{r-rh}),
    \begin{align*}
        \widehat{\bR}^{i_1}-\bR^{i_1}\bH_\bR&=\left[\frac{1}{ps_1}\sum_{j=1}^{s_1}(\widehat{\mathbf{R}}^{i_1}_j-\mathbf{R}^{i_1}\mathbf{H}_{\mathbf{R}_j})(\widehat{\mathbf{R}}_{j}-\mathbf{R}\mathbf{H}_{\mathbf{R}_j})^\top\widehat{\mathbf{R}}\right. \\  &\left.\ \ \ +\frac{1}{ps_1}\sum_{j=1}^{s_1}(\widehat{\mathbf{R}}^{i_1}_j-\mathbf{R}^{i_1}\mathbf{H}_{\mathbf{R}_j})\mathbf{H}_{\mathbf{R}_j}^\top\mathbf{R}^\top\widehat{\mathbf{R}}\right. \\  &\left.\ \ \ +\frac{1}{ps_1}\sum_{j=1}^{s_1}\mathbf{R}^{i_1}\mathbf{H}_{\mathbf{R}_j}(\widehat{\mathbf{R}}_{j}-\mathbf{R}\mathbf{H}_{\mathbf{R}_j})^\top\widehat{\mathbf{R}}
        \right]\widehat{\bV}_\bR^{-1}\\
        &=\left(\mathrm{I}+\mathrm{I}\hspace{-1.2pt}\mathrm{I}+\mathrm{I}\hspace{-1.2pt}\mathrm{I}\hspace{-1.2pt}\mathrm{I}\right)\widehat{\bV}_\bR^{-1}.
    \end{align*}
    According to Lemma~\ref{lem3}, term $\mathrm{I}$ is dominated by $\mathrm{I}\hspace{-1.2pt}\mathrm{I}$ and $\mathrm{I}\hspace{-1.2pt}\mathrm{I}\hspace{-1.2pt}\mathrm{I}$, i.e.,
    \begin{align*}
        &\sqrt{\overline{m}T}\left(\widehat{\bR}^{i_1}-\bR^{i_1}\bH_\bR\right)\\
        &=\frac{\sqrt{\overline{m}T}}{ps_1}\sum_{j=1}^{s_1}\left[(\widehat{\mathbf{R}}^{i_1}_j-\mathbf{R}^{i_1}\mathbf{H}_{\mathbf{R}_j})\mathbf{H}_{\mathbf{R}_j}^\top\mathbf{R}^\top+\mathbf{R}^{i_1}\mathbf{H}_{\mathbf{R}_j}(\widehat{\mathbf{R}}_{j}-\mathbf{R}\mathbf{H}_{\mathbf{R}_j})^\top\right]\widehat{\bR}\widehat{\bV}_\bR^{-1}+\mathbf{o}_p(1).
    \end{align*}
    Considering term $\mathrm{I}\hspace{-1.2pt}\mathrm{I}$, we expand $\mathbf{R}^\top\widehat{\bR}$ as below,
    \begin{align*}
        &\frac{\sqrt{\overline{m}T}}{ps_1}\sum_{j=1}^{s_1}(\widehat{\mathbf{R}}^{i_1}_j-\mathbf{R}^{i_1}\mathbf{H}_{\mathbf{R}_j})\mathbf{H}_{\mathbf{R}_j}^\top\mathbf{R}^\top\widehat{\bR}\widehat{\bV}_\bR^{-1}\\
        &=\frac{\sqrt{\overline{m}T}}{ps_1}\sum_{j=1}^{s_1}\sum_{l=1}^{p}(\widehat{\mathbf{R}}^{i_1}_j-\mathbf{R}^{i_1}\mathbf{H}_{\mathbf{R}_j})\mathbf{H}_{\mathbf{R}_j}^\top(\mathbf{R}^l)^\top\widehat{\bR}^l\widehat{\bV}_\bR^{-1}\\
        &=\frac{\sqrt{\overline{m}T}}{ps_1}\sum_{j=1}^{s_1}\sum_{l=1}^{p}(\widehat{\mathbf{R}}^{i_1}_j-\mathbf{R}^{i_1}\mathbf{H}_{\mathbf{R}_j})\mathbf{H}_{\mathbf{R}_j}^\top(\mathbf{R}^l)^\top\bR^l\mathbf{H}_{\mathbf{R}}\widehat{\bV}_\bR^{-1}+\mathbf{o}_p(1),
    \end{align*}
    where the last step follows from Theorem 1. Next, we transform term $\mathrm{I}\hspace{-1.2pt}\mathrm{I}\hspace{-1.2pt}\mathrm{I}$ to match its dominant part with the above one of $\mathrm{I}\hspace{-1.2pt}\mathrm{I}$ as
    \begin{align*}
        &\frac{\sqrt{\overline{m}T}}{ps_1}\sum_{j=1}^{s_1}\mathbf{R}^{i_1}\mathbf{H}_{\mathbf{R}_j}(\widehat{\mathbf{R}}_{j}-\mathbf{R}\mathbf{H}_{\mathbf{R}_j})^\top\widehat{\bR}\widehat{\bV}_\bR^{-1}\\
        &=\frac{\sqrt{\overline{m}T}}{ps_1}\sum_{j=1}^{s_1}\sum_{l=1}^{p}\mathbf{R}^{i_1}\mathbf{H}_{\mathbf{R}_j}(\widehat{\mathbf{R}}_{j}^l-\mathbf{R}^l\mathbf{H}_{\mathbf{R}_j})^\top\widehat{\bR}^l\widehat{\bV}_\bR^{-1}\\
        &=\frac{\sqrt{\overline{m}T}}{ps_1}\sum_{j=1}^{s_1}\sum_{l=1}^{p}(\widehat{\mathbf{R}}_{j}^l-\mathbf{R}^l\mathbf{H}_{\mathbf{R}_j})\mathbf{H}_{\mathbf{R}_j}^\top(\mathbf{R}^{i_1})^\top\widehat{\bR}^l\widehat{\bV}_\bR^{-1}\\
        &=\frac{\sqrt{\overline{m}T}}{ps_1}\sum_{j=1}^{s_1}\sum_{l=1}^{p}(\widehat{\mathbf{R}}_{j}^l-\mathbf{R}^l\mathbf{H}_{\mathbf{R}_j})\mathbf{H}_{\mathbf{R}_j}^\top(\mathbf{R}^{i_1})^\top\bR^l\mathbf{H}_{\mathbf{R}}\widehat{\bV}_\bR^{-1}+\mathbf{o}_p(1),
    \end{align*}
    where $\mathbf{R}^{i_1}\mathbf{H}_{\mathbf{R}_j}(\widehat{\mathbf{R}}_{j}^l-\mathbf{R}^l\mathbf{H}_{\mathbf{R}_j})^\top\in\mathbb{R}$, therefore the second equality above holds, i.e., $\mathbf{R}^{i_1}\mathbf{H}_{\mathbf{R}_j}(\widehat{\mathbf{R}}_{j}^l-\mathbf{R}^l\mathbf{H}_{\mathbf{R}_j})^\top=(\widehat{\mathbf{R}}_{j}^l-\mathbf{R}^l\mathbf{H}_{\mathbf{R}_j})\mathbf{H}_{\mathbf{R}_j}^\top(\mathbf{R}^{i_1})^\top$, and the last step follows from Theorem 1. Combining the above two equations, we deduce the following results:
    \begin{align*}
         &\sqrt{\overline{m}T}\left(\widehat{\bR}^{i_1}-\bR^{i_1}\bH_\bR\right)\\
         &=\frac{\sqrt{\overline{m}T}}{ps_1}\sum_{j=1}^{s_1}\sum_{l=1}^{p}(\widehat{\mathbf{R}}^{i_1}_j-\mathbf{R}^{i_1}\mathbf{H}_{\mathbf{R}_j})\mathbf{H}_{\mathbf{R}_j}^\top(\mathbf{R}^l)^\top\bR^l\mathbf{H}_{\mathbf{R}}\widehat{\bV}_\bR^{-1}\\
         &\ \ \ +\frac{\sqrt{\overline{m}T}}{ps_1}\sum_{j=1}^{s_1}\sum_{l=1}^{p}(\widehat{\mathbf{R}}_{j}^l-\mathbf{R}^l\mathbf{H}_{\mathbf{R}_j})\mathbf{H}_{\mathbf{R}_j}^\top(\mathbf{R}^{i_1})^\top\bR^l\mathbf{H}_{\mathbf{R}}\widehat{\bV}_\bR^{-1}+\mathbf{o}_p(1)\\
         &=\frac{\sqrt{\overline{m}T}}{ps_1}\sum_{j=1}^{s_1}\sum_{\stackrel{l=1}{l\neq i_1}}^{p}\left[(\widehat{\mathbf{R}}^{i_1}_j-\mathbf{R}^{i_1}\mathbf{H}_{\mathbf{R}_j})(\mathbf{R}^l\mathbf{H}_{\mathbf{R}_j})^\top+(\widehat{\mathbf{R}}_{j}^l-\mathbf{R}^l\mathbf{H}_{\mathbf{R}_j})(\mathbf{R}^{i_1}\mathbf{H}_{\mathbf{R}_j})^\top\right]\bR^l\mathbf{H}_{\mathbf{R}}\widehat{\bV}_\bR^{-1}\\
         &\ \ \ +\frac{2\sqrt{\overline{m}T}}{ps_1}\sum_{j=1}^{s_1}(\widehat{\mathbf{R}}^{i_1}_j-\mathbf{R}^{i_1}\mathbf{H}_{\mathbf{R}_j})(\mathbf{R}^{i_1}\mathbf{H}_{\mathbf{R}_j})^\top\bR^{i_1}\mathbf{H}_{\mathbf{R}}\widehat{\bV}_\bR^{-1}+\mathbf{o}_p(1)\\
         &=\frac{1}{ps_1}\sum_{j=1}^{s_1}\sum_{\stackrel{l=1}{l\neq {i_1}}}^{p}\sqrt{m_jT}\sqrt{\frac{\overline{m}}{m_j}}\left[(\widehat{\mathbf{R}}^{i_1}_j-\mathbf{R}^{i_1}\mathbf{H}_{\mathbf{R}_j})(\mathbf{R}^l\mathbf{H}_{\mathbf{R}_j})^\top+(\widehat{\mathbf{R}}_{j}^l-\mathbf{R}^l\mathbf{H}_{\mathbf{R}_j})(\mathbf{R}^{i_1}\mathbf{H}_{\mathbf{R}_j})^\top\right]\\
         &\ \ \ \times\bR^l\mathbf{H}_{\mathbf{R}}\widehat{\bV}_\bR^{-1}+\frac{2}{ps_1}\sum_{j=1}^{s_1}\sqrt{m_jT}\sqrt{\frac{\overline{m}}{m_j}}(\widehat{\mathbf{R}}^{i_1}_j-\mathbf{R}^{i_1}\mathbf{H}_{\mathbf{R}_j})(\mathbf{R}^{i_1}\mathbf{H}_{\mathbf{R}_j})^\top\bR^{i_1}\mathbf{H}_{\mathbf{R}}\widehat{\bV}_\bR^{-1}+\mathbf{o}_p(1)\\
         &=\frac{1}{ps_1}\sum_{j=1}^{s_1}\sum_{l=1}^{p}\sqrt{m_jT}(\widehat{\mathbf{R}}^l_j-\mathbf{R}^l\mathbf{H}_{\mathbf{R}_j})\bG_{{i_1},p,q,T}^{j,l}+\mathbf{o}_p(1),
    \end{align*}
    where $\bG_{{i_1},p,q,T}^{j,l}\in \mathbb{R}^{k\times k}$ are defined as
    $$\bG_{{i_1},p,q,T}^{j,l}=\begin{cases}
        \sqrt{\frac{\overline{m}}{m_j}}(\mathbf{R}^{i_1}\mathbf{H}_{\mathbf{R}_j})^\top\bR^{i_1}\mathbf{H}_{\mathbf{R}}\widehat{\bV}_\bR^{-1}+\sqrt{\frac{\overline{m}}{m_j}}\sum_{\tau=1}^p(\mathbf{R}^\tau\mathbf{H}_{\mathbf{R}_j})^\top\bR^\tau\mathbf{H}_{\mathbf{R}}\widehat{\bV}_\bR^{-1},\quad&l={i_1},\\
        \sqrt{\frac{\overline{m}}{m_j}}(\mathbf{R}^{i_1}\mathbf{H}_{\mathbf{R}_j})^\top\bR^l\mathbf{H}_{\mathbf{R}}\widehat{\bV}_\bR^{-1},\ &l\neq {i_1}.
    \end{cases}$$
    As $\underline{m},p,T\to \infty$, by Propositions~\ref{prop5}(c), \ref{prop6}, and \ref{prop8}, we have the following asymptotic properties:
    $$\bH_{\bR_j}\overset{\mathcal{P}}{\longrightarrow}\bQ_{\bR_j}^{-1},\quad\widehat{\bV}_{\bR}\overset{\mathcal{P}}{\longrightarrow}\bV_{\bR},\quad\bH_{\bR}\overset{\mathcal{P}}{\longrightarrow}\bQ_{\bR}^{-1},\quad\forall j\in[s_1].$$
    Then, we can derive asymptotic behavior of $\bG_{i_1,p,q,T}^{j,l}$, as $\underline{m},p,T\to \infty$,
    $$\bG_{{i_1},p,q,T}^{j,l}\overset{\mathcal{P}}{\longrightarrow}\bG_{i_1}^{j,l},$$
    where
    $$\bG_{i_1}^{j,l}=\begin{cases}
        \sqrt{\frac{\overline{m}}{m_j}}(\mathbf{R}^{i_1}\bQ_{\bR_j}^{-1})^\top\bR^{i_1}\bQ_{\bR}^{-1}{\bV}_\bR^{-1}+\sqrt{\frac{\overline{m}}{m_j}}\sum_{\tau=1}^p(\mathbf{R}^\tau \bQ_{\bR_j}^{-1})^\top\bR^\tau\bQ_{\bR}^{-1}{\bV}_\bR^{-1},&\quad l={i_1},\\
        \sqrt{\frac{\overline{m}}{m_j}}(\mathbf{R}^{i_1}\bQ_{\bR_j}^{-1})^\top\bR^l\bQ_{\bR}^{-1}{\bV}_\bR^{-1},&\quad l\neq {i_1}.
    \end{cases}$$
    As $\overline{m}/\underline{m}=\mathcal{O}(1)$, $\bG_{i_1}^{j,l}$ is bounded for any $j\in[s_1]\text{ and }l\in[p]$. Thus, according to Lemma~\ref{lem4} and Slutsky's theorem, we have
    \begin{align*}
       \sqrt{\overline{m}T}\left(\widehat{\bR}^{i_1}-\bR^{i_1}\bH_\bR\right)^\top&=\frac{1}{ps_1}\sum_{j=1}^{s_1}\sum_{l=1}^{p}\sqrt{m_jT}(\bG_{{i_1},p,q,T}^{j,l})^\top(\widehat{\mathbf{R}}^l_j-\mathbf{R}^l\mathbf{H}_{\mathbf{R}_j})^\top+\mathbf{o}_p(1)\\&\overset{\mathcal{D}}{\longrightarrow}\mathcal{N}\left(\mathbf{0},\mathbf{\Sigma}_{\bR_{i_1}}\right), 
    \end{align*}
    where
    \begin{align*}
        \mathbf{\Sigma}_{\bR_{i_1}}=\mathbf{\Sigma}_{\bR,\bG_{i_1}^\top}=&\frac{1}{p^2s_1^2}\sum_{j_1=1}^{s_1}\sum_{j_2=1}^{s_1}\sum_{l_1=1}^p\sum_{l_2=1}^p(\bG_{i_1}^{j_1,l_1})^\top\bV_{\bR_{j_1}}^{-1}\bQ_{\bR_{j_1}}\left(\mathbf{\Phi}_{\bR_{j_1,j_2,l_1,l_2}}^{1,1}+\alpha\mathbf{\Phi}_{\bR_{j_1,j_2,l_1,l_2}}^{1,2}\overline{\bF}^\top\right.\\
        &\left.\ \ \ +\alpha\overline{\bF}\mathbf{\Phi}_{\bR_{j_1,j_2,l_1,l_2}}^{2,1}+\alpha^2\overline{\bF}\mathbf{\Phi}_{\bR_{j_1,j_2,l_1,l_2}}^{2,2}\overline{\bF}^\top\right)\bQ_{\bR_{j_2}}^\top\bV_{\bR_{j_2}}^{-1}\bG_{i_1}^{j_2,l_2}.
    \end{align*}
    
\end{proof}

\subsection*{A.7 Proof of Theorem 5}
\begin{proof}
    We only prove the conclusion for $\wh{\boldsymbol{\Sigma}}_{\bR_{i_1}}$, since proof of results for $\wh{\boldsymbol{\Sigma}}_{\bC_{i_2}}$ is similar. According to Propositions~\ref{prop1} and \ref{prop5}(c) and Theorems 3.3 and 3.5 in \cite{yu2020}, 
    $$\wh{\bF}_{j,t}(\wh{\bC}_{j}^*)^\top(\wh{\bE}_{jt}^{l})^\top\overset{\mathcal{P}}{\longrightarrow}\bH_{\bR_{j}}^{-1}\bF_t(\bH_{\bC_{j}^*}^{-1})^\top(\bC_{j}\bH_{\bC_{j}^*})^\top({\bE}_{jt}^{l})^\top\overset{\mathcal{P}}{\longrightarrow}\bQ_{\bR_{j}}\bF_t\bC_{j}^\top(\bE_{jt}^{l})^\top,$$
    and similarly
    $$\overline{\wh{\bF}}_{j}(\wh{\bC}_{j}^*)^\top(\wh{\bE}_{jt}^{l})^\top\overset{\mathcal{P}}{\longrightarrow}\bQ_{\bR_{j}}\overline{\bF}\bC_{j}^\top(\bE_{jt}^{l})^\top,$$
    where $\frac{1}{m_j}\|\wh{\bC}_{j}^*-\bC_{j}\bH_{\bC_{j}^*}\|_2^2=\mathcal{O}_p(\frac{1}{pT}+\frac{1}{m_j^2})$, $\|\wh{\bF}_{j,t}-\bH_{\bR_{j}}^{-1}\bF_t(\bH_{\bC_{j}^*}^{-1})^\top\|_2^2=\mathcal{O}_p(\frac{1}{pT}+\frac{1}{m_jT}+\frac{1}{pm_j})$, and $\wh{\bE}_{jt,l\tau}-{\bE}_{jt,l\tau}=\mathcal{O}_p(\frac{1}{\sqrt{pT}}+\frac{1}{\sqrt{m_jT}}+\frac{1}{\sqrt{pm_j}})$, for any $j\in[s_1]$, $l\in[p]$, and $\tau\in[q]$, as $m_j,p,T\to\infty$. Then, the HAC estimator $\wh{\boldsymbol{\Sigma}}_{\bR_{j_1,j_2,l_1,l_2}}$ is a consistent estimator of ${\boldsymbol{\Sigma}}_{\bR_{j_1,j_2,l_1,l_2}}$ by Theorem 2 in \cite{Newey1986}, as $\underline{m},\underline{n},T\to\infty$. 
    
    Next, by Propositions~\ref{prop1}, \ref{prop5}(a) and (c), \ref{prop6} and \ref{prop8}, we have $\wh{\bV}_{\bR_{j}}\overset{\mathcal{P}}{\longrightarrow}{\bV}_{\bR_{j}}$, 
    $\wh{\bV}_{\bR}\overset{\mathcal{P}}{\longrightarrow}{\bV}_{\bR}$,
    and
    $$(\wh{\bR}_j^{i_1})^\top\wh{\bR}^l\overset{\mathcal{P}}{\longrightarrow}(\bR^{i_1}\bH_{\bR_{j}})^\top\bR^{l}\bH_{\bR}\overset{\mathcal{P}}{\longrightarrow}(\bR^{i_1}\bQ_{\bR_{j}}^{-1})^\top\bR^{l}\bQ_{\bR}^{-1}.$$
    Therefore, for any $j\in[s_1]$ and $l\in[p]$, as $\underline{m},\underline{n},T\to\infty$, we have
    $$\wh{\bG}_{i_1}^{j,l}\overset{\mathcal{P}}{\longrightarrow}{\bG}_{i_1}^{j,l}.$$
    Thus, by Slutsky's theorem, $\wh{\boldsymbol{\Sigma}}_{\bR_{i_1}}$ is a consistent estimator of ${\boldsymbol{\Sigma}}_{\bR_{i_1}}$.
\end{proof}

\subsection*{A.8 Proof of  Theorem 6}

\begin{proof}
We only prove the conclusion for $\hat{\bR}$, since the proof of results for $\hat{\bC}$ is similar.
Following Equation~(\ref{r-rh}), we have,
\begin{equation*}
	\begin{split}
		\hat{\bR}-\bR\bH_\bR^*&=\frac{1}{ps_1}\sum_{i=1}^{s_1} (\wh{\bR}_i-\bR\bH_{\bR_i}^*)(\wh{\bR}_i-\bR\bH_{\bR_i}^*)^\top\hat{\bR}\widehat{\bV}_{\bR}^{-1} \\
		&\ \ \ +\frac{1}{ps_1}\sum_{i=1}^{s_1} \bR\bH_{\bR_i}^*(\wh{\bR}_i-\bR\bH_{\bR_i}^*)^\top\hat{\bR}\widehat{\bV}_{\bR}^{-1} \\
		&\ \ \ + \frac{1}{ps_1}\sum_{i=1}^{s_1} (\wh{\bR}_i-\bR\bH_{\bR_i}^*)(\bH_{\bR_i}^*)^\top\bR^\top\hat{\bR}\widehat{\bV}_{\bR}^{-1} \\
		&= \mathrm{I}_1+\mathrm{I}_2+\mathrm{I}_3.
	\end{split}
\end{equation*}
Next, we bound $\mathrm{I}_1$, $\mathrm{I}_2$, and $\mathrm{I}_3$. We observe that $\mathrm{I}_1$ is dominated by $\mathrm{I}_2$ and $\mathrm{I}_3$, and $\|\mathrm{I}_2\|_\mathsf{F}\asymp \|\mathrm{I}_3\|_\mathsf{F}$. For $\mathrm{I}_2$, by Assumption 3, Proposition~\ref{prop13} and Lemmas~\ref{lem4.5} and \ref{lem5}, we have
$$\frac{1}{p} \|\mathrm{I}_2\|_\mathsf{F}^2 \leq \frac{C}{s_1}\sum_{i=1}^{s_1} \frac{\|\bR\|_\mathsf{F}^2}{p} \| \bH_{\bR_i}^*\|_\mathsf{F}^2 \frac{\|\wh{\bR}_i-\bR\bH_{\bR_i}^*\|_\mathsf{F}^2}{p}\frac{\|\hat{\bR}\|_\mathsf{F}^2}{p}\| \widehat{\bV}_{\bR}^{-1}\|_\mathsf{F}^2=\mathcal{O}_p\left(\frac{1}{\underline{m}T^2}+\frac{1}{p^2T^2}\right),$$
where $C$ is a positive constant. Thus, as $\underline{m},p,T\to\infty$,
$$\frac{1}{p} \left\|\hat{\bR}-\bR\bH_\bR^*\right\|_\mathsf{F}^2=\mathcal{O}_p\left(\frac{1}{\underline{m}T^2}+\frac{1}{p^2T^2}\right).$$

\end{proof}

\subsection*{A.9 Proof of Theorem 7}

\begin{proof}
Following Equation~(\ref{extendFt}), we have,
\begin{align*}
	\widehat{\mathbf{F}}_t-(\mathbf{H}_\bR^*)^{-1} \mathbf{F}_t [(\mathbf{H}_\bC^*)^{-1}]^{\top}&=  \frac{1}{p q} \widehat{\mathbf{R}}^{\top}[\mathbf{R}-\widehat{\mathbf{R}} (\mathbf{H}_\bR^*)^{-1}] \mathbf{F}_t[\mathbf{C}-\widehat{\mathbf{C}} (\mathbf{H}_\bC^*)^{-1}]^{\top} \widehat{\mathbf{C}} \\
	&\ \ \ +\frac{1}{p} \widehat{\mathbf{R}}^{\top}[\mathbf{R}-\widehat{\mathbf{R}} (\mathbf{H}_\bR^*)^{-1}] \mathbf{F}_t [(\mathbf{H}_\bC^*)^{-1}]^{\top} \\
    &\ \ \ +\frac{1}{q} (\mathbf{H}_\bR^*)^{-1} \mathbf{F}_t[\mathbf{C}-\widehat{\mathbf{C}} (\mathbf{H}_\bC^*)^{-1}]^{\top} \widehat{\mathbf{C}} \\
	&\ \ \  +\frac{1}{p q}(\widehat{\mathbf{R}}-\mathbf{R} \mathbf{H}_\bR^*)^{\top} \mathbf{E}_t(\widehat{\mathbf{C}}-\mathbf{C H}_\bC^*)   +\frac{1}{p q} (\mathbf{H}_\bR^*)^{\top} \mathbf{R}^{\top} \mathbf{E}_t \mathbf{C} \mathbf{H}_\bC^* \\
	&\ \ \  +\frac{1}{p q} (\bH_\bR^*)^\top{\mathbf{R}}^{\top} \mathbf{E}_t(\widehat{\mathbf{C}}-\mathbf{C} \mathbf{H}_\bC^*)+\frac{1}{p q}(\widehat{\mathbf{R}}-\mathbf{R} \mathbf{H}_\bR^*)^{\top} \mathbf{E}_t \mathbf{C H}_\bC^* \\
	& =\sum_{i=1}^7 \mathrm{I}_i .
\end{align*}
Since $\frac{1}{\sqrt{p}}\|\mathbf{R}-\widehat{\mathbf{R}} (\mathbf{H}_\bR^*)^{-1}\|_2=\mathbf{o}_p(1)$ and $\frac{1}{\sqrt{q}}\|\mathbf{C}-\widehat{\mathbf{C}} (\mathbf{H}_\bC^*)^{-1}\|_2=\mathbf{o}_p(1)$ as $\underline{m},\underline{n},T\to\infty$ by Theorem 6, term $\mathrm{I}_1$ is dominated by $\mathrm{I}_2$ and $\mathrm{I}_3$, and term $\mathrm{I}_4$ is dominated by $\mathrm{I}_6$ and $\mathrm{I}_7$. Therefore we bound $\mathrm{I}_2$, $\mathrm{I}_3$, $\mathrm{I}_5$, $\mathrm{I}_6$, and $\mathrm{I}_7$ as:
\begin{align*}
    \|\mathrm{I}_2\|_\text{2} & =\left\|\frac{1}{p}(\widehat{\mathbf{R}}-\mathbf{R} \mathbf{H}_\bR^*)^{\top}[\mathbf{R}-\widehat{\mathbf{R}} (\mathbf{H}_\bR^*)^{-1}] \mathbf{F}_t [(\mathbf{H}_\bC^*)^{-1}]^{\top}\right.\\
    &\ \ \ \left.+\frac{1}{p} (\mathbf{H}_\bR^*)^{\top} \mathbf{R}^{\top}[\mathbf{R}-\widehat{\mathbf{R}} (\mathbf{H}_\bR^*)^{-1}] \mathbf{F}_t [(\mathbf{H}_\bC^*)^{-1}]^{\top}\right\|_2 \\
	& =\mathcal{O}_p \left(\frac{1}{\sqrt{\underline{m}}T}+\frac{1}{pT}\right),
\end{align*}
by Assumptions 3 and 7, Lemma~\ref{lem6}, and Theorem 6. Similarly, adding Assumption 8(c), we have
\begin{align*}
    \|\mathrm{I}_3\|_2&=\mathcal{O}_p \left(\frac{1}{\sqrt{\underline{n}}T}+\frac{1}{qT}\right),\quad  \|\mathrm{I}_5\|_2=\mathcal{O}_p\left(\frac{1}{\sqrt{pq}}\right), \\
    \|\mathrm{I}_6\|_2&=\mathcal{O}_p\left(\frac{1}{\sqrt{p}}\left(\frac{1}{\sqrt{\underline{n}}T}+\frac{1}{qT}\right)\right), \quad \|\mathrm{I}_7\|_2=\mathcal{O}_p\left(\frac{1}{\sqrt{q}}\left(\frac{1}{\sqrt{\underline{m}}T}+\frac{1}{pT}\right)\right).
\end{align*}
Finally, as $\underline{m},\underline{n},T\to\infty$, we have
$$
\left\|\widehat{\mathbf{F}}_t-(\mathbf{H}_\bR^*)^{-1} \mathbf{F}_t [(\mathbf{H}_\bC^*)^{-1}]^{\top}\right\|_2=\mathcal{O}_p\left(\frac{1}{\sqrt{\underline{m}}T}+\frac{1}{\sqrt{\underline{n}}T}+\frac{1}{\sqrt{pq}}\right).
$$
\end{proof}

\subsection*{A.10 Proof of Theorem 8}
\begin{proof}
Let $\bR^{\text{H}*}=\bR\bH_{\bR}^*$, ${\bC}^{\text{H}*}=\bC\bH_{\bC}^*$, and $\bF_t^{\text{H}*}=(\bH_{\bR}^*)^{-1}\bF_t[(\bH_{\bC}^*)^{-1}]^\top$. Following Equation~(\ref{extendst}), for any $i\in[p]$ and $j\in[q]$, we have,
\begin{align*}
        \wh{\bS}_{t,ij}-\bS_{t,ij}&=[\wh{\bR}^i- ({\bR}^{\text{H}*})^i](\wh{\bF}_t-{\bF}_t^{\text{H}*})[\wh{\bC}^j-({\bC}^{\text{H}*})^j]^\top\\
        &\ \ \ +[\wh{\bR}^i- ({\bR}^{\text{H}*})^i]{\bF}_t^{\text{H}*}[\wh{\bC}^j-({\bC}^{\text{H}*})^j]^\top+ ({\bR}^{\text{H}*})^i(\wh{\bF}_t-{\bF}_t^{\text{H}*})[\wh{\bC}^j-({\bC}^{\text{H}*})^j]^\top \\
        &\ \ \ +[\wh{\bR}^i- ({\bR}^{\text{H}*})^i](\wh{\bF}_t-{\bF}_t^{\text{H}*})[({\bC}^{\text{H}*})^j]^\top
        +({\bR}^{\text{H}*})^i(\wh{\bF}_t-{\bF}_t^{\text{H}*})[({\bC}^{\text{H}*})^j]^\top\\
        &\ \ \ +({\bR}^{\text{H}*})^i{\bF}_t^{\text{H}*}[\wh{\bC}^j-({\bC}^{\text{H}*})^j]^\top+[\wh{\bR}^i- ({\bR}^{\text{H}*})^i]{\bF}_t^{\text{H}*}[({\bC}^{\text{H}*})^j]^\top.
\end{align*}
Since $\|\widehat{\mathbf{R}}^i-({\bR}^{\text{H}*})^i\|_2=\mathbf{o}_p(1)$, $\|\widehat{\mathbf{C}}^j-({\bC}^{\text{H}*})^i\|_2=\mathbf{o}_p(1)$, and $\|\widehat{\mathbf{F}}_t-{\bF}_t^{\text{H}*}\|_2=\mathbf{o}_p(1)$ by Theorems 6-7, the last three terms dominate $\wh{\bS}_{t,ij}-\bS_{t,ij}$. By the Cauchy-Schwarz inequality, we bound these terms as
$$({\bR}^{\text{H}*})^i(\wh{\bF}_t-{\bF}_t^{\text{H}*})[({\bC}^{\text{H}*})^j]^\top\le\|({\bR}^{\text{H}*})^i\|_2\|\wh{\bF}_t-{\bF}_t^{\text{H}*}\|_2\|({\bC}^{\text{H}*})^j\|_2,$$
$$({\bR}^{\text{H}*})^i{\bF}_t^{\text{H}*}[\wh{\bC}^j-({\bC}^{\text{H}*})^j]^\top\le\|({\bR}^{\text{H}*})^i\|_2\|{\bF}_t^{\text{H}*}\|_2\|\wh{\bC}^j-({\bC}^{\text{H}*})^j\|_2,$$
and
$$[\wh{\bR}^i- ({\bR}^{\text{H}*})^i]{\bF}_t^{\text{H}*}[({\bC}^{\text{H}*})^j]^\top\le\|\wh{\bR}^i- ({\bR}^{\text{H}*})^i\|_2\|{\bF}_t^{\text{H}*}\|_2\|({\bC}^{\text{H}*})^j\|_2.$$
According to Assumptions 3 and 7, Lemma~\ref{lem6}, and Theorems 6-7, we then have
$$
\wh{\bS}_{t,ij}-\bS_{t,ij}=\mathcal{O}_p\left(\frac{1}{\sqrt{\underline{m}}T}+\frac{1}{\sqrt{\underline{n}}T}+\frac{1}{\sqrt{pq}}\right).
$$
\end{proof}
\newpage
\section*{Appendix B: Additional Numerical Results}
\renewcommand{\thefigure}{B\arabic{figure}}
\renewcommand{\thetable}{B\arabic{table}}
\renewcommand{\theequation}{B\arabic{equation}}
\begin{figure}[h]
\begin{center}
\includegraphics[width=\linewidth]{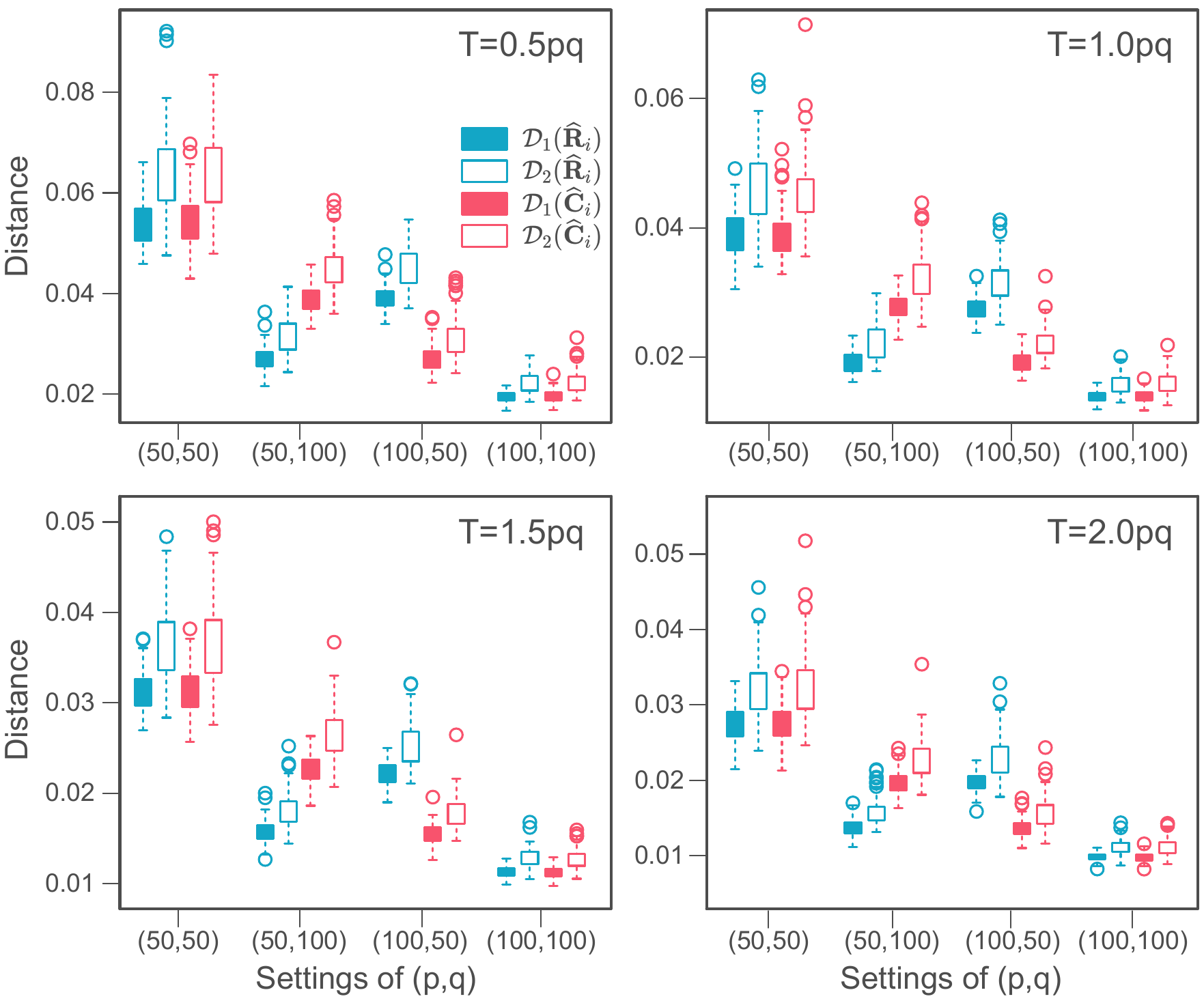}   
\end{center}
\caption{\small Boxplots of the average and maximum column space distances $\mathcal{D}_1(\widehat{\mathbf{R}}_i)$, $\mathcal{D}_2(\widehat{\mathbf{R}}_i)$, $\mathcal{D}_1(\widehat{\mathbf{C}}_j)$, and $\mathcal{D}_2(\widehat{\mathbf{C}}_j)$, defined in Equation (5.3). Data are generated from model (2.1), $\bY_t=\bR\bF_t\bC^\top+\bE_t$, where $\bF_t$ and $\bE_t$ follow the AR(1) processes in Equation (5.2) with $(k,r)=(3,3)$ and $\psi=0.5$, and entries of $\bR$ and $\bC$ are drawn independently from $\mathcal{U}(-1, 1)$. We consider four dimension settings, $(p,q)\in\{(50,50),(50,100),(100,50),(100,100)\}$, and four sample size settings, $T\in\{0.5pq,1.0pq,1.5pq,2.0pq\}$. Results are based on 200 repetitions, with $\alpha=0$, $s_1=s_2=5$, and even partitions.  
}
\label{figure_rici2}
\end{figure}


\begin{figure}[p]
\centering
\includegraphics[width=\linewidth]{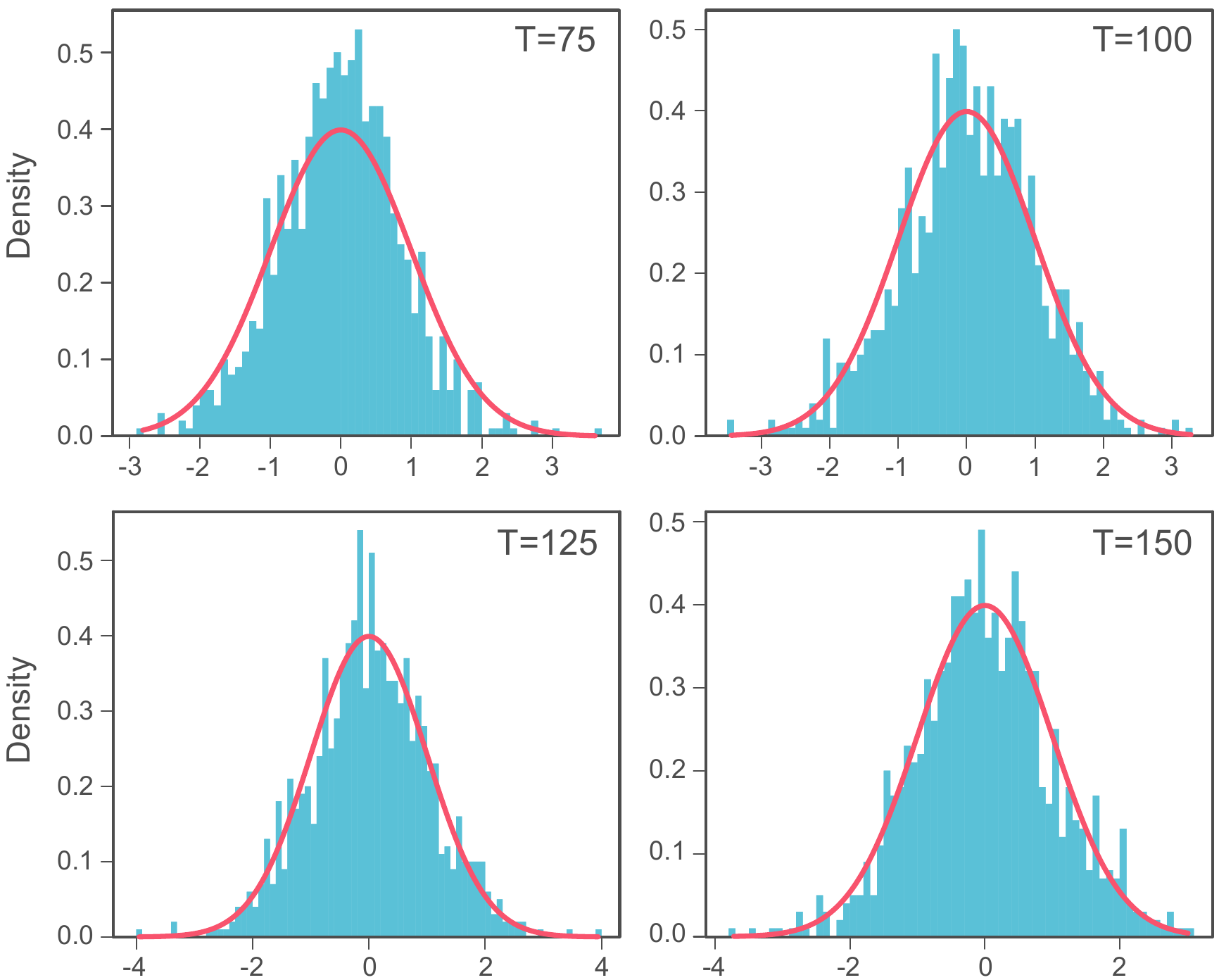}        
\caption{\small Density histograms of the first element of $\widehat{\Sigma}_{\bR_1}^{-1/2}\sqrt{\overline{m}T}(\widehat{\bR}^1-\bR^1\bH_\bR)^\top$. Red curves denote the standard normal distribution. Data are generated from model (2.1), $\bY_t=\bR\bF_t\bC^\top+\bE_t$, where $\bF_t$ and $\bE_t$ follow the matrix normal distributions in Equation (5.4) with $(k,r)=(3,3)$ and $(p,q)=(200,200)$, and entries of $\bR$ and $\bC$ are drawn independently from $\mathcal{U}(-1, 1)$. We consider four sample size settings, $T\in\{0.5pq,1.0pq,1.5pq,2.0pq\}$. Results are based on 1000 repetitions, with $\alpha=0$, $s_1=s_2=3$, and even partitions.   
}
\label{figure_asymptotic_1}
\end{figure}



\begin{sidewaystable}[p]
\caption{\raggedright Values of rotated row loading matrix estimators for the CPI dataset.
}
\centering
\resizebox{\linewidth}{!}{
\centering
\Huge
\begin{tabular}{cccccccccccccccccccccccccccc}
\hline
Estimator & Column & G20 & OECD & OECDE & EU27 & EA20 & G7 & USA & CHN & DEU & JPN & IND & GBR & FRA & ITA & BRA & CAN & RUS & MEX & AUS & KOR & ESP & IDN & NLD & TUR & SAU & CHE \\ \hline
\multirow{3}{*}{$\widehat{\bR}_\text{ours,rot}$} & Column1 & 0 & 0 & 0 & 0 & 0 & 0 & 0 & 0 & 0 & -2 & 0 & 0 & 0 & 0 & 0 & 0 & 0 & 0 & 0 & 0 & 0 & 0 & 0 & 9 & 0 & -1 \\
 & Column2 & 0 & 0 & 0 & 0 & 0 & 0 & 0 & 0 & -2 & -3 & 0 & -2 & -2 & -2 & 0 & 0 & 0 & -1 & 0 & -2 & -1 & 0 & -2 & -1 & 0 & -3 \\
 & Column3 & 0 & 1 & 1 & 1 & 1 & 1 & 2 & 1 & -1 & -3 & 1 & -1 & 0 & 1 & 1 & 2 & 2 & 3 & 0 & 0 & 2 & 0 & 1 & -2 & 2 & -2 \\ \hline
\multicolumn{2}{c}{Cluster} & 3 & 3 & 3 & 3 & 3 & 3 & 4 & 3 & 2 & 2 & 3 & 2 & 2 & 4 & 3 & 4 & 3 & 4 & 3 & 2 & 4 & 3 & 4 & 1 & 3 & 2 \\ \hline
 &  &  &  &  &  &  &  &  &  &  &  &  &  &  &  &  &  &  &  &  &  &  &  &  &  &  &  \\ \hline
Estimator & Column & G20 & OECD & OECDE & EU27 & EA20 & G7 & USA & CHN & DEU & JPN & IND & GBR & FRA & ITA & BRA & CAN & RUS & MEX & AUS & KOR & ESP & IDN & NLD & TUR & SAU & CHE \\ \hline
\multirow{3}{*}{$\widehat{\bR}_{\alpha\text{,rot}}$} & Column1 & 0 & 0 & 0 & 0 & 0 & 0 & 0 & 0 & -2 & -4 & 0 & -2 & -2 & -1 & 0 & -1 & 0 & 0 & 0 & -2 & -1 & 0 & -2 & 0 & 0 & -4 \\
 & Column2 & 0 & 0 & 0 & 0 & 0 & 0 & 0 & 0 & 0 & 1 & 0 & 0 & 0 & 0 & 1 & 0 & 1 & 1 & 0 & 0 & 0 & 0 & 0 & -10 & 0 & 0 \\
 & Column3 & 0 & 0 & 0 & -1 & -1 & 0 & -2 & 0 & 0 & 2 & -1 & 0 & 0 & -1 & -1 & -1 & -2 & -2 & 0 & 0 & -2 & 0 & -1 & 0 & -1 & 1 \\ \hline
\multicolumn{2}{c}{Cluster} & 3 & 3 & 3 & 3 & 3 & 3 & 4 & 3 & 2 & 2 & 3 & 2 & 2 & 3 & 3 & 3 & 4 & 4 & 3 & 2 & 4 & 3 & 3 & 1 & 4 & 2 \\ \hline
 &  &  &  &  &  &  &  &  &  &  &  &  &  &  &  &  &  &  &  &  &  &  &  &  &  &  &  \\ \hline
Estimator & Column & POL & BEL & SWE & IRL & ISR & AUT & NOR & DNK & ZAF & COL & CZE & CHL & FIN & PRT & GRC & HUN & SVK & BGR & CRI & LUX & HRV & LTU & SVN & LVA & EST & ISL \\ \hline
\multirow{3}{*}{$\widehat{\bR}_\text{ours,rot}$} & Column1 & 1 & 0 & 0 & -1 & 0 & 0 & 0 & 0 & 0 & 1 & 0 & 0 & 1 & 0 & 0 & 1 & 0 & 1 & -1 & 0 & 0 & 0 & 0 & 0 & 1 & 1 \\
 & Column2 & -1 & -2 & -2 & -2 & -2 & -2 & -2 & -2 & 0 & 0 & -2 & 0 & -1 & -2 & -1 & -1 & -1 & -1 & -1 & -2 & -1 & -2 & -1 & -1 & -1 & -1 \\
 & Column3 & 2 & 1 & -1 & -1 & 0 & 1 & -2 & 1 & 2 & 2 & 1 & 3 & 2 & 0 & 1 & 2 & 1 & 1 & 2 & 0 & 1 & 1 & 1 & 1 & 1 & 1 \\ \hline
\multicolumn{2}{c}{Cluster} & 4 & 4 & 2 & 2 & 2 & 4 & 2 & 4 & 4 & 4 & 4 & 4 & 4 & 2 & 4 & 4 & 4 & 4 & 4 & 4 & 4 & 4 & 4 & 4 & 4 & 4 \\ \hline
 &  &  &  &  &  &  &  &  &  &  &  &  &  &  &  &  &  &  &  &  &  &  &  &  &  &  &  \\ \hline
Estimator & Column & POL & BEL & SWE & IRL & ISR & AUT & NOR & DNK & ZAF & COL & CZE & CHL & FIN & PRT & GRC & HUN & SVK & BGR & CRI & LUX & HRV & LTU & SVN & LVA & EST & ISL \\ \hline
\multirow{3}{*}{$\widehat{\bR}_{\alpha\text{,rot}}$} & Column1 & 0 & -1 & -3 & -3 & -2 & -1 & -3 & -1 & 0 & 0 & -1 & 0 & 0 & -2 & -1 & 0 & -1 & -1 & 0 & -1 & -1 & -1 & 0 & -1 & 0 & 0 \\
 & Column2 & 0 & 0 & 0 & 0 & 0 & 0 & 0 & 0 & 0 & 0 & 0 & 1 & 0 & 0 & 0 & -1 & 0 & 0 & 1 & 0 & 0 & 0 & 0 & 0 & 0 & 0 \\
 & Column3 & -2 & -1 & 1 & 1 & 0 & -1 & 1 & -1 & -2 & -2 & -1 & -3 & -2 & 0 & -2 & -2 & -2 & -2 & -2 & -1 & -1 & -1 & -2 & -2 & -2 & -2 \\ \hline
\multicolumn{2}{c}{Cluster} & 4 & 3 & 2 & 2 & 2 & 3 & 4 & 3 & 4 & 4 & 3 & 4 & 4 & 2 & 4 & 4 & 4 & 4 & 4 & 3 & 4 & 3 & 4 & 4 & 4 & 4 \\ \hline
\end{tabular}
}
\label{Table_oecd_estir}
\begin{minipage}{\textheight}
    \footnotesize
    \vspace{0.1cm}
    This table reports the rotated row loading matrix estimators $\widehat{\bR}_\text{ours,rot}$ and $\widehat{\bR}_{\alpha\text{,rot}}$ and the corresponding clusters from the CPI dataset, using our method and $\alpha$-PCA. All values are multiplied by 10. The dataset contains 27 expenditure items for 52 countries over 1343 months. The factor numbers are set to $(k,r)=(3,2)$, with $(s_1,s_2)=(4,9)$ and even partitions in our method, and $\alpha=0$ in $\alpha$-PCA.
\end{minipage}
\end{sidewaystable}

\begin{sidewaystable}[p]
\caption{\raggedright Values of rotated column loading matrix estimators for the CPI dataset.
}
\centering
\resizebox{\linewidth}{!}{
\centering
\Huge
\begin{tabular}{cccccccccccccccc}
\hline
Estimator & Column & FNAB & CPI & RH & Energy & MRD & FLPTE & Core CPI & Transport & MGS & FHERHM & EGOF & HWEGF & ABTN & Health \\ \hline
\multirow{2}{*}{$\widehat{\bC}_\text{ours,rot}$} & Column1 & -2 & -1 & -2 & -1 & -3 & -2 & -1 & -2 & -2 & -2 & -2 & -2 & -2 & -2 \\
 & Column2 & 2 & 3 & 0 & 3 & -2 & 0 & 3 & 0 & 0 & 0 & 1 & 1 & 1 & 0 \\ \hline
\multicolumn{2}{c}{Cluster} & 2 & 2 & 3 & 2 & 1 & 3 & 2 & 3 & 3 & 3 & 3 & 3 & 3 & 3 \\ \hline
 &  &  &  &  &  &  &  &  &  &  &  &  &  &  &  \\ \hline
Estimator & Column & FNAB & CPI & RH & Energy & MRD & FLPTE & Core CPI & Transport & MGS & FHERHM & EGOF & HWEGF & ABTN & Health \\ \hline
\multirow{2}{*}{$\widehat{\bC}_{\alpha\text{,rot}}$} & Column1 & -2 & -2 & -3 & -2 & -3 & -2 & -2 & -2 & -2 & -2 & -2 & -2 & -2 & -2 \\
 & Column2 & 0 & 1 & -1 & 1 & -2 & 0 & 1 & 0 & 0 & 1 & 1 & 0 & 0 & 0 \\ \hline
\multicolumn{2}{c}{Cluster} & 2 & 2 & 2 & 2 & 1 & 2 & 2 & 2 & 2 & 2 & 2 & 2 & 2 & 2 \\ \hline
 &  &  &  &  &  &  &  &  &  &  &  &  &  &  &  \\ \cline{1-15}
Estimator & Column & ARH & Commu & RC & EDU & CF & WMSD & HEIRH & Servives & SLIRH & Goods & Housing & IRH & SLH &  \\ \cline{1-15}
\multirow{2}{*}{$\widehat{\bC}_\text{ours,rot}$} & Column1 & -2 & -2 & -2 & -2 & -2 & -2 & -3 & -1 & -2 & -1 & 1 & 1 & 1 &  \\
 & Column2 & 1 & 1 & 1 & 0 & 1 & 0 & -2 & 2 & -2 & 1 & 5 & 4 & 3 &  \\ \cline{1-15}
\multicolumn{2}{c}{Cluster} & 3 & 3 & 3 & 3 & 3 & 3 & 1 & 2 & 1 & 3 & 2 & 2 & 2 &  \\ \cline{1-15}
 &  &  &  &  &  &  &  &  &  &  &  &  &  &  &  \\ \cline{1-15}
Estimator & Column & ARH & Commu & RC & EDU & CF & WMSD & HEIRH & Servives & SLIRH & Goods & Housing & IRH & SLH &  \\ \cline{1-15}
\multirow{2}{*}{$\widehat{\bC}_{\alpha\text{,rot}}$} & Column1 & -2 & -1 & -2 & -2 & -1 & -2 & -3 & -1 & -2 & -1 & 0 & 0 & 0 &  \\
 & Column2 & 0 & 3 & 1 & 0 & 3 & 0 & -2 & 2 & -3 & 2 & 4 & 4 & 3 &  \\ \cline{1-15}
\multicolumn{2}{c}{Cluster} & 2 & 3 & 2 & 2 & 3 & 2 & 1 & 3 & 1 & 3 & 3 & 3 & 3 &  \\ \cline{1-15}
\end{tabular}
}
\label{Table_oecd_estic}
\begin{minipage}{\textheight}
    \footnotesize
    \vspace{0.1cm}
    This table reports the rotated column loading matrix estimators $\widehat{\bC}_\text{ours,rot}$ and $\widehat{\bC}_{\alpha\text{,rot}}$ and the corresponding clusters from the CPI dataset, using our method and $\alpha$-PCA. All values are multiplied by 10. The dataset contains 27 expenditure items for 52 countries over 1343 months. The factor numbers are set to $(k,r)=(3,2)$, with $(s_1,s_2)=(4,9)$ and even partitions in our method, and $\alpha=0$ in $\alpha$-PCA.
\end{minipage}
\end{sidewaystable}

\begin{sidewaystable}[p]
\caption{\raggedright Short and full names of countries and expenditure items in the CPI dataset.}
\begin{subtable}[t]{\textwidth}
    \caption{Short and full names of the countries}
    \resizebox{\linewidth}{!}{
\centering
\Huge
\begin{tabular}{lccccccccccccc}
\hline
Nation\_short & G20 & OECD & OECDE & EU27 & EA20 & G7 & USA & CHN & DEU & JPN & IND & GBR & FRA \\ \hline
\multirow{2}{*}{Nation} & \multirow{2}{*}{G20} & \multirow{2}{*}{OECD} & \multirow{2}{*}{OECD Europe} & \multirow{2}{*}{\begin{tabular}[c]{@{}c@{}}European\\ Union (2020)\end{tabular}} & \multirow{2}{*}{\begin{tabular}[c]{@{}c@{}}Euro area\\ (20 countries)\end{tabular}} & \multirow{2}{*}{G7} & \multirow{2}{*}{United States} & \multirow{2}{*}{China} & \multirow{2}{*}{Germany} & \multirow{2}{*}{Japan} & \multirow{2}{*}{India} & \multirow{2}{*}{\begin{tabular}[c]{@{}c@{}}United\\ Kingdom\end{tabular}} & \multirow{2}{*}{France} \\
 &  &  &  &  &  &  &  &  &  &  &  &  &  \\ \hline
\multicolumn{1}{c}{} &  &  &  &  &  &  &  &  &  &  &  &  &  \\ \hline
Nation\_short & ITA & BRA & CAN & RUS & MEX & AUS & KOR & ESP & IDN & NLD & TUR & SAU & CHE \\ \hline
Nation & Italy & Brazil & Canada & Russia & Mexico & Australia & Korea & Spain & Indonesia & Netherlands & Turkiye & Saudi Arabia & Switzerland \\ \hline
\multicolumn{1}{c}{} &  &  &  &  &  &  &  &  &  &  &  &  &  \\ \hline
Nation\_short & POL & BEL & SWE & IRL & ISR & AUT & NOR & DNK & ZAF & COL & CZE & CHL & FIN \\ \hline
Nation & Poland & Belgium & Sweden & Ireland & Israel & Austria & Norway & Denmark & South Africa & Colombia & Czechia & Chile & Finland \\ \hline
\multicolumn{1}{c}{} &  &  &  &  &  &  &  &  &  &  &  &  &  \\ \hline
Nation\_short & PRT & GRC & HUN & SVK & BGR & CRI & LUX & HRV & LTU & SVN & LVA & EST & ISL \\ \hline
\multirow{2}{*}{Nation} & \multirow{2}{*}{Portugal} & \multirow{2}{*}{Greece} & \multirow{2}{*}{Hungary} & \multirow{2}{*}{Slovak Republic} & \multirow{2}{*}{Bulgaria} & \multirow{2}{*}{Costa Rica} & \multirow{2}{*}{Luxembourg} & \multirow{2}{*}{Croatia} & \multirow{2}{*}{Lithuania} & \multirow{2}{*}{Slovenia} & \multirow{2}{*}{Latvia} & \multirow{2}{*}{Estonia} & \multirow{2}{*}{Iceland} \\
 &  &  &  &  &  &  &  &  &  &  &  &  &  \\ \hline
\end{tabular}
}
\end{subtable}

\begin{subtable}[t]{\textwidth}
    \caption{Short and full names of the expenditure items}
    \resizebox{\linewidth}{!}{
\centering
\huge
\begin{tabular}{lccccccc}
\hline
Item\_short & FNAB & CPI & RH & Energy & MRD & FLPTE & Core CPI \\ \hline
\multirow{2}{*}{Item} & \multirow{2}{*}{\begin{tabular}[c]{@{}c@{}}Food and non-alcoholic\\ beverages\end{tabular}} & \multirow{2}{*}{Total} & \multirow{2}{*}{\begin{tabular}[c]{@{}c@{}}Restaurants\\ and hotels\end{tabular}} & \multirow{2}{*}{Energy} & \multirow{2}{*}{\begin{tabular}[c]{@{}c@{}}Maintenance and repair\\ of the dwelling\end{tabular}} & \multirow{2}{*}{\begin{tabular}[c]{@{}c@{}}Fuels and lubricants for\\ personal transport equipment\end{tabular}} & \multirow{2}{*}{\begin{tabular}[c]{@{}c@{}}All items without \\ food or energy\end{tabular}} \\
 &  &  &  &  &  &  &  \\ \hline
\multicolumn{1}{c}{} &  &  &  &  &  &  &  \\ \hline
Item\_short & Transport & MGS & FHERHM & EGOF & HWEGF & ABTN & Health \\ \hline
\multirow{2}{*}{Item} & \multirow{2}{*}{Transport} & \multirow{2}{*}{\begin{tabular}[c]{@{}c@{}}Miscellaneous goods\\ and services\end{tabular}} & \multirow{2}{*}{\begin{tabular}[c]{@{}c@{}}Furnishings, household equipment\\ and routine household maintenance\end{tabular}} & \multirow{2}{*}{\begin{tabular}[c]{@{}c@{}}Electricity, gas and\\ other fuels\end{tabular}} & \multirow{2}{*}{\begin{tabular}[c]{@{}c@{}}Housing, water, electricity,\\ gas and other fuels\end{tabular}} & \multirow{2}{*}{\begin{tabular}[c]{@{}c@{}}Alcoholic beverages,\\ tobacco and narcotics\end{tabular}} & \multirow{2}{*}{Health} \\
 &  &  &  &  &  &  &  \\ \hline
\multicolumn{1}{c}{} &  &  &  &  &  &  &  \\ \hline
Item\_short & ARH & Commu & RC & EDU & CF & WMSD & HEIRH \\ \hline
\multirow{2}{*}{Item} & \multirow{2}{*}{\begin{tabular}[c]{@{}c@{}}Actual rentals\\ for housing\end{tabular}} & \multirow{2}{*}{Communication} & \multirow{2}{*}{Recreation and culture} & \multirow{2}{*}{Education} & \multirow{2}{*}{Clothing and footwear} & \multirow{2}{*}{\begin{tabular}[c]{@{}c@{}}Water supply and miscellaneous\\ services relating to the dwelling\end{tabular}} & \multirow{2}{*}{\begin{tabular}[c]{@{}c@{}}Housing excluding imputed\\ rentals for housing\end{tabular}} \\
 &  &  &  &  &  &  &  \\ \hline
 & \multicolumn{1}{l}{} & \multicolumn{1}{l}{} & \multicolumn{1}{l}{} & \multicolumn{1}{l}{} & \multicolumn{1}{l}{} & \multicolumn{1}{l}{} & \multicolumn{1}{l}{} \\ \hline
Item\_short & Servives & SLIRH & Goods & Housing & IRH & SLH &  \\ \hline
\multirow{2}{*}{Item} & \multirow{2}{*}{Services} & \multirow{2}{*}{\begin{tabular}[c]{@{}c@{}}Services less imputed\\ rentals for housing\end{tabular}} & \multirow{2}{*}{Goods} & \multirow{2}{*}{Housing} & \multirow{2}{*}{\begin{tabular}[c]{@{}c@{}}Imputed rentals\\ for housing\end{tabular}} & \multirow{2}{*}{Services less housing} & \multirow{2}{*}{} \\
 &  &  &  &  &  &  &  \\ \hline
\end{tabular}
}
\end{subtable}
\label{Table_oecd_labels}
\begin{minipage}{\textheight}
    \footnotesize
    \vspace{0.1cm}
    This table reports the short and full names of countries and expenditure items in the CPI dataset. The dataset contains 27 expenditure items for 52 countries over 1343 months. The observations are sorted by GDP rankings along the nation dimension (World Bank, December 16, 2024) and by variance along the item dimension in decreasing order.
\end{minipage}
\end{sidewaystable}

\end{document}